\documentclass[12pt]{article}
\usepackage[utf8]{inputenc}
\usepackage{hyperref}
\usepackage{amsmath, amssymb, amsthm, amsfonts}
\usepackage{times}
\usepackage{xcolor}
\usepackage{graphicx}
\usepackage[letterpaper, margin=1in]{geometry}
\usepackage{enumitem}
\usepackage{cleveref}
\usepackage{titlesec}
\usepackage{mathrsfs}
\usepackage{physics}
\usepackage{booktabs}
\usepackage{tensor}
\usepackage{algpseudocode,algorithm}
\usepackage[numbers]{natbib}

\titlespacing*{\section}{0pt}{6pt}{0pt}
\titlespacing*{\subsection}{0pt}{6pt}{0pt}

\let\P\relax\DeclareMathOperator{\P}{\mathbb{P}}
\DeclareMathOperator{\E}{\mathbb{E}}

\DeclareMathOperator{\htheta}{\hat\theta}

\newcommand{\R}{\mathbb{R}}

\newcommand{\1}{\mathbf{1}}

\DeclareMathOperator{\diag}{diag}

\DeclareMathOperator{\unif}{Unif}

\newtheorem{lemma}{Lemma}
\newtheorem{corollary}{Corollary}
\newtheorem{theorem}{Theorem}

\newtheorem{assumption}{Assumption}

\newtheorem{definition}{Definition}
\newtheorem{task}{Task}
\newtheorem{construction}{Construction}
\theoremstyle{remark}

\ifdefined\usebigfont

\usepackage{times}
\usepackage[fontsize=13pt]{scrextend}
\AtBeginDocument{
	\newgeometry{letterpaper,left=1.56in,right=1.56in,top=1.71in,bottom=1.77in}
}
\else
\fi

\usepackage{nicematrix}
\usepackage{multirow}
\usepackage{tikz}
\usepackage{subfigure}
\usepackage{authblk}
\usetikzlibrary{positioning}
\usepackage[mode=buildnew]{standalone}
\NiceMatrixOptions{cell-space-limits = 1pt}
\DeclareMathOperator{\attn}{attn}

\renewcommand{\S}{\mathcal{S}}
\newcommand{\TT}{T_{\text{eff}}}
\newcommand{\calE}{\mathcal{E}}
\newcommand{\calR}{\mathcal{R}}
\let\P\relax\DeclareMathOperator{\P}{\mathbb{P}}
\newcommand\smalldots{\hbox to 1em{.\hss.\hss.}}

\title{How Transformers Learn Causal Structure with Gradient Descent}
\author{Eshaan Nichani}
\author{Alex Damian}
\author{Jason D. Lee}
\affil{Princeton University}

\ifdefined\usebigfont

\usepackage{times}
\usepackage[fontsize=13pt]{scrextend}
\makeatletter
\@ifpackageloaded{geometry}{\AtBeginDocument{\newgeometry{letterpaper,left=1.56in,right=1.56in,top=1.71in,bottom=1.77in}}}{\usepackage[letterpaper,left=1.56in,right=1.56in,top=1.71in,bottom=1.77in]{geometry}}
\AtBeginDocument{\newgeometry{letterpaper,left=1.56in,right=1.56in,top=1.71in,bottom=1.77in}}
\makeatother
\else
\fi

\begin{document}
\maketitle

\begin{abstract}
The incredible success of transformers on sequence modeling tasks can be largely attributed to the \emph{self-attention} mechanism, which allows information to be transferred between different parts of a sequence. Self-attention allows transformers to encode causal structure which makes them particularly suitable for sequence modeling. However, the process by which transformers learn such causal structure via gradient-based training algorithms remains poorly understood. To better understand this process, we introduce an in-context learning task that requires learning latent causal structure. We prove that gradient descent on a simplified two-layer transformer learns to solve this task by encoding the latent causal graph in the first attention layer. The key insight of our proof is that the gradient of the attention matrix encodes the mutual information between tokens. As a consequence of the data processing inequality, the largest entries of this gradient correspond to edges in the latent causal graph. As a special case, when the sequences are generated from in-context Markov chains, we prove that transformers learn an \emph{induction head}~\cite{olsson2022context}. We confirm our theoretical findings by showing that transformers trained on our in-context learning task are able to recover a wide variety of causal structures.
\end{abstract}

\section{Introduction}

The transformer architecture~\citep{vaswani2017attention} has revolutionized the field of deep learning, and has led to state-of-the art performance on tasks spanning language modeling~\citep{brown2020language}, computer vision~\citep{dosovitskiy2020image}, reinforcement learning~\citep{chen2021decision}, and the sciences~\citep{jumper2021highly}. The basic primitive of a transformer is a \emph{self-attention} head, a sequence-to-sequence mapping in which each token in the output is a weighted linear combination of, i.e ``attends to," the other tokens in the sequence. Prior work \citep{elhage2021mathematical} has sought to understand which specific computational operations, or ``circuits," are implemented by self-attention layers in trained transformers. However, the process by which such circuits arise when transformers are trained from scratch via gradient-based algorithms is still unknown.

One hallmark capability of transformers is \emph{in-context learning}~\citep{brown2020language}, which is the ability to learn from information present in the input context without needing to update the model parameters. For example, given a prompt of input-output pairs, in-context learning is the ability to predict the output corresponding to a new input. Prior work has shown that this in-context learning ability relies on the existence of specific circuits called \emph{induction heads}~\citep{olsson2022context}. Given a prompt of the form $[\cdots, A, B, \cdots, A]$, an induction head copies the token which follows the previous occurrence of $A$, in this case being $B$. This can be implemented using two attention layers: the first performs the operation of ``copying" the previous token, while the second compares this previous token to the last token of the context. By copying the previous token, the first attention layer thus implicitly encodes the causal structure of a Markov chain.

As another example, consider the setting of learning a function class in-context, introduced by \citet{garg2022can}. Each prompt sequence is formed by sampling a new function $f$ from some function class $\mathcal{F}$, and generating the prompt $[x_1, f(x_1), x_2, f(x_2), \dots, x_n, f(x_n), x_{test}]$ for i.i.d inputs $x_1, \dots, x_n, x_{test}$. The model must learn to estimate $f(x_{test})$ from the $(x_i,f(x_i))$ pairs given in-context. This setting has proved a useful testbed to understand which in-context learning algorithms can be implemented by transformers; see \Cref{sec:related_work} for additional discussion.
Such prompts also possess special causal structure. Conditioned on $f$, the $2k$-th token in the prompt only depends on the $(2k - 1)$-th token, and is independent of the rest of the sequence. The model must thus learn to associate each $f(x_k)$, at position $2k$, with its corresponding $x_k$, at position $2k-1$.
We view this instance as a problem with both a global causal structure, which comes from pairing $x_k$ with $f(x_k)$, and an in-context transition which comes from the specific $f \in \mathcal{F}$ sampled for the sequence.

Transformers are clearly able to model causal structure, yet despite the necessity of doing so for performing in-context learning tasks, we still do not understand how such structures are learned by gradient descent when training from scratch. We thus ask the following question:
\begin{center}
    \textbf{How do transformers learn causal structure with gradient descent?}
\end{center}

\subsection{Our Contributions}
In this work, we analyze the gradient descent dynamics of an autoregressive two-layer attention-only transformer, and prove that it recovers latent causal structure. Our specific contributions are as follows:
\begin{itemize}
    \item We introduce a novel family of in-context learning problems, which we call \emph{random sequences with causal structure} (\Cref{task:single_parent}). The task fixes a latent causal structure, unknown to the transformer, and samples each sequence from a different distribution which respects the causal structure.    
    \item When the latent causal graph is a tree, we prove that gradient descent on a simplified two-layer transformer solves this task by encoding the causal graph in the first attention layer in order to perform in-context estimation of the transition distribution (\Cref{thm:main_thm}). As a special case of \Cref{thm:main_thm}, we show that when the sequences are generated from in-context Markov chains, the transformer learns an induction head.
    \item The proof of \Cref{thm:main_thm} relies on showing that the gradient of the first attention layer automatically computes the $\chi^2$-mutual information between pairs of tokens. As a result of the data processing inequality, the largest entries of this gradient correspond to edges in the latent causal graph, and hence the first attention layer converges to the adjacency matrix of this graph.
    \item When the causal graph is not a tree, we explicitly construct a multi-head transformer which solves this task by distributing the latent causal graph across many heads. We show empirically that transformers trained by gradient descent on this task learn our construction.
\end{itemize}

\subsection{Related Work}\label{sec:related_work}

\paragraph{In-Context Learning.} \citet{brown2020language} demonstrated that GPT-3 can perform in-context learning, which has led to much subsequent work on understanding how such in-context learning ability arises. \citet{xie2021explanation} presents a Bayesian perspective on in-context learning by looking at the log-likelihood of out-of-distribution prompt sequences. \citet{olsson2022context} posits that in-context learning relies on the emergence of induction heads.

\citet{garg2022can} formalizes the setting of learning a function class in-context, and shows that transformers can be trained to in-context learn various simple function classes such as (sparse) linear models or shallow neural networks. This bears similarity to transformer neural processes ~\citep{nguyen2022transformer}, which recasts uncertainty-aware meta learning as an in-context learning task. Many recent works have sought to understand which in-context learning algorithms can be efficiently expressed by a transformer. \citet{akyurek2023what, bai2023transformers, von2023transformers} show that transformers can implement gradient descent to solve in-context linear regression, while \citet{fu2023transformers} constructs transformers which implement higher-order learning algorithms such as Newton's method. \citet{giannou2023looped} constructs a transformer that can express general-purpose computational operations. However, these works are solely concerned with the representational capabilities of transformers, and do not answer the question of whether gradient descent indeed learns such constructions.

\paragraph{Training dynamics of transformers.} Prior works have primarily studied the optimization dynamics of a single attention layer. \citet{lu2021on, li2023transformers} show that a single attention layer trained via gradient descent learns to encode topic structure. \citet{snell2021approximating} studies the dynamics of a simplified attention layer on sequence-to-sequence translation tasks. \citet{tarzanagh2023max, tarzanagh2023transformers} show an equivalence between the optimization dynamics of a single attention layer and a certain SVM problem. \citet{ahn2023transformers} shows that the global optimum of a single linear attention layer trained on in-context linear regression implements a single step of preconditioned gradient descent. \citet{mahankali2023one, zhang2023trained} study the optimization dynamics of a single linear attention layer for performing in-context linear regression. \citet{huang2023context} shows that gradient descent on a single softmax attention layer learns to solve linear regression in-context when the input data are orthogonal. \citet{jelassi2022vision} analyzes the gradient descent dynamics of the position-position block of a single layer vision transformer, and shows that it converges to a solution encoding spacial structure. \citet{tian2023scan} studies the optimization dynamics of a single attention layer for a specific toy dataset, while \citet{tian2023joma} shows that jointly training a self-attention and MLP layer corresponds to the optimization dynamics of a certain modified MLP module. \citet{boix2023transformers} shows that transformers with diagonal attention matrices display incremental learning behavior.

\citet{bietti2023birth} studies a synthetic ICL task, and shows that an induction head is formed during training. They demonstrate heuristically that a few gradient steps on a modifed transformer architecture approximately learns the induction head. Our synthetic task handles more general causal structure, and requires attending to all prior instances of the final token rather than the most recent one. Furthermore, \citet{bietti2023birth} requires the alphabet size $S$ to be significantly larger than the context length $T$; we, however, assume that $T \gg S$, and our main theorem (\Cref{thm:main_thm}) is an end-to-end guarantee on learning the causal structure and obtaining vanishing population loss.

\paragraph{Concurrent Work.} A number of concurrent works also study the ability of transformers to solve synthetic in-context learning tasks. \citet{reddy2023mechanistic} shows empirically that an induction head suddenly emerges during the training of a two-attention layer transformer on a specific in-context learning task. \citet{akyurek2024context} demonstrates that transformers can learn regular languages in-context, where each prompt consists of strings generated from a prompt-dependent formal language. This is due to the ability of transformers to compute in-context $n$-gram counts, via a generalization of the induction head mechanism using multiple attention layers. In \Cref{sec:multiple_parents}, we show that a two-layer transformer with $n$ heads can also compute such $n$-gram counts. 


\citet{edelman2024evolution} study the formation of induction heads on the task of learning Markov chains in context, which is equivalent to our \Cref{task:single_parent} when the latent causal graph is the chain graph. Their theoretical analysis focuses on two steps of gradient descent on a simplified linear transformer model, for Markov chains over two states. An interesting empirical observation of theirs is that transformers trained to learn $n$-grams in-context undergo a sequential learning procedure, by first predicting using the unigram counts, then the bigram counts, and so on. 

\section{Setup}

\subsection{Transformer Architecture}

Let $[S]$ be a finite alphabet. Transformers are models mapping sequences $s_{1:T} := (s_1, \dots, s_T) \in [S]^T $ of length $T$ to a length $T$ sequence of vectors $z_1, \dots, z_T \in \R^{d_{out}}$. A transformer first embeds the sequence $s_{1:T}$ as a matrix $X {=} \begin{bmatrix} x_1, x_2, \dots, x_T \end{bmatrix}^\top \in \mathbb{R}^{T \times d}$, where $d$ is the embedding dimension. This is parameterized by the token embeddings $E \in \R^{d \times S}$ and positional embeddings $P \in \R^{d \times T}$:
\begin{align}
    \mathrm{embed}(s_{1:T};(E,P))_i := E e_{s_i} + Pe_i ~\text{for}~i=1,\ldots,T.
\end{align}
Transformers consist of two types of layers: attention layers and MLP layers. Throughout, we focus on decoder-based, attention-only transformers. These are models in which every layer is a \emph{causal self-attention layer}, defined below:

\begin{definition}[Causal self-attention head]\label{def:attention}
    For a vector $v \in \mathbb{R}^k$, define the \emph{softmax} function $\S: \mathbb{R}^k \rightarrow \R^k$ by $\S(v)_i := \frac{\exp(v_i)}{\sum_{j=1}^k \exp(v_j)}$. For a matrix $A \in \R^{d \times d}$, define the operator $\attn(\cdot; A) : \R^{T \times d} \rightarrow \R^{T \times d}$ by
    \begin{align}
        \attn(h; A) := \S\qty(\mathrm{MASK}(hAh^\top))h \in \mathbb{R}^{T \times d},
    \end{align}
    where $\mathrm{MASK}(M)_{i,j}$ is $M_{i,j}$ when $i \ge j$ and $-\infty$ otherwise, and the softmax function $\S$ is applied row-wise.
\end{definition}
In \Cref{def:attention}, the masking operator ensures tokens only attend to previous tokens in the sequence, and the softmax normalizes the output so that each row sums to $1$. The amount that token $i$ attends to token $j$, for $j \le i$, is thus $\S\qty(\mathrm{MASK}(XAX^\top))_{i,j} = \S\qty(X_{\le i}A^\top  x_i)_j$, where $X_{\le i} \in \R^{i \times d}$ is the submatrix formed by the first $i$ rows of $X$.

A single attention head is parameterized by the tuple of $d \times d$ matrices $(Q, K, V)$, referred to as the query, key, and value matrices, and maps $X$ to the sequence $\attn(X; QK^\top )V^\top $.

A decoder-based transformer aggregates multiple causal self-attention heads over many layers:
\begin{definition}[Decoder-based transformer]\label{def:transformer}
    Let $L$ be the depth, $\{m_\ell\}_{\ell \in [L]}$ be the number of heads per layer, and $d$ be the embedding dimension. For $\ell \in [L]$, $i \in [m_\ell]$, let $(Q^{(\ell)}_{i}, K^{(\ell)}_i, V^{(\ell)}_i)$ be the query, key, and value matrices for the $i$th head in the $\ell$th layer. Let $W_{O} \in \R^{d_{out} \times d}$ be the output layer and let $E \in \R^{d \times S}$ and $P \in \R^{d \times T}$ be the token and positional embeddings respectively. Define the parameter vector $\theta := \{(Q^{(\ell)}_{i}, K^{(\ell)}_i, V^{(\ell)}_i)\}_{\ell \in [L], i \in [m_\ell]} \cup \{E,P,W_{O}\}$. A decoder-based transformer $\mathrm{TF}_\theta :[S]^T  \rightarrow \mathbb{R}^{T \times d_{out}}$ operates on $s_{1:T}$ by
    \begin{align}\label{eq:decoder-based-TF}
        &h^{(0)} = \mathrm{embed}(s_{1:T};(E,P)) \nonumber\\
        &h^{(\ell)} = h^{(\ell - 1)} + \sum_{i=1}^{m_\ell} \attn\qty(h^{(\ell - 1)}; Q^{(\ell)}_{i}{K^{(\ell)}_i}^\top ){V^{(\ell)}_i}^\top  \\
        &\mathrm{TF}_\theta(s_{1:T})  = h^{(L)}W_O^\top .\nonumber
    \end{align}
    We remark that $h^{(\ell)} \in \R^{T \times d}$ for $\ell = 0,\ldots,L$.
\end{definition}

\paragraph{Disentangled Transformer.} Prior works on mechanistic interpretability have introduced the \emph{residual stream} viewpoint to understand the behavior of trained transformers~\cite{elhage2021mathematical}. The residual stream exists as a memory and communication channel that various attention heads read and write to. Information in the residual stream is stored in low-dimensional subspaces of intermediate layers $h^{(\ell)}$. For a single attention layer $\attn(\cdot; QK^\top )V^\top $, the query and key matrices ``read" information from the relevant subspace, and the value matrix ``writes" the output to a new subspace of the residual stream. The weight matrices thus act as associative memories~\citep{bietti2023birth}, storing various embeddings within the residual stream.

While this residual stream viewpoint provides intuition for the flow of information through the forward pass of a transformer, from an interpretability perspective it is difficult to know which subspaces contain which information. The outputs of each attention layer are added together and thus their informations may overlap with each other, leading to a memory bottleneck~\citep{elhage2021mathematical, bietti2023birth}. \citet{friedman2023learning} thus consider a transformer model in which the residual stream is disentangled, and the outputs of each attention layer are \emph{appended} to the residual stream. The dimension of the residual stream thus grows with the depth. We formalize this as a \emph{disentangled transformer}, defined below:

\begin{definition}[Disentangled Transformer]\label{def:disentangled_TF}
    Let $L$ be the depth, and $\{m_\ell\}_{\ell \in [L]}$ be the number of heads per layer. Define the set of dimensions $d_0,\ldots,d_L$ by $d_0 = S+T$ and $d_{\ell} = (1 + m_\ell)d_{\ell - 1}$. Let $\{\widetilde{A}_i^{\ell}\}$ be the attention matrices with $\widetilde{A}_i^{(\ell)} \in \R^{d_{\ell-1} \times d_{\ell-1}}$, let $\widetilde{W}_O \in \R^{d_{out} \times d_L}$ be the output matrix, and let $\widetilde{\theta} = \{\widetilde{A}^{(\ell)}_i\}_{\ell \in [L], i \in [m_\ell]} \cup \{\widetilde{W}_O\}$. A disentangled transformer $\widetilde{\mathrm{TF}}_{\widetilde{\theta}}$ acts on a sequence $s_{1:T}$ by:
    \begin{align}
        &h^{(0)} = \tilde X = [\tilde x_1,\ldots,\tilde x_T]^\top  ~\text{where}~ \tilde x_t = [e_{s_t},e_t] \in \R^{d_0} \nonumber\\
        &h^{(\ell)} {=} \qty[ h^{(\ell - 1)}, \attn(h^{(\ell - 1)}; \widetilde{A}^{(\ell)}_1), \smalldots, \attn(h^{(\ell - 1)}; \widetilde{A}^{(\ell)}_{m_\ell})] \\
        &\widetilde{\mathrm{TF}}_{\widetilde{\theta}}(s_{1:T}) = h^{(L)} \widetilde{W}_O^\top .\nonumber
    \end{align}
    We remark that $h^{(\ell)} \in \R^{T \times d_\ell}$ for $\ell=0,\ldots,L$.
\end{definition}

In addition to disentangling the residual stream, \Cref{def:disentangled_TF} replaces the query and key matrices with a single attention matrix $\widetilde{A}: = QK^\top $ and absorbs the value matrices into $\widetilde{W}_O$. By allowing $d_\ell$ to grow with the depth, this disentangled transformer is actually \emph{equivalent} to a decoder-based attention-only transformer (see \Cref{thm:disentangled_TF_equivalence} for the formal statement). Given this equivalence, throughout the rest of the paper we study the disentangled transformer.

Finally, we remark that when the target is a vector in $\R^{d_{out}}$ rather than a sequence in $\R^{T \times d_{out}}$, it is customary to use the embedding of the last token, i.e. $\mathrm{TF}_\theta(s_{1:T}) = W_O h^{(L)}_T$ and similarly for $\widetilde{\mathrm{TF}}_{\widetilde{\theta}}$. 

\begin{figure*}[t!]
    \centering
    \begin{tikzpicture}[
            node distance=1.5cm,
            auto,
            every node/.style={circle, draw, minimum size=1cm},
        ]
        \def\parents{-1,1,1,2,3,-1}
        \def\T{0}
        \foreach \x in \parents {
            \pgfmathparse{\T+1}
            \xdef\T{\pgfmathresult}
        }
        \foreach \x in {1,...,\T}{
                \ifnum\x=1
                    \node (node\x) {$s_\x$};
                \else
                    \pgfmathtruncatemacro{\y}{\x-1}
                    \node[right=of node\y] (node\x) {$s_\x$};
                \fi
            }
        \foreach \x [count=\xi] in \parents{
            \ifnum\x>-1
                \draw[->, thick, -latex, bend left] (node\x) to (node\xi);
            \fi
        }
    \end{tikzpicture}
    \caption{\textbf{Random Sequence with Causal Structure:} The causal structure is defined by the graph $\mathcal{G}$, denoted by the arrows. In this figure, $p(1) = \emptyset$, $p(2) = \{1\}$, $p(3) = \{1\}$, $p(4) = \{2\}$ and $p(5) = \{3\}$. Sequences are generated by sampling $\pi \sim P_\pi$, $s_1 \sim \mu_\pi$, $s_2 \sim \pi(\cdot|s_1)$, $s_3 \sim \pi(\cdot|s_1)$, $s_4 \sim \pi(\cdot|s_2)$, $s_5 \sim \pi(\cdot|s_3)$, and finally $s_6 \sim \unif([S])$. The target $y$ for this sequence is drawn from $\pi(\cdot|s_6)$.}
    \label{fig:graph_example}
\end{figure*}
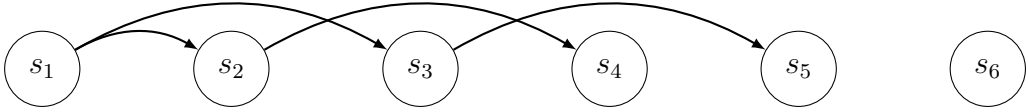

\subsection{Problem Setup: Random Sequences with Causal Structure}\label{sec:dag_single_parent}

Let $\mathcal{G} = ([T],\mathcal{E})$ be a directed acyclic graph on $[T] = \{1,\ldots,T\}$ with edge set $\calE$, which represents the latent causal structure. We assume that $(j \to i) \in \mathcal{E}$ only if $j < i$, i.e. each token can only point to future tokens. For a position $i \in [T]$, we let $p(i)$ denote the set of parents of $i$, i.e. $p(i) := \{j : (j \to i) \in \mathcal{E}\}.$ We let $\calR$ denote the set of root nodes, i.e $\calR = \{i : p(i) = \emptyset\}$. For most of the paper, we assume that each position has at most one parent, i.e. $|p(i)| \le 1$ for all $i \in [T]$. See \Cref{sec:multiple_parents} for the generalization to multiple parents. When $|p(i)| = 1$, we overload notation and use $p(i) \in [T]$ to denote the unique parent of $i$.

We will also assume there exists a prior $P_\pi$ over irreducible and aperiodic Markov chains $\pi$ on $[S] = \{1,\ldots,S\}$. For each $\pi$, we will use $\mu_\pi$ to denote the unique stationary measure of $\pi$. Then each sequence $[s_1,\ldots,s_T]$ and its corresponding target $y$ are generated by the following procedure:
\begin{task}[Random Sequence with Causal Structure]\label{task:single_parent}~
\begin{enumerate}
    \item First, draw $\pi \sim P_\pi$.
    \item For $i=1,\ldots,T-1$, sample $s_i \sim \mu_\pi$ if $p(i) = \emptyset$. Otherwise sample $s_i \sim \pi(\cdot|s_{p(i)})$.
    \item Draw $s_{T} \sim \text{Unif}([S])$ and $s_{T+1} \sim \pi(\cdot | s_{T})$
    \item Return the input $x = s_{1:T}$ and the target $y=s_{T+1}$.
\end{enumerate}
Because $s_T \sim \text{Unif}([S])$, $T$ is a root node of $\mathcal{G}$, i.e. $T \in \calR$.
\end{task}
See \Cref{fig:graph_example} for an example instance of this task. 

\subsection{Examples}

\paragraph{Markov Chains and Induction Heads.} First, consider the case where $p(i) = i-1$. The sequence $s_1, \dots, s_{T-1}$ is a Markov chain conditioned on $\pi$, with transition matrix $\pi$. \Cref{task:single_parent} reduces to the problem of \emph{estimating the Markov chain $\pi$ in-context}. This can be solved via an induction head~\citep{olsson2022context} which, when presented with a prompt $\mathcal{P} = [{\cdots}, A, B, {\cdots}, A, C, {\cdots}, A]$, averages over the tokens following the previous occurrences of $A$, (in this case $B$ and $C$). Explicitly, the output of an induction head on the sequence $s_{1:T}$ will be the empirical estimate for $\pi(\cdot \mid s_T)$:
\begin{align*}
    \mathrm{TF}_\theta(s_{1:T})_{s'} = \frac{|\{i ~:~ (s_{i-1},s_i) = (s_T,s')\}|}{|\{i ~:~ s_{i-1} = s_T\}|}.
\end{align*}
In the limit as $T \to \infty$, this converges to the true transition $\pi(\cdot \mid s_T)$. 

\paragraph{In-Context Learning.} Consider the in-context learning setup from \citet{garg2022can}. This corresponds to the causal graph where $p(2k - 1) = \emptyset$ and $p(2k) = 2k-1$. Sequences are generated by sampling $f: [S] \to [S]$ from $\mathcal{F}$ and using the transition matrix $\pi(s'|s) = \1(s'=f(s))$. To learn this function class in-context, the transformer must learn to associate the $(x, y)$ pairs in positions $2k-1$ and $2k$.

\section{What does the Transformer Learn?}

\subsection{Experiments}

\begin{figure*}[t!]
\centering
    \includegraphics[width=0.8\textwidth]{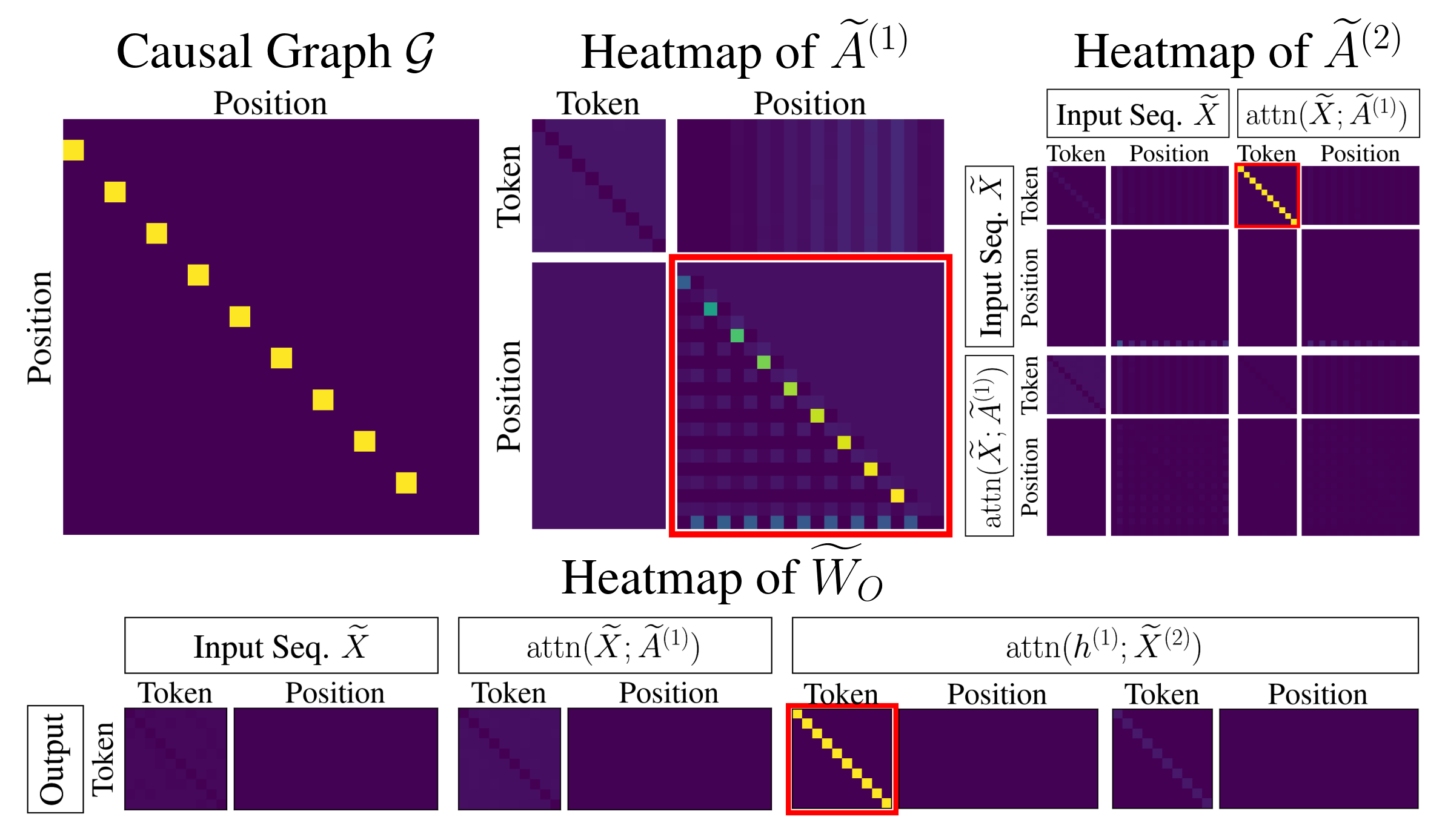}
    \caption{\textbf{The Weights of a Trained Transformer:} We plot the weights of a two layer disentangled transformer trained on \Cref{task:single_parent} with $S=10$ and $T=20$ when the causal graph is the in-context learning graph where $p(2i) = 2i-1$ for all $i > 0$. All entries of $A^{(1)}, A^{(2)}, W_O$ remain small except the three blocks highlighted in red. The highlighted block in $A^{(1)}$ converges to the adjacency matrix of the causal graph $\mathcal{G}$, and the highlighted blocks in $A^{(2)},W_O$ converge to the identity matrix $I_S$.}
    \label{fig:induction_head_combined}
\end{figure*}

We train a series of two-layer disentangled transformers with one head per layer on \Cref{task:single_parent}, for varying latent graphs $\mathcal{G}$. The prior $P_\pi$ is chosen so that each row of $\pi$ is sampled i.i.d from the Dirichlet distribution with parameter $\alpha$, i.e $\pi(\cdot \mid s) \sim \mathrm{Dir}(\alpha \cdot 1_S)$, for varying values of $\alpha$. We train using gradient descent on the cross entropy loss with initial learning rate $1$ and cosine decay over $2^{17}$ steps. 

We observe that the weights of the trained disentangled transformers exhibit consistent structure. First, all of the entries of $\widetilde{A}^{(1)}$ remain small except the position-position block (red box under $\widetilde{A}^{(1)}$ in \Cref{fig:induction_head_combined}), which converges to the adjacency matrix of the graph $\mathcal{G}$. Next, all of the entries of $\widetilde{A}^{(2)}$ are small, except the token/token block comparing the $h^{(0)}$ component of the residual stream of token $i$ to the $\attn(h^{(0)},\widetilde{A}^{(1)})$ component of the residual stream of token $j$ for $j \le i$
(red box under $\widetilde{A}^{(2)}$ in \Cref{fig:induction_head_combined}). Finally, all of the entries of the output matrix $W_O$ are small except the token/token block which returns the value of the first component of the output of the second attention $\attn(h^{(1)},A^{(2)})$ (red box under $W_O$ in \Cref{fig:induction_head_combined}). In \Cref{fig:appendix_single_parent}, we observe that this weight pattern persists for different latent graphs $\mathcal{G}$.

In the following section, we explicitly define this construction and describe the corresponding dynamics of the forward pass in \Cref{fig:trace_forward_pass}.

\subsection{Construction}\label{sec:construction}
In \Cref{fig:induction_head_combined} we observe that the attention weights $\widetilde{A}^{(1)}, \widetilde{A}^{(2)}$ and output weight $\widetilde{W}_O$ are of the form 
\begin{align}\label{eq:sparsity_pattern}
    \widetilde{A}^{(1)} & = \begin{bmatrix} 0_{S \times S} & 0_{S \times T} \\ 0_{T \times S} & A^{(1)} \end{bmatrix}\nonumber\\
    \widetilde{A}^{(2)} &= \left[\begin{array}{c c | c c} 0_{S \times S} & 0_{S \times T} & A^{(2)} & 0_{S \times T} \\ 
    0_{T \times S} & 0_{T \times T} & 0_{T \times S} & 0_{T \times T}\\
    \hline
    0_{S \times S} & 0_{S \times T} & 0_{S \times S} & 0_{S \times T}\\
    0_{T \times S} & 0_{T \times T} & 0_{T \times S} & 0_{T \times T}\\
    \end{array}\right]\\
    \widetilde{W}_O       & = \left[\begin{array}{c | c | c c | c} 0_{S \times d} & 0_{S \times d} & I_S & 0_{S \times T} & 0_{S \times d}\end{array}\right]\nonumber
\end{align}
for matrices $A^{(1)} \in \R^{T \times T}$ and $A^{(2)} \in \R^{S \times S}$. We now explicitly construct the $A^{(1)}$ and $A^{(2)}$ from \Cref{fig:induction_head_combined} that solve \Cref{task:single_parent}. Indeed, we show that this construction solves the task by estimating the \emph{empirical transition} matrix $\hat \pi_{s_{1:T}}$:
\begin{align}
    \hat\pi_{s_{1:T}}(s' \mid s) := \frac{|\{(j \to i) \in \mathcal{E} ~:~ (s_j,s_i)=(s,s')\}|}{|\{(j \to i) \in \mathcal{E} ~:~ s_j = s\}|}.
\end{align}

\begin{construction}\label{thm:single_parent_construct}
    There exists a two-layer disentangled transformer $f_{\widetilde{\theta}} = \widetilde{\mathrm{TF}}_{\widetilde {\theta}}$ such that 
    \begin{align}
        f_{\widetilde{\theta}}(s_{1:T})_{s'} \approx \hat\pi_{s_{1:T}}(s' \mid s_T).
    \end{align}
\end{construction}

\begin{proof}
Set $A^{(1)}$ to be the (scaled) adjacency matrix of $\mathcal{G}$, i.e $A^{(1)}_{i,j} = \beta_1 \mathbf{1}(j = p(i))$, and $A^{(2)} = \beta_2 I_S$, where $\beta_1, \beta_2 \rightarrow \infty$. We will now show that the output of the disentangled transformer approximates $\hat\pi_{s_{1:T}}(\cdot \mid s_T)$.

\paragraph{First Attention.} Note that by the construction of $\widetilde{A}^{(1)}$, $\tilde X \widetilde{A}^{(1)} \tilde X^\top = A^{(1)}$, which is the scaled adjacency matrix of $\mathcal{G}$. If $i$ is not a root node (i.e. $i \in \overline{\calR}$, $p(i) \ne \emptyset$), then 
\begin{align}
\S(\tilde X \widetilde{A}^{(1)} \tilde X^\top )_{i,j} = \1(j=p(i))
\end{align}
so $i$ attends to its parent $p(i)$. Therefore, the output of the first attention is the token at position $p(i)$, i.e. $\attn(\tilde X; \widetilde{A}^{(1)})_i = \tilde x_{p(i)}$. The transformer then appends $\tilde x_{p(i)}$ to the residual stream of token $i$.


When $i$ is a root node (i.e. $i \in \calR$, $p(i) = \emptyset$), then for all $j$, $(\tilde X \widetilde{A}^{(1)} \tilde X^\top)_{ij} = 0$. Therefore after the softmax, $i$ will attend equally to all previous tokens:
\begin{align}
    \S(\tilde X \widetilde{A}^{(1)} \tilde X^\top )_{i, j} = \frac{1}{i} \qq{for all} j \le i.
\end{align}
Thus the first attention layer averages all of the tokens in the sequence: $\attn(\tilde X; \widetilde{A}^{(1)})_i = \frac{1}{i}\sum_{j \le i} \tilde x_j$. It then copies this average into the residual stream.


\paragraph{Second Attention.} We next show that the $T$th token attends to all prior tokens whose parents are equal to $s_T$. It then averages them and copies them into the residual stream.

After the first attention layer, the residual stream is $h_j^{(1)} = [ \tilde x_j, \attn(\tilde X; \widetilde{A}^{(1)})_j]^\top$. The second attention layer compares the $T$th token of the original sequence $\tilde{x}_T$ to the output of the first attention at all other positions. Explicitly, the attention pattern is equal to:
\begin{align}
    {h^{(1)}_T}^\top \widetilde{A}^{(2)} h_j^{(1)} &= \beta_2 \cdot \tilde x_T^\top \begin{bmatrix} A^{(2)} & 0_{S \times T} \\ 0_{T \times S} & 0_{T \times T}\end{bmatrix}\attn(\tilde X; \widetilde{A}^{(1)})_j = \beta_2 \cdot \begin{cases}
        \mathbf{1}(s_{p(i)} = s_T) & i \in \overline{\calR}\\
        \frac{1}{i}\sum_{j \le i}\mathbf{1}(s_{j} = s_T) & i \in \calR.\\
    \end{cases}
\end{align}
As $\beta_2 \to \infty$, the softmax converges to a hard max, and so the $T$th token attends equally to all tokens $i$ such that $s_{p(i)} = s_T$. The attention then averages all of these tokens, so the $T$th token in the residual stream is equal to $h^{(2)}_T = 
        \qty[\tilde{x}_T, \frac{1}{T}\sum_{j \le T} \tilde x_j,Z,\tilde{x}_T]$
where \begin{align}\label{eq:Z} 
    Z := \frac{\sum_{s_{p(i)}=s_{T}} \tilde x_i}{|\{i ~:~ s_{p(i)} = s_T\}|}
\end{align}
is the average of the tokens whose parent is equal to $s_T$.

\paragraph{Output Layer.} $W_O$ reads from the third block in this stream, which we denoted by $Z$ in \eqref{eq:Z} above. It then returns the token embedding of $Z$ which is equal to:
\begin{align}
    f_{\widetilde{\theta}}(s_{1:T}) = \frac{\sum_{s_{p(i)} = s_T} e_{s_{i}}}{|\{i ~:~ s_{p(i)} = s_T\}|} = \hat \pi_{s_{1:T}}(\cdot | s_T),
\end{align}
as desired.
\end{proof}

See \Cref{fig:trace_forward_pass} for a breakdown of this forward pass through the transformer for a specific sequence.

\begin{figure*}[t!]
\centering
    \includegraphics[width=0.4\textwidth]{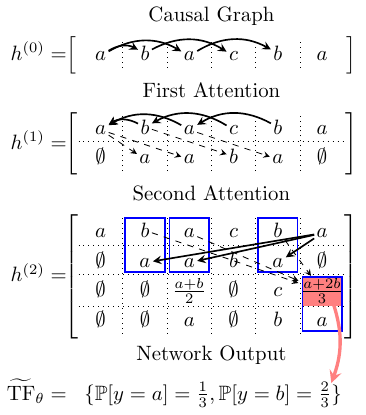}
    \caption{\textbf{Understanding the Forward Pass:} The solid arrows represent the causal graph $\mathcal{G}$ defined in \Cref{fig:graph_example} and $h^{(0)}$ denotes the unmodified input sequence. The first attention \emph{reverses} this causal pattern, as every token attends to its parent (solid arrows). It then \emph{appends} this parent token to the residual stream (dashed arrows). In the second attention layer, each token $i$ attends to all previous tokens $j$ whose parent token $p(j)$ has the same value, i.e. $s_i = s_{p(j)}$ (solid arrows), and appends the \emph{average of these tokens} into the residual stream (dashed arrows). Finally, the transformer returns the third entry in the last column (red box), which is the average of all of the tokens whose parent token has the same value as the last token.}
    \label{fig:trace_forward_pass}
\end{figure*}

\subsection{The Reduced Model}

Motivated by the sparsity pattern in \Cref{fig:induction_head_combined} and \Cref{eq:sparsity_pattern}, we consider training a simplified two-layer transformer architecture where the sparsity in \Cref{eq:sparsity_pattern} is fixed, and only $A^{(1)}$ and $A^{(2)}$ are trained. The transformer $\widetilde{\mathrm{TF}}_{\widetilde{\theta}}$ can be rewritten as the following reduced model:

\begin{lemma}\label{lem:rewrite_TF}
    Let $\theta = (A^{(1)}, A^{(2)})$, and let $\widetilde{\theta} = (\widetilde{A}^{(1)},\widetilde{A}^{(2)}, \widetilde{W}_O)$ be defined in \Cref{eq:sparsity_pattern}. Let $f_\theta = \widetilde{\mathrm{TF}}_{\widetilde{\theta}}$ be a two-layer disentangled transformer parameterized by $\theta$. Then if $\overline{X} = [\overline{x}_1,\ldots,\overline{x}_T]^T$ where $\overline{x_i} = e_{s_i}$,
    \begin{align}\label{eq:reduced_model}
        f_\theta(s_{1:T})  = \overline{X}^\top \S\qty(\S(\mathrm{MASK}(A^{(1)}))\overline{X}{A^{(2)}}^\top\overline{x}_T).
    \end{align}
\end{lemma}
Due to the masking operation, we can restrict $A^{(1)}$ to be lower diagonal. 

Our goal is to analyze the gradient descent dynamics of $f_\theta$ under the cross entropy loss. However, if the token $s'$ does not appear in $s_{1:T}$, then $f_\theta(s_{1:T})_{s'}$ is $0$ and the cross entropy loss is infinite. As such, we perturb the predictions by some small $\epsilon > 0$. The perturbed population loss is thus:
\begin{align}\label{eq:CE_loss}
    L(\theta) = -\underset{\pi,s_{1:T}}{\mathbb{E}}\qty[\sum_{s' \in [S]}\pi(s' | s_T) \log \qty(f_\theta(s_{1:T})_{s'} + \epsilon)]
\end{align}

In the following sections, we will study the gradient descent dynamics of the reduced model \eqref{eq:reduced_model} on the loss \eqref{eq:CE_loss}.

\section{Main Results}\label{sec:main_results}

\subsection{Training Algorithm}

Our training algorithm is stage-wise gradient descent on the population loss \eqref{eq:CE_loss} using the reduced model \eqref{eq:reduced_model}. The model is initialized at $A^{(1)} = 0_{T \times T}, A^{(2)} = \beta_0 I_{S \times S}$ for small initialization scale $\beta_0$. The first stage is gradient descent on $A^{(1)}$ with learning rate $\eta_1$ for $\tau_1$ timesteps. The second stage is gradient descent on $A^{(2)}$ with learning rate $\eta_2$ for $\tau_2$ timesteps. Pseudocode is given in \Cref{alg:training_alg}.


\renewcommand{\algorithmicrequire}{\textbf{Input:}}
\begin{algorithm}
    \caption{Training Algorithm}\label{alg:training_alg}
    \begin{algorithmic}
        \Require{init size $\beta_0$; learning rates $\eta_1, \eta_2$; times $\tau_1, \tau_2$}
        \State{Initialize $A^{(1)}(0) = 0_{T \times T}, A^{(2)}(0) = \beta_0 \cdot I_{S \times S}$}
        \For{$t=1,\dots,\tau_1$}
        \State{$A^{(1)}(t) \leftarrow A^{(1)}(t-1) - \eta_1\nabla_{A^{(1)}}L(\theta^{(t-1)})$} \Comment{Stage 1}
        \State{$\theta^{(t)} = (A^{(1)}(t), A^{(2)}(0))$}
        \EndFor
        \For{$t=\tau_1,\dots,\tau_1 + \tau_2$}
        \State{$A^{(2)}(t) \leftarrow A^{(2)}(t-1) - \eta_2\nabla_{A^{(2)}}L(\theta^{(t-1)})$} \Comment{Stage 2}
        \State{$\theta^{(t)} = (A^{(1)}(\tau_1), A^{(2)}(t))$}
        \EndFor
        \State{$\hat{\theta}\leftarrow \theta^{(\tau_1 + \tau_2)}$}
    \end{algorithmic}
    \textbf{Output:} $\hat{\theta}$.
\end{algorithm}

We require the following assumptions on the prior $P_\pi$:
\begin{assumption}[Assumptions on prior $P_\pi$.]\label{assume:pi_prior}
    There exists $\gamma > 0$ such that almost surely over the draw of $\pi$:
    \begin{itemize}
        \item (Transition lower bounded): $\min_{s, s'}\pi(s' \mid s) > \gamma / S$.
        \item (Non-degeneracy of chain): The chain does not immediately mix to the stationary measure $\mu_\pi$ in one step:
        \begin{align*}
            \textstyle \sum_s \norm{\pi(\cdot \mid s) - \mu_\pi(\cdot)}_2^2 \ge \gamma^2/S
        \end{align*}
        \item (Symmetry): For any permutation $\sigma$ on $[S]$, $\sigma^{{-}1}\pi\sigma {\stackrel{d}{=}} \pi$.
        \item (Constant mean): $\mathbb{E}_\pi[\pi] = \frac{1}{S}1_S1_S^\top $.
    \end{itemize}
\end{assumption}

The final two assumptions imply that the marginal distributions of $\pi(s' \mid s)$ are equal for any $s' \neq s$, and likewise for $\pi(s \mid s)$, and that these distributions have mean $1/S$. We remark that \Cref{assume:pi_prior} is satisfied with probability $0.99$ for some $\gamma = \Theta(1)$ when each row of $\pi$ is sampled i.i.d from a Dirichlet distribution with parameter $\alpha = \Theta(1)$.

Additionally, we assume that a non-vanishing fraction of nodes have a parent.
\begin{assumption}\label{assume:r}
    Let $r := \abs{\calR}/{T}$. Then $r \le 1 - \gamma$.
\end{assumption}

Throughout the proof, we let $C_{\gamma, S}$ denote an absolute constant that depends \emph{polynomially} on $\gamma^{-1}$ and $S$. If $A \le C_{\gamma, S}B$, then we write $A = O_{\gamma, S}(B)$ or $A \lesssim_{\gamma, S} B$. For convenience, we also drop the dependence on $\gamma, S$, and write $O(\cdot)$ or $\lesssim$.

\subsection{Main Theorem}

The minimum possible value for the (unperturbed) loss is the mean entropy of $\pi$, averaged over the prior $P_\pi$:
\begin{align}
    L^* := -\mathbb{E}_{\pi}\qty[\frac{1}{S}\sum_{s, s'} \pi(s' \mid s)\log \pi(s' \mid s)].
\end{align}
We also define the effective sequence length as follows:
\begin{definition}[Effective Sequence Length]\label{def:T_eff}
    Decompose $\mathcal{G} = \bigcup_{i=1}^k \mathcal{T}_i$ where $\mathcal{T}_i$ are disjoint trees. Let $L_i$ denote the number of leaves of tree $\mathcal{T}_i$. Then, $\TT := \frac{T}{\max_{i=1}^k L_i}$.
\end{definition}

The effective sequence length roughly captures the number of independent samples present in the sequence $s_{1:T}$, and is related to the mixing time of the process on $\mathcal{G}$. For both the Markov chain and in-context learning examples, we see that $L_i = 1$ and thus $\TT = T$. 

Our main theorem is the following:
\begin{theorem}[Guarantee for \Cref{alg:training_alg}]\label{thm:main_thm}
    Assume that the effective sequence length satisfies $\TT \ge \text{poly}(\gamma^{-1}, S)$. There exist $\epsilon, \eta_1, \eta_2, \tau_1, \tau_2$ such that the output of \Cref{alg:training_alg}, $\hat\theta = (\hat A^{(1)}, \hat A^{(2)})$, satisfies
     \begin{align}
        L(\hat\theta) - L^* \lesssim \frac{\log T}{\TT^{c\gamma}}\qand \S(\hat A^{(1)})_{i, p(i)} \ge 1 - O\qty(\frac{1}{T})~\text{for}~i \in \overline{\mathcal{R}},
    \end{align}
    for some constant $c > 0$ (independent of $\gamma, S$).
\end{theorem}

\Cref{alg:training_alg} thus approximately minimizes the loss by encoding the adjacency matrix of $\mathcal{G}$ in the first attention layer $\hat A^{(1)}$. Furthermore, we show that the trained model $\hat \theta$ achieves good prediction on transitions $\pi$ which are out of distribution:

\begin{theorem}[OOD Generalization]\label{thm:OOD}
Let $\tilde \pi$ have transition lower bounded as $\min_{s, s'}\tilde \pi(s' \mid s) \ge \gamma/S$ and let $\htheta$ be the trained model from \Cref{thm:main_thm}. Let $s_{1:T}$ be generated by steps 2-4 of \Cref{task:single_parent}. Then with probability at least $0.99$ over the draw of $s_{1:T}$,
\begin{align}
    \sup_{s'}\abs{f_{\hat \theta}(s_{1:T})_{s'} - \tilde \pi(s' \mid s_T)} \lesssim \frac{\log T}{\TT^{c\gamma}}.
\end{align}
\end{theorem}

We remark that the only assumption needed on $\tilde \pi$ is the lower bound on the transition; it does not need to be close to typical draw from the prior $P_\pi$.

\section{Proof Sketch}

\subsection{Stage 1: Learning the Causal Graph}

The first step of the proof is to show that during the first stage of training, the first attention layer $A^{(1)}$ learns the latent causal graph $\mathcal{G}$. 

\subsubsection{The Oracle Algorithm}
We begin by describing an efficient algorithm for learning the graph $\mathcal{G}$. The goal is to recover the parent node $p(i)$ for each $i$. The key idea is that as a result of the data generating process, $s_{p(i)}$ is the node which maximizes mutual information with $s_i$. 

We briefly recall the definition of an $f$-divergence.
\begin{definition}
    Let $f$ be a convex function with $f(1) = 0$. The $f$-divergence between two probability distributions $P, Q$ on state space $\mathcal{X}$ is defined as
    \begin{align}
        D_f(P || Q) := \sum_{x \in \mathcal{X}}Q(x)f\qty(\frac{P(x)}{Q(x)}).
    \end{align}
    The $f$ mutual information between two random variables $Y, Z$, denoted by $I_{f}(Y; Z)$, is
    \begin{align}
        I_{f}(Y; Z) := D_f(P_{Y,Z} || P_Y \otimes P_Z).
    \end{align}
\end{definition}
Given a latent variable $C$, the conditional mutual information $I_f(Y; Z \mid C)$ is defined as
\begin{align}
    I_f(Y; Z \mid C) := \mathbb{E}_{C}\qty[D_f(P_{(Y,Z) \mid C} || P_{Y \mid C} \otimes P_{Z \mid C})].
\end{align}

Information measures admit a \emph{data processing inequality}:

\begin{lemma}[Data Processing Inequality] Let $I_f$ be an information measure, and let $W \rightarrow Y \rightarrow Z$ be a Markov chain. Then $I_f(W; Z) \le I_f(Y; Z).$
\end{lemma}
The data processing inequality suggests an efficient algorithm for recovering $\mathcal{G}$. For non-root nodes $i \in \overline{\calR}$, $s_j \rightarrow s_{p(i)} \rightarrow s_i$ form a Markov chain conditioned on $\pi$. Therefore by the data processing inequality, $p(i) \in \arg\max_{j < i} I_f(s_i; s_j \mid \pi)$. Otherwise, if $i \in \mathcal{R}$, $s_j$ and $s_i$ are independent given $\pi$, and thus $I_f(s_i; s_j \mid \pi) = 0$. To recover the graph $\mathcal{G}$, one can compute the conditional mutual informations $I_f(s_i; s_j \mid \pi)$. If $I_f(s_i; s_j \mid \pi) = 0$ for all $j < i$, then $i$ is a root node. Otherwise, $p(i) = \arg\max_{j < i} I_f(s_i; s_j \mid \pi)$. Pseudocode for this oracle algorithm is given in \Cref{alg:oracle_alg}.

\begin{algorithm}
    \caption{Oracle Algorithm}\label{alg:oracle_alg}
    \begin{algorithmic}
        \State{$\mathcal{E} \leftarrow \emptyset$}
        \For{$i = 1, \dots, T-1$:}
        \If{$\max_{j < i}I_f(s_i; s_j \mid \pi) > 0$}
        \State{$p(i) \leftarrow \arg\max_{j < i} I_f(s_i, s_j \mid \pi)$.}
        \State{$\mathcal{E} \leftarrow \mathcal{E} \cup \{(p(i) \to i)\}$.}
        \EndIf
        \EndFor
    \end{algorithmic}
\end{algorithm}


We remark that \Cref{alg:oracle_alg} is a special case of the celebrated Chow-Liu algorithm~\citep{chow1968approximating}, when a topological ordering of the tree is known a priori: the tree consisting of edges $(p(i) \to i)$ where $p(i) = \arg\max_{j < i} I_{f}(s_i; s_j \mid \pi)$ is indeed the max-weight spanning forest when the edge weights are the conditional mutual informations.

\subsubsection{The Gradient Descent Dynamics}


We next compute the gradient with respect to $A^{(1)}$. Let $A_i^{(1)} \in \R^i$ denote the $i$th row of $A^{(1)}$ (restricted to the first $i$ entries, since $A^{(1)}$ is lower triangular). Define $J:\R^k \rightarrow \R^{k \times k}$ by $J(v) = \diag(v) - vv^\top $; $J$ is the Jacobian of the softmax function $\S$, in that $\nabla_u \S(u) = J(\S(u)).$

The following lemma computes the gradient with respect to $A^{(1)}$; a heuristic derivation is deferred to \Cref{sec:lemma1_sketch}.
\begin{lemma}\label{lem:gradient_informal}
\begin{align}
    \nabla_{A^{(1)}_i} L(\theta) = - \frac{\beta_0}{ST}J\qty(\S(A_i^{(1)}))\qty(g_i + O(\TT^{-1/2})),
\end{align}
where the $j$th entry of $g_i$, $g_{i,j}$, is
\begin{align}
    g_{i, j} := \E_{\pi}\qty[\sum_{s, s'}\frac{\pi(s' \mid s)}{\mu_\pi(s')}\mathbb{P}_X[s_j = s, x_i = s']] - 1.
\end{align}
\end{lemma}

\begin{figure*}[t!]
\centering
    \includegraphics[width=0.9\textwidth]{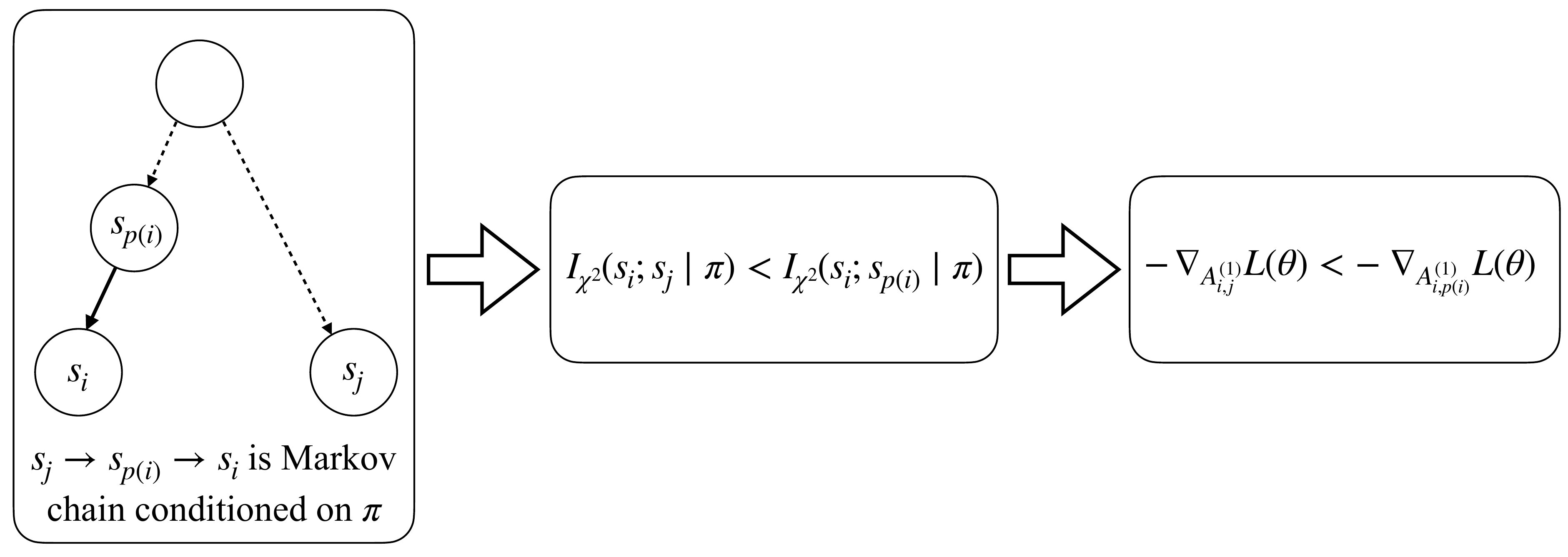}
    \caption{By the data processing inequality, $A^{(1)}_{i, p(i)}$ grows faster than $A^{(1)}_{i, j}$.}
    \label{fig:pf_sketch}
\end{figure*}

For non-root nodes $i \in \overline{\mathcal{R}}$, $(s_i, s_{p(i)})$ has joint distribution $\mathbb{P}[s_i = s', s_{p(i)} = s] = \mu_\pi(s)\pi(s' | s)$, and thus $g_{i, p(i)}$ is
\begin{align}
g_{i,p(i)} = \E_\pi\qty[\sum_{s, s'}\frac{\pi(s' \mid s)^2\mu_\pi(s)^2}{\mu_\pi(s')\mu_\pi(s)} - 1].
\end{align}
It turns out that this expression is exactly equal to the $\chi^2$-mutual information, $I_{\chi^2}$, between $s_i$ and $s_{p(i)}$ conditioned on $\pi$. The $\chi^2$-divergence is the $f$-divergence obtained by setting $f(z) = (z-1)^2$. Therefore
\begin{align}
    g_{i,p(i)} = I_{\chi^2}(s_i; s_{p(i)} \mid \pi).
\end{align}
By Cauchy-Schwarz, we can also upper bound $g_{i,j}$ by the sum of two $\chi^2$-mutual informations:
\begin{align}
    g_{i, j} \le \frac12 I_{\chi^2}(s_i; s_{p(i)} \mid \pi) + \frac12 I_{\chi^2}(s_i; s_j \mid \pi).
\end{align}
Applying the data processing inequality\footnote{By the assumptions on the prior $P_\pi$ (\Cref{assume:pi_prior}), the data processing inequality is indeed strict.}, we obtain that for $j \neq p(i)$
\begin{align}
    g_{i, j} < I_{\chi^2}(s_i; s_{p(i)} \mid \pi) = g_{i, p(i)}.
\end{align}
Therefore $g_{i,j}$ is maximized at $j=p(i)$, and the gradient is aligned with the adjacency matrix of the causal graph $\mathcal{G}$. In fact, the gradient descent dynamics mimic \Cref{alg:oracle_alg}!

Maintaining the inductive hypothesis that $\arg\max_j A_{i,j}^{(1)} = p(i)$, we see by the gradient formula in \Cref{lem:gradient_informal} that $\arg\max\qty[-\nabla_{A_i^{(1)}}L(\theta)_j] = p(i)$. Thus $A_{i, p(i)}^{(1)}$ continues to grow faster than the other entries throughout stage 1. This growth continues until $\S(A^{(1)})_{i, p(i)} \approx 1.$

For root nodes $i \in \calR$, $i$ is independent of $j$. Since both $s_i$ and $s_j$ have the marginal $\mu_\pi$, one has
\begin{align}
    g_{i,j} = \mathbb{E}_\pi\qty[\sum_{s, s'}\frac{\pi(s' \mid s)}{\mu_\pi(s')}\mu_\pi(s')\mu_\pi(s)] - 1 = 0
\end{align}
and thus $\nabla_{A_{i}^{(1)}}L(\theta) \approx 0$. Therefore at the end of stage 1, $\S(A^{(1)})_{i, j} \approx \frac{1}{i}$ for all $j < i$.

Altogether, at the end of stage $1$, $A^{(1)}$ satisfies
\begin{align}
    \S(A^{(1)})_{i, j} \approx \begin{cases} \mathbf{1}(j = p(i)) & i \in \overline{\calR}\\
        \frac{1}{i} & i \in \calR
    \end{cases}.
\end{align}
The  precise quantitative bound is given in \Cref{cor:end_of_stage_1}, and requires controlling the various error terms throughout multiple steps of gradient descent.

\subsection{Stage 2: Decreasing the Loss}
We next show that during the second stage, $A^{(2)}$ grows large in the direction $I_S - \frac{1}{S}1_S1_S^\top $. By a symmetry argument, one can show that $\nabla_{A^{(2)}}L(\theta)$ is proportional to $I_S - \frac{1}{S}1_S1_S^\top $. Writing $A^{(2)} = \beta I_S + \frac{\beta - \beta_0}{S}1_S1_S^\top $, it suffices to show that $\nabla_\beta L(\theta) < 0$.

In \Cref{thm:stage2}, we show that $-\nabla_\beta L(\theta)$ can be approximated by a quantity which is an $f$-mutual information for some convex $f$ defined in terms of $\beta$. We show that this quantity is strictly positive (\Cref{lem:bound_g_beta}) until $\beta = \Theta(\log \TT)$. Thus at the end of stage 2, $\beta = \Theta(\log \TT)$.

To conclude the proof of \Cref{thm:main_thm}, we must show that $f_{\hat \theta}(X; s)_{s'} \approx \pi(s' \mid s)$. Indeed, \Cref{lem:f_to_pi} shows that
\begin{align}
    \abs{f_{\hat \theta}(X; s)_{s'} - \pi(s' \mid s)} \le \exp(-\Theta(\beta)) = \TT^{-\Theta(1)},
\end{align}
which implies the desired bound on the loss.

\section{Causal Graphs with Multiple Parents}\label{sec:multiple_parents}

We next consider a generalization of \Cref{task:single_parent}. Let $\mathcal{G}$ be a directed acyclic graph over the vertex set $[T+1]$. For each node $i \in [T+1]$, we assume that the set of parent nodes $p(i) \subset [i-1]$ satisfy the property that either $p(i) = \emptyset$ or $\abs{p(i)} = k$. If $p(i) \neq \emptyset$, we write $p(i) = \{p(i)_1, \dots, p(i)_k\}$, where $p(i)_1 < \cdots < p(i)_k$. As before, let $\calR = \{i \in [T+1] : p(i) = \emptyset\}$ be the root nodes. We additionally assume that $T+1 \not\in \calR$.

We now consider $k$-parent transition tensors $\pi$: For any $a_1, \dots, a_k \in [S]$, $\pi(\cdot | a_1, \dots, a_k)$ is a probability distribution over $[S]$. Let $P^k_{\pi}$ be a prior over such $\pi$. Each sequence is now generated as follows:

\begin{task}[Graphs with Multiple Parents]\label{task:multi_parent}~
\begin{enumerate}
    \item Draw $\pi \sim P^k_\pi$.
    \item For $i = 1, \dots, T+1$, if $p(i) = \emptyset$, sample $s_i \sim \mathrm{Unif}([S])$. Otherwise, sample $s_i \sim \pi(\cdot | s_{p(i)_1}, \dots, s_{p(i)_k})$
    \item Return the input $x = s_{1:T}$ and the target $y = s_{T+1}$.
\end{enumerate}
\end{task}

\paragraph{Example.} One example of \Cref{task:multi_parent} is learning in-context $n$-grams. In an $n$-gram model, each token only depends on the prior $n-1$ tokens in the sequence. This $n$-gram model can be obtained by setting $k = n-1$, letting the root nodes be $\mathcal{R} = [n-1]$, and choosing the parent sets $p(i) = \{i - n + 1, i - n + 2, \dots, i-1\}$ for $i \ge n$. The conditional density $\mathbb{P}(s_{k + n} \mid s_{k+1:k+n-1})$ is then just the transition $\pi(s_{k+n} \mid s_{k+1}, \dots, s_{k + n - 1})$; the goal is to estimate this transition in-context, by first learning that all sequences share the same $n$-gram causal structure.

Given a sequence $s_{1:T}$, a good estimate for the transition $\pi$ is the empirical transition $\hat \pi_{s_{1:T}}(s' \mid a_1, \dots, a_k)$, defined as
\begin{align}
    \frac{\abs{\{i : s_i = s', s_{p(i)_1} = a_1, \dots, s_{p(i)_k} = a_k\}}}{\abs{\{i : s_{p(i)_1} = a_1, \dots, s_{p(i)_k} = a_k\}}}
\end{align}

We explicitly construct a two-layer transformer with $k$ heads in the first layer that approximately expresses this empirical transition.

\begin{construction}\label{thm:multi_parent_construction}
    There exists a two attention layer transformer $f_{\tilde \theta}$ with $k$ heads such that
    \begin{align}
        f_{\tilde \theta}(s_{1:T})_{s'} \approx \hat \pi_{s_{1:T}} (s' \mid s_{p(T+1)_1}, \dots, s_{p(T+1)_k})
    \end{align}
\end{construction}

\Cref{thm:multi_parent_construction} is deferred to \Cref{sec:multiparent}. At a high level, the $\ell$th head in the first layer copies $p(i)_\ell$ to the residual stream of $i$, and copies $p(T+1)_\ell$ to the residual stream of $T$; the second attention head compares these tuples of parents, and thus attends to tokens $i$ where $s_{p(i)_\ell} = s_{p(T+1)_\ell}$ for all $\ell \in [k]$. 

\begin{figure}[t!]
    \centering
    \subfigure[3-gram where each position $i$ attends to $i-1,i-2$.]{\includegraphics[height=0.19\textwidth]{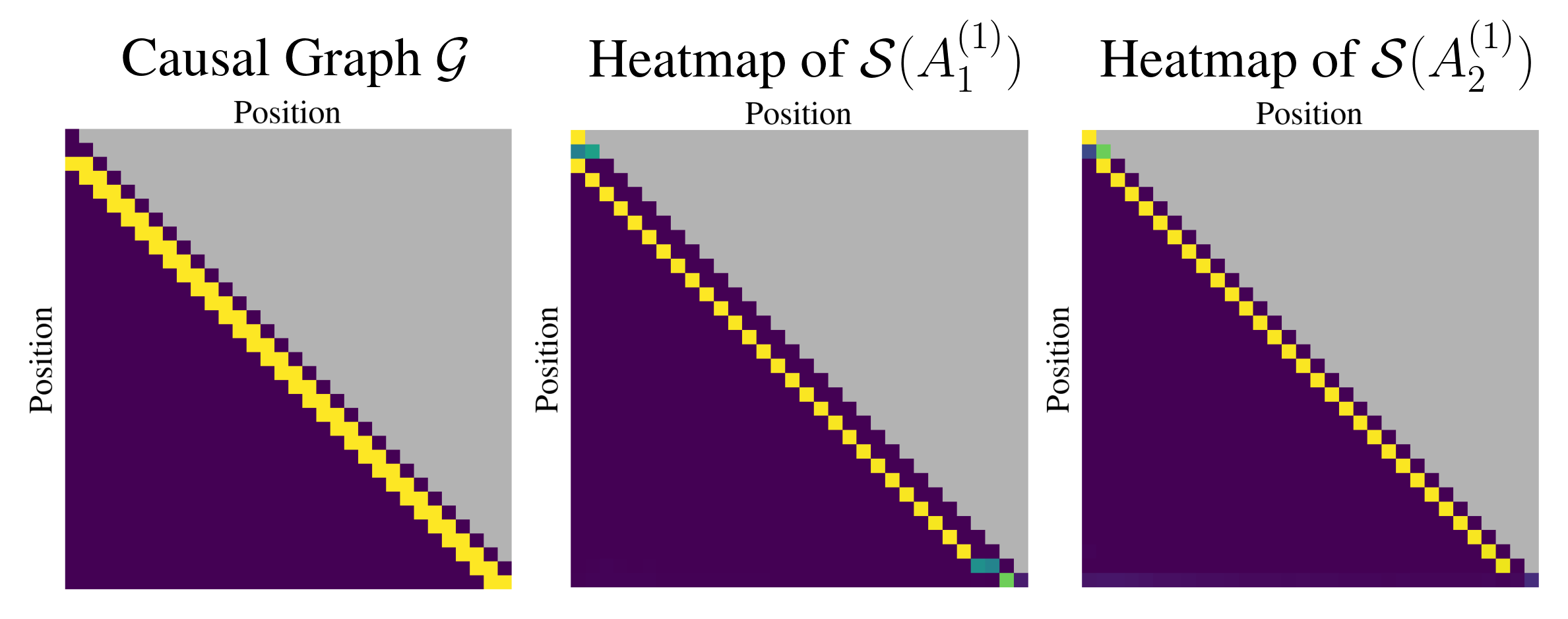}}
    \subfigure[Each position $i$ attends to $i-1$ and $\lfloor \frac{i-1}{2} \rfloor$]{\includegraphics[height=0.19\textwidth]{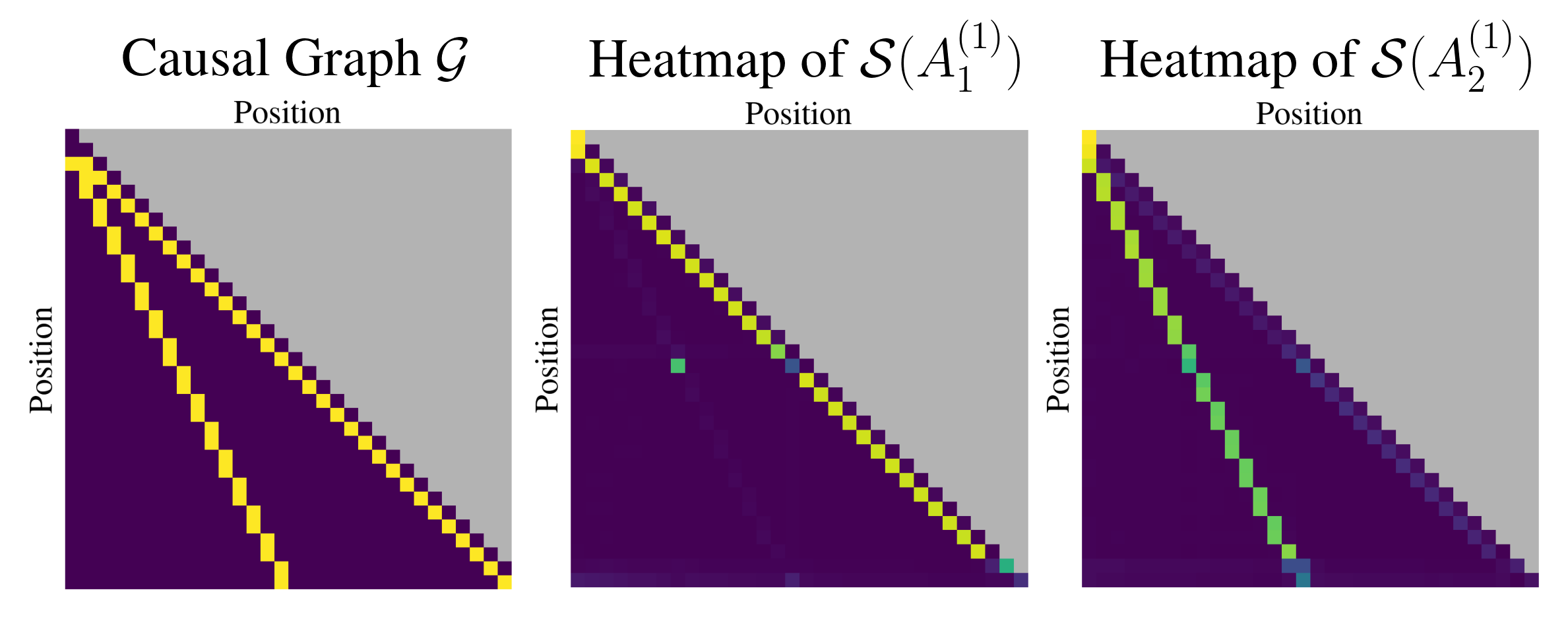}}
    \subfigure[4-gram where each position $i$ attends to $i-1,i-2, i-3$.]{\includegraphics[height=0.19\textwidth]{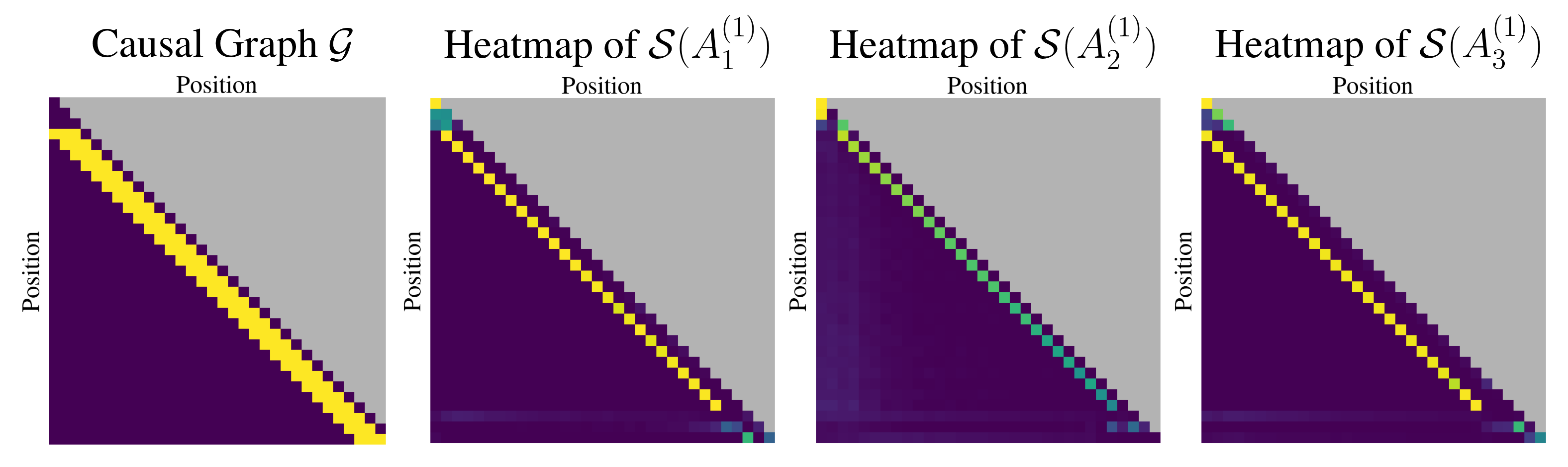}}
    \caption{\textbf{Multiple Parents:} We show three examples of trained transformers on \Cref{task:multi_parent} with $k=2,2,3$ respectively. The left column shows the adjacency matrix of the causal graphs $\mathcal{G}$. To their right, we plot the attention patterns $\mathcal{S}(A^{(1)}_i)$ for each head $i$ where $A^{(1)}_i$ is the position-position component of $\widetilde{A}^{(1)}_i$. We see that each attention head learns a single set of parents in the causal graph $\mathcal{G}$, which agrees with \Cref{thm:multi_parent_construction}. See \Cref{fig:appendix_multi_parent} for plots of the full matrices $\widetilde{A}^{(1)}_i$.}
    \label{fig:multi_parent_smax}
\end{figure}

In \Cref{fig:multi_parent_smax,fig:appendix_multi_parent}, we show empirically that transformers trained on \Cref{task:multi_parent} for varying latent graphs $\mathcal{G}$ indeed converge to such a construction. The challenge, however, in analyzing the gradient descent dynamics is that there are multiple attention heads each of which attends to a different parent. The dynamics must thus break the symmetry between the multiple heads. Analyzing the optimization dynamics of a multi-head transformer for solving \Cref{task:multi_parent} is thus a very interesting direction for future work.

\section*{Acknowledgements}
EN acknowledges support from a National Defense Science \& Engineering Graduate Fellowship.
AD acknowledges support from a NSF Graduate Research Fellowship. EN, AD, and JDL acknowledge support of the ARO under MURI Award W911NF-11-1-0304, the Sloan Research Fellowship,
NSF CCF 2002272, NSF IIS 2107304, NSF CIF 2212262, ONR Young Investigator Award, and
NSF CAREER Award 2144994


\bibliography{references}

\newpage
\appendix
\onecolumn

\section{Disentangled Transformer Equivalence}\label{sec:aux_proofs}

\begin{theorem}\label{thm:disentangled_TF_equivalence}
    For any transformer $\mathrm{TF}_\theta$ with any hidden dimension $d$, there exists a disentangled transformer $\widetilde{\mathrm{TF}}_{\tilde\theta}$ with the same depth and number of heads such that $\mathrm{TF}_\theta(s_{1:T}) = \widetilde{\mathrm{TF}}_{\tilde\theta}(s_{1:T})$ for any input sequence $s_{1:T} \in [S]^T $. Likewise, for any disentangled transformer $\widetilde{\mathrm{TF}}_{\tilde\theta}$, there exists a transformer $\mathrm{TF}_\theta$ with the same depth and number of heads and with hidden dimension $d^{(L)}$ such that $\mathrm{TF}_\theta(s_{1:T}) = \widetilde{\mathrm{TF}}_{\tilde\theta}(s_{1:T})$ for any $s_{1:T} \in [S]^T $.
\end{theorem}

\begin{proof}
    Let
    \begin{align*}
        \theta = \{(Q_i^{(\ell)}, K_i^{(\ell)}, V_i^{(\ell)})_{\ell \in [L], i \in [m_\ell]}\} \cup \{E,P,W_O\} \qand \tilde \theta = \{A_i^{(\ell)}\}_{\ell \in [L], i \in [m_\ell]} \cup \{\tilde W_O\}.
    \end{align*}
    Note that the reverse direction is clear as any disentangled transformer is also a transformer with hidden dimension $d_L$:
    \begin{align*}
        E &= \begin{bmatrix} I_S \\ 0_{(d_L - S) \times S} \end{bmatrix} \in \R^{d_L \times S} & P &= \begin{bmatrix} 0_{S \times T} \\ I_T \\ 0_{(d_L - d) \times T} \end{bmatrix} \in \R^{d_L \times T} & W_O &= \tilde W_O \\
        Q_i^{(\ell)} &= \begin{bmatrix} A_i^{(\ell)} & 0 \\ 0 & 0 \end{bmatrix} \in \R^{d_L,d_L} &  K_i^{(\ell)} &= \begin{bmatrix} I_{d_\ell} & 0 \\ 0 & 0 \end{bmatrix} \in \R^{d_L,d_L} & V_i^{(\ell)} &=
        \begin{bmatrix}
        0_{i \cdot d_\ell \times d_\ell} &  \\
        I_{d_\ell} & 0_{d_{L} \times (d_L - d_\ell)} \\
        0_{(d_L - (i+1) \cdot d_\ell) \times d_\ell} & 
        \end{bmatrix}.
    \end{align*}
    We will prove that every transformer $\theta$ can be represented by a disentangled transformer $\tilde \theta$. We will begin by defining a sequence of matrices $Z^{(\ell)} \in \R^{d \times d_\ell}$ for $\ell \in \{0,\ldots,L\}$. Let $Z^{(0)} := [E, P] \in \R^{d \times d_0}$, and for $\ell > 1$ let
    \begin{align*}
        Z^{(\ell)} :=
        \begin{bmatrix}
            Z^{(\ell-1)} & V_1^{(\ell)} Z^{(\ell-1)} & \cdots & V_{m_\ell}^{(\ell)} Z^{(\ell-1)}
        \end{bmatrix} \in \R^{d \times d_\ell}.
    \end{align*}
    Then we define
    \begin{align*}
        A_i^{(\ell)} := (Z^{(\ell-1)})^\top  Q_i^{(\ell)} (V_i^{(\ell)})^\top  Z^{(\ell-1)} \in \R^{d_{\ell-1} \times d_{\ell-1}} \qq{and} \tilde W_O = W_O Z^{(L)}.
    \end{align*}
    We will prove by induction that for any sequence $s_{1:T}$, $h^{(\ell)} = \tilde h^{(\ell)} (Z^{(\ell)})^\top $ for $\ell = 0,\ldots,L$ where $\{h^{(\ell)}\}$ is the residual stream of $\mathrm{TF}_\theta$ and $\{\tilde h^{(\ell)}\}$ is the residual stream of $\widetilde{\mathrm{TF}}_{\tilde \theta}$. First, when $\ell = 0$ we have that
    \begin{align*}
        h^{(0)}_i = E e_{s_i} + P e_i = \begin{bmatrix} E & P \end{bmatrix} \begin{bmatrix} e_{s_i} \\ e_i \end{bmatrix} = Z^{(0)} \tilde h^{(0)}_i.
    \end{align*}
    Next, assume the result for $\ell-1 \ge 0$. Then
    \begin{align*}
        h^{(\ell)}
        &= h^{(\ell - 1)} + \sum_{i=1}^{m_\ell} \attn\qty(h^{(\ell - 1)}; Q^{(\ell)}_{i}{K^{(\ell)}_i}^\top ){V^{(\ell)}_i}^\top  \\
        &= \tilde h^{(\ell-1)} (Z^{(\ell-1)})^\top  + \sum_{i=1}^{m_\ell} \mathrm{attn}\qty(\tilde h^{(\ell-1)} (Z^{(\ell-1)})^\top ; Q_i^{(\ell)} (K_i^{(\ell)})^\top ) (V_i^{(\ell)})^\top  \\
        &= \tilde h^{(\ell-1)} (Z^{(\ell-1)})^\top  + \sum_{i=1}^{m_\ell} \mathrm{attn}\qty(\tilde h^{(\ell-1)}; (Z^{(\ell-1)})^\top  Q_i^{(\ell)} (K_i^{(\ell)})^\top  Z^{(\ell-1)}) (Z^{(\ell-1)})^\top  (V_i^{(\ell)})^\top  \\
        &= \tilde h^{(\ell-1)} (Z^{(\ell-1)})^\top  + \sum_{i=1}^{m_\ell} \mathrm{attn}\qty(\tilde h^{(\ell-1)};A_i^{(\ell)}) (V_i^{(\ell)} Z^{(\ell-1)})^\top  \\
        &= \tilde h^{(\ell)} (Z^{(\ell)})^\top 
    \end{align*}
    which completes the induction. Therefore,
    \begin{align*}
        \mathrm{TF}_\theta(s_{1:T}) = h^{(L)}W_O^\top  = \tilde h^{(L)} (Z^{(\ell)})^\top  W_O^\top  = h^{(L)} (W_O Z^{(\ell)})^\top  = h^{(L)} \tilde W_O^\top  = \widetilde{\mathrm{TF}}_\theta(s_{1:T})
    \end{align*}
    which completes the proof.
\end{proof}

\textbf{Example: Single Head Transformer.} As an example, let us walk through the construction for a single-head transformer. A single layer attention-only transformer with only one head can be written as
\begin{align*}
    h^{(0)} &= \mathrm{embed}\qty(s_{1:T}; (E, P))\\
    \mathrm{TF}_\theta(s_{1:T}) &= \qty(h^{(0)} + \attn\qty(h^{(0)}; QK^\top)V^\top)W_O^\top.
\end{align*}
Recall that the input to the disentangled transformer is
\begin{align*}
    \tilde X = \begin{bmatrix} \tilde x_1, \dots, \tilde x_T \end{bmatrix}^\top,
\end{align*}
where $\tilde x_t =\begin{bmatrix}e_{s_t}, e_t\end{bmatrix} \in \mathbb{R}^{d_0} = \mathbb{R}^{S + T}$. The input to the regular transformer, $h^{(0)}$, can then be written as
\begin{align*}
    h^{(0)} = \tilde X \begin{bmatrix} E^\top \\ P^\top \end{bmatrix}.
\end{align*}
Define $Z = \begin{bmatrix} E & P \end{bmatrix}$, so that $h^{(0)} = \tilde X Z^\top$. Set the weights $\tilde A, \tilde W_O$ of the disentangled transformer as
\begin{align*}
    \tilde A = Z^\top QK^\top Z \qand \tilde W_O = W_O\begin{bmatrix} Z & VZ \end{bmatrix}.
\end{align*}
The output of the disentangled transformer is then
\begin{align*}
    \widetilde{\mathrm{TF}}_{\tilde\theta}(s_{1:T}) &= \begin{bmatrix} \tilde X, & \attn(\tilde X; \tilde A)\end{bmatrix} \tilde W_O^\top \\
    &= \begin{bmatrix} \tilde X, & \attn(\tilde X; Z^\top QK^\top Z)\end{bmatrix} \tilde W_O^\top \\
    &= \begin{bmatrix}\tilde X, & \S\qty(\mathrm{MASK}\qty(\tilde X Z^\top QK^\top Z \tilde X^\top))\tilde X \end{bmatrix} \begin{bmatrix} Z^\top \\ Z^\top V^\top \end{bmatrix} W_O^\top\\
    &= \begin{bmatrix}\tilde X Z^\top , & \S\qty(\mathrm{MASK}\qty(\tilde X Z^\top QK^\top Z \tilde X^\top))\tilde X Z^\top \end{bmatrix} \begin{bmatrix} I_d \\ V^\top  \end{bmatrix} W_O^\top \\
    &= \begin{bmatrix} h^{(0)}, & \attn(h^{(0)}; QK^\top) \end{bmatrix} \begin{bmatrix} I_d \\ V^\top  \end{bmatrix} W_O^\top \\
    &= \qty(h^{(0)} + \attn(h^{(0)}; QK^\top)V^\top) W_O^\top\\
    &= \mathrm{TF}_\theta(s_{1:T}),
\end{align*}
as desired.

\section{Multiple Parents Construction}\label{sec:multiparent}

We now present \Cref{thm:multi_parent_construction}.

\begin{proof}
Let $\tilde X \in \mathbb{R}^{T \times d}$ be the embedding of the sequence. Recall that the $\ell$th attention block is of the form
\begin{align*}
    \attn(\tilde X; \widetilde{A}^{(1)}_\ell) := \S(\tilde X\widetilde{A}^{(1)}_\ell \tilde X^\top )\tilde X \in \mathbb{R}^{T \times d},
\end{align*}
and
\begin{align*}
    h^{(1)}_i = \begin{bmatrix}
                       x_i             \\
                       \attn(\tilde X; \widetilde{A}^{(1)}_1)_i \\
                       \vdots          \\
                       \attn(\tilde X; \widetilde{A}^{(1)}_{k})_{i}.
                   \end{bmatrix} \in \R^{(k+1)d}.
\end{align*}
The $\ell$th attention head performs two roles. For $i < T$, it copies $p(i)_\ell$ to the residual stream of $i$. Additionally, it copies $p(T+1)_\ell$ to the residual stream of $T$.

Formally, let $\widetilde{A}^{(1)}_\ell$ follow the same sparsity pattern as the construction in \Cref{thm:single_parent_construct}, where only the position-position block $A_\ell^{(1)}$ is nonzero, and on this block let
\begin{align*}
    (A^{(1)}_\ell)_{ij} = \beta_1 \cdot \begin{cases}
        \mathbf{1}(j = p(i)_\ell) & i < T\\
        \mathbf{1}(j = p(T+1)_\ell) & i = T
    \end{cases}.
\end{align*}
Taking $\beta_1 \to \infty$, the output of this attention block is
\begin{align*}
    \attn(\tilde X; \widetilde{A}^{(1)}_{\ell})_i = \begin{cases}
        \tilde x_{p(i)_\ell} & i \in \overline{\cal R} \setminus \{T\}\\
        \tilde x_{p(T+1)_\ell} & i = T
    \end{cases}.
\end{align*}

We let $\widetilde{A}^{(2)} \in \R^{(k+1)d \times (k+1)d}$ be the block diagonal matrix which compares the $\attn(\tilde X; \widetilde{A}_\ell^{(1)})_i$ components of the residual streams of $h^{(1)}_i$ to each other via their token embeddings.

Formally, we let
\begin{align*}
    \widetilde{A}^{(2)} = \begin{bmatrix}
        0_{d \times d} & 0_{d \times d} & 0_{d \times d} & \cdots & 0_{d \times d}\\
        0_{d \times d} & A^{(2)}_1 & 0_{d \times d} & \cdots & 0_{d \times d} \\
        0_{d \times d} & 0_{d \times d} & A^{(2)}_2 & \cdots & 0_{d \times d} \\
        \vdots& \vdots&  \vdots & \ddots & \vdots \\
        0_{d \times d} & 0_{d \times d} & 0_{d \times d} & \cdots & A^{(2)}_k
    \end{bmatrix}
\end{align*}
where each $A_k^{(2)} \in R^{d \times d}$ is
\begin{align*}
    A_k^{(2)} = \beta_2\begin{bmatrix}
        I_{S \times S} & 0_{S \times T}\\
        0_{T \times S} & 0_{T \times T}
    \end{bmatrix}.
\end{align*}
We thus have, for $i \in \overline{\calR} \setminus \{T\}$.
\begin{align*}
    {h^{(1)}_i}^\top \widetilde{A}^{(2)}h^{(1)}_T & = \beta_2\sum_{\ell=1}^k \qty(\attn(\tilde X; \widetilde{A}^{(1)}_\ell)_i)^\top  A_k^{(2)}\attn(\tilde X; \widetilde{A}^{(1)}_{\ell})_T \\
    & = \beta_2\cdot\sum_{\ell=1}^k\mathbf{1}(s_{p(i)_\ell} = s_{p(T+1)_\ell})
\end{align*}
Taking $\beta_2 \rightarrow \infty$, the softmax converges to a uniform distribution over tokens where ${h^{(1)}_i}^\top A^{(2)}h^{(1)}_T$ is maximized. These are the tokens $i$ in which $s_{p(i)_\ell} = s_{p(T+1)_\ell}$ for all $\ell$, along with the token $T$\footnote{It is possible for certain root nodes at the beginning of the sequence to be included, but this will be a vanishing fraction of tokens for typical sequences}. Thus
\begin{align*}
    \S\qty(h^{(1)}\widetilde{A}^{(2)}h^{(1)}_T)_i &\approx \frac{\mathbf{1}_{i = T} + \mathbf{1}\qty(s_{p(i)_1} = s_{p(T+1)_1}, \cdots, s_{p(i)_k} = s_{p(T+1)_k})}{1 + \sum_{j < T} \mathbf{1}\qty(s_{p(j)_1} = s_{p(T+1)_1}, \cdots, s_{p(j)_k} = s_{p(T+1)_k})}.
\end{align*}
Finally, choose $W_O$ to output the token embedding of the $x_i$ block of $h^{(1)}(X)_i$, so that $h^{(1)}(X)W_O = e_{s_i}$. We then have that
\begin{align*}
    f_{\hat \theta}(s_{1:T})_{s'} &= \sum_i \mathbf{1}(s_i = s')\cdot \S\qty(h^{(1)}\widetilde{A}^{(2)}h^{(1)}_T)_i \\
    & \approx \frac{\mathbf{1}(s_T = s') + \sum_{i < T}\mathbf{1}\qty(s_i = s', s_{p(i)_1} = s_{p(T+1)_1}, \cdots, s_{p(i)_k} = s_{p(T+1)_k})}{1 + \sum_{j < T} \mathbf{1}\qty(s_{p(j)_1} = s_{p(T+1)_1}, \cdots, s_{p(j)_k} = s_{p(T+1)_k})} \\
    & \approx \frac{\sum_{i}\mathbf{1}\qty(s_i = s', s_{p(i)_1} = s_{p(T+1)_1}, \cdots, s_{p(i)_k} = s_{p(T+1)_k})}{\sum_{j} \mathbf{1}\qty(s_{p(j)_1} = s_{p(T+1)_1}, \cdots, s_{p(j)_k} = s_{p(T+1)_k})} \\
    & = \hat \pi_{s_{1:T}}\qty(s' \mid s_{p(T+1)_1}, \dots, s_{p(T+1)_k}),
\end{align*}
as desired. 
\end{proof}

\section{Additional Experiments and Details}\label{app:extra_experiments}

\paragraph{Single Parent Experiments:} All single-parent experiments were run with a vocabulary size of $S = 10$, a sequence length of $T=20$, a batch size of $1024$, $\alpha = 0.1$, and learning rate $\eta = 0.3$. We initialize $\widetilde{A}^{(1)} = 0$, $\widetilde{A}^{(2)} = 0$, and $W_O = 0$.

In \Cref{fig:appendix_single_parent}, we repeat \Cref{fig:induction_head_combined} for the in-context learning graph in \cref{subfig:appendix_icl}, in addition to versions when the graph $\mathcal{G}$ comes from a Markov chain (\cref{subfig:appendix_markov}) and when it is random graph (\cref{subfig:appendix_random}).

\paragraph{Multiple Parent Experiments:} For experiment with multiple parents (\Cref{fig:appendix_multi_parent}), we used $\alpha = 1$ and Adam \cite{kingma2017adam} with $\eta = 0.01$ but we initialized $\widetilde{A}^{(1)}_{ij}, \widetilde{A}^{(2)}_{ij} \sim N(0,\sigma^2)$ for $\sigma = 0.01$. This was necessary to break the symmetry between the heads.

\begin{figure}
    \centering
    \subfigure[Markov Chain]{\label{subfig:appendix_markov}\includegraphics[width=0.65\textwidth]{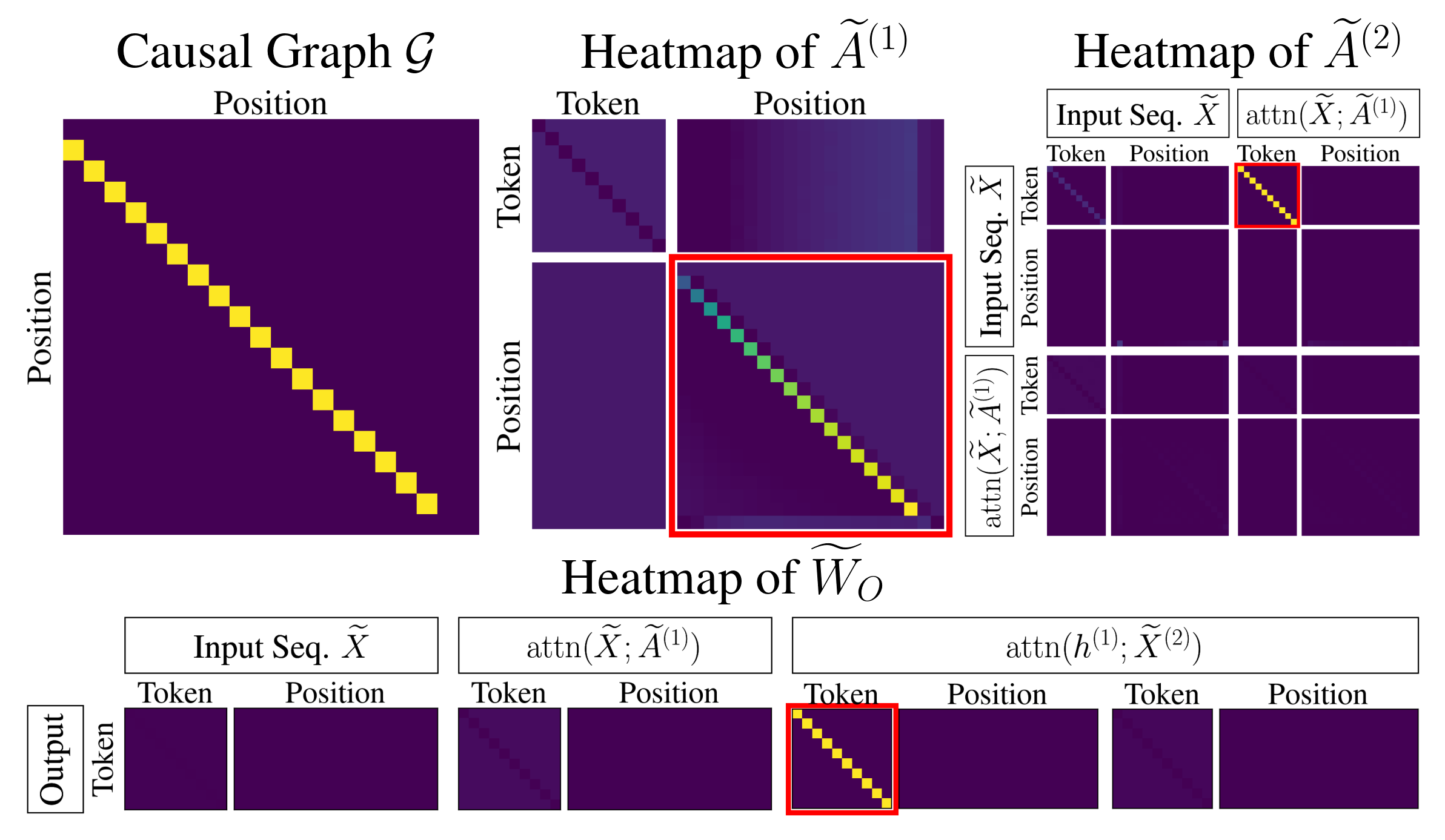}} \\
    \subfigure[In-Context Learning]{\label{subfig:appendix_icl}\includegraphics[width=0.65\textwidth]{figures/icl.png}} \\
    \subfigure[Random Causal Graph]{\label{subfig:appendix_random}\includegraphics[width=0.65\textwidth]{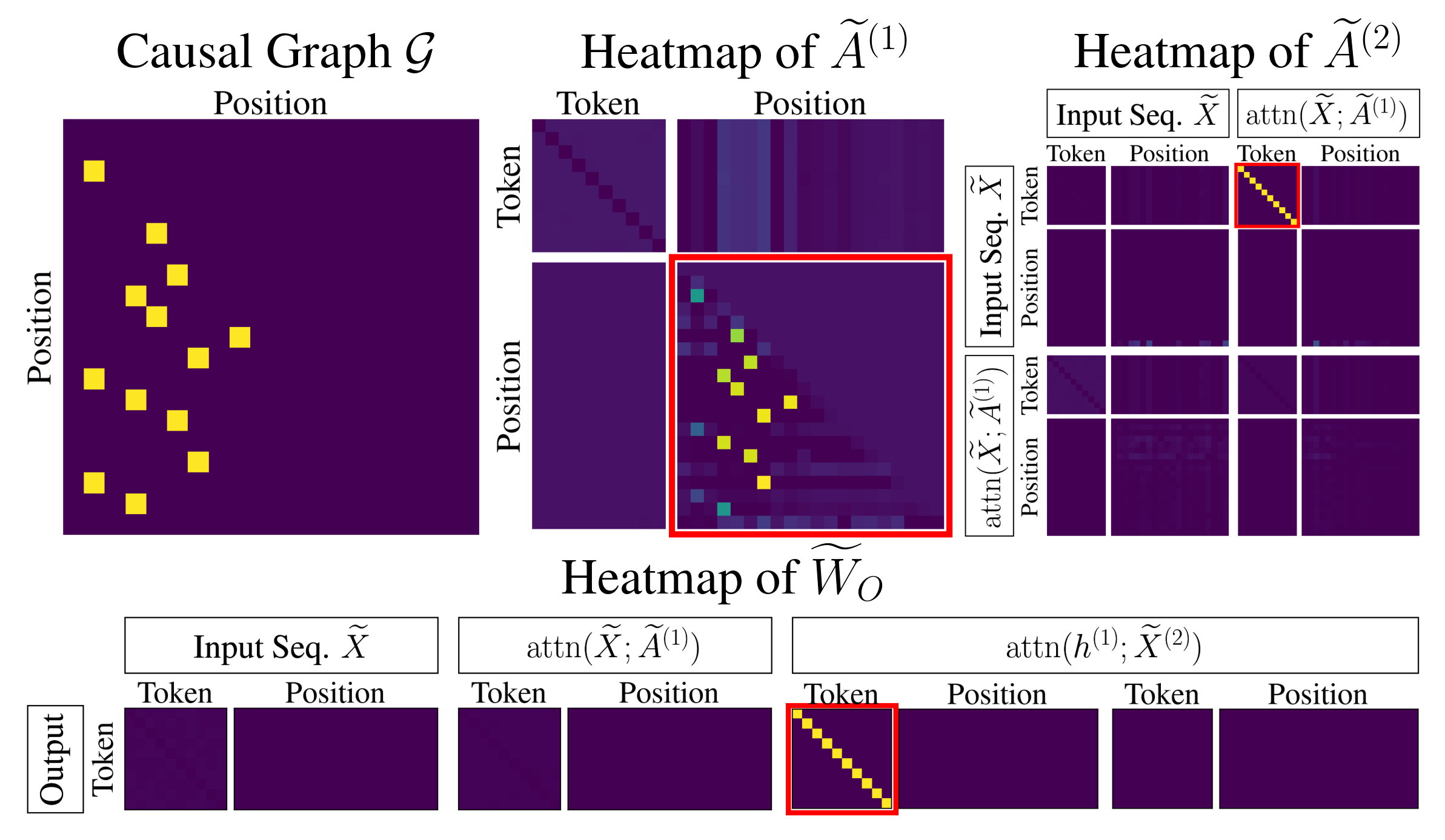}}
    \caption{$\widetilde{A}^{(1)}$ encodes the graph structure, for different latent graphs.}
    \label{fig:appendix_single_parent}
\end{figure}


\begin{figure}
    \centering
    \subfigure[3-gram where each position $i$ attends to $i-1,i-2$.]{\includegraphics[width=0.88\textwidth]{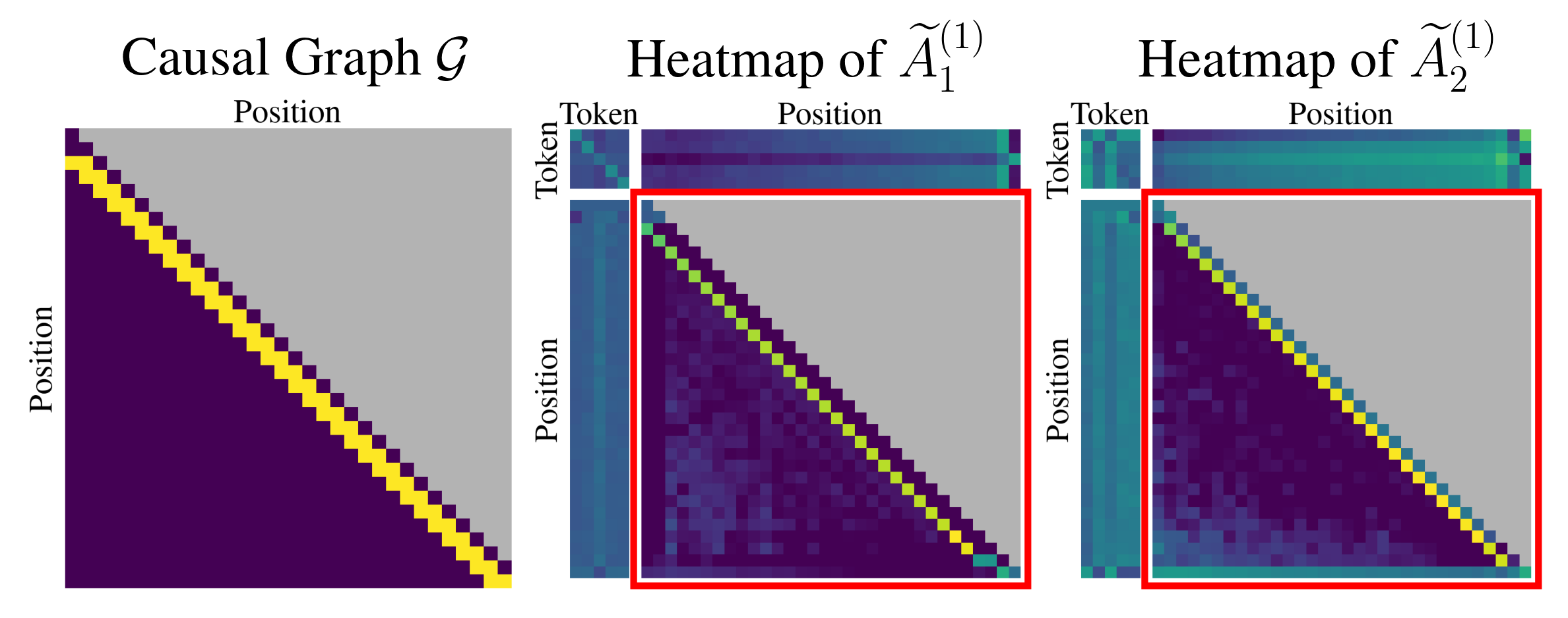}}
    \subfigure[Each position $i$ attends to $i-1$ and $\lfloor \frac{i-1}{2} \rfloor$]{\includegraphics[width=0.88\textwidth]{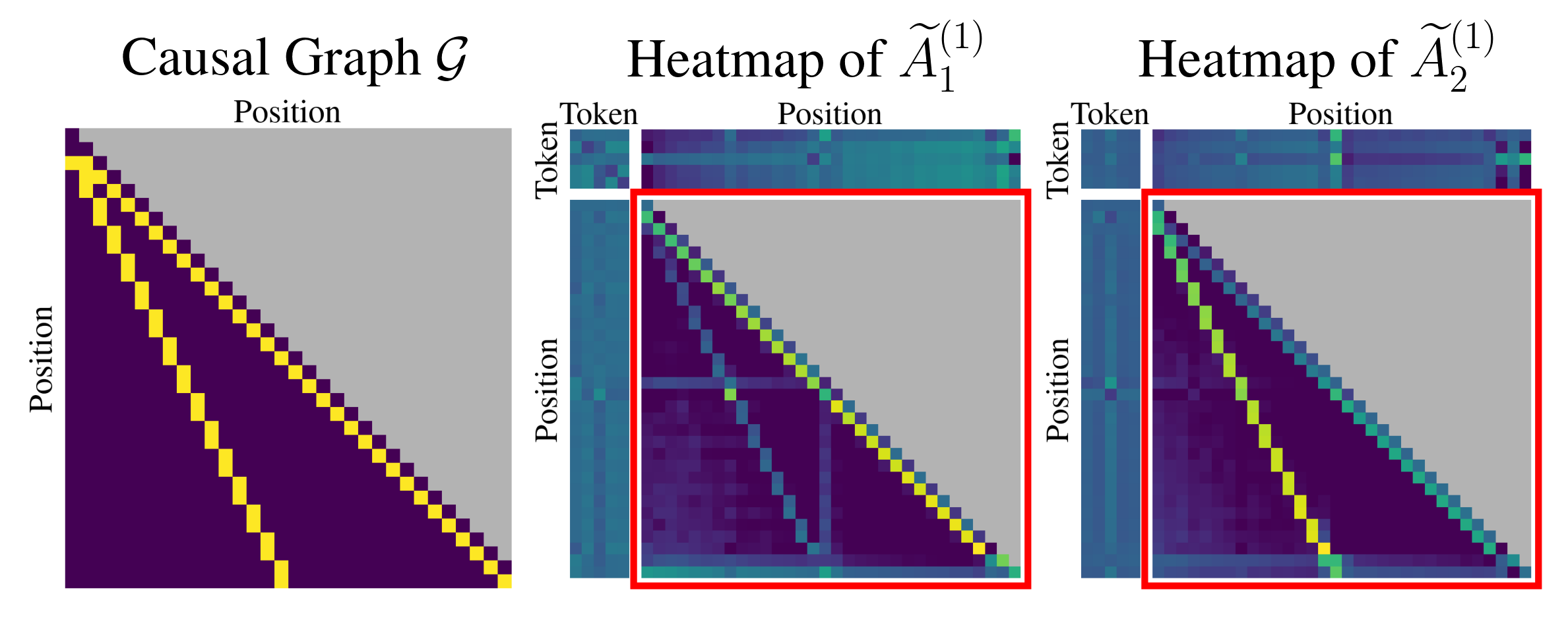}}
    \subfigure[4-gram where each position $i$ attends to $i-1,i-2,i-3$]{\includegraphics[width=0.88\textwidth]
    {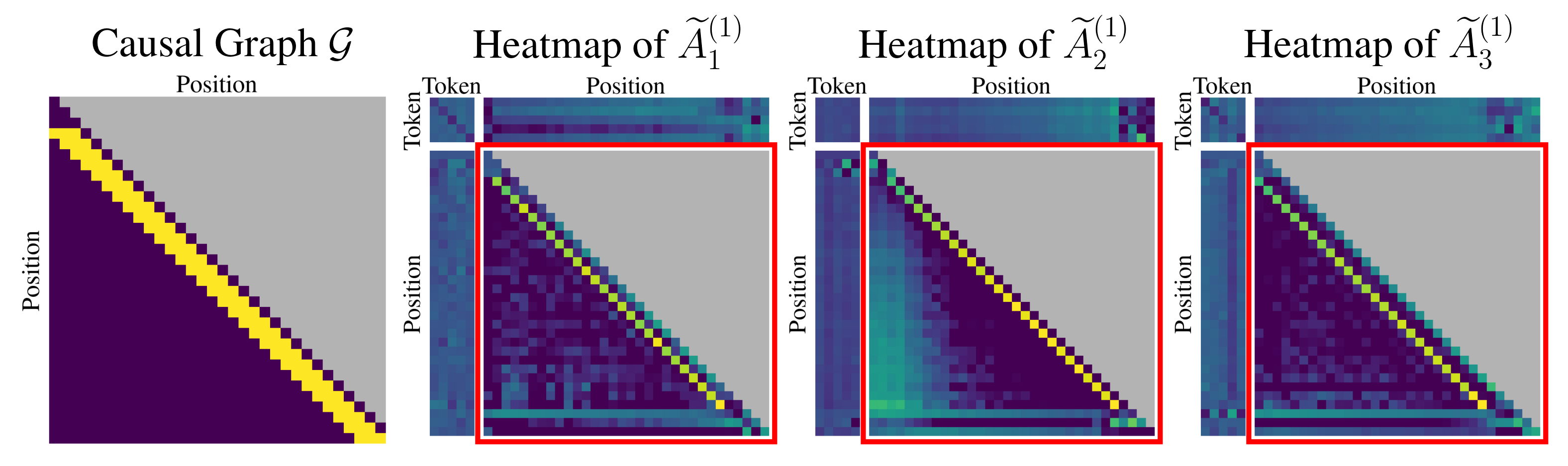}}
    \caption{\textbf{Multiple Parents:} On the left, we plot the causal graph in the setting of \Cref{sec:multiple_parents} with $k=2,2,3$ respectively. The first row corresponds to the 3-gram task in which each token depends on the previous 2. In the second row, each token at position $i$ depends on the previous token and the token at position $\lfloor \frac{i-1}{2} \rfloor$. The third row corresponds to 4-gram in which each token depends on the previous 3 tokens. We train two-layer disentangled transformers on these tasks with $k$ heads in each layer. On the right, we plot the first layer attention matrices, i.e. $\{\widetilde{A}^{(1)}_i\}_i$. We see that each attention head learns a single set of parents in the causal graph $\mathcal{G}$, which agrees with our \Cref{thm:multi_parent_construction}.}
    \label{fig:appendix_multi_parent}
\end{figure}

\paragraph{Experiments with standard transformer architecture:} We trained a two attention layer, decoder-based transformer of the form \eqref{eq:decoder-based-TF} on \Cref{task:single_parent}. We consider fixed position and token embeddings $P \in \mathbb{R}^{d \times T}$ and $E \in \mathbb{R}^{d \times S}$, with each column of $P, E$ drawn i.i.d from $\mathcal{N}(0, \frac{1}{d}I_d)$. Additionally, between each attention layer, we add a one-hidden layer ReLU MLP with hidden width $d$. The model has one head per layer.

In \Cref{fig:TF_with_MLP}, we plot the average attention pattern of the first layer, averaged over a batch of $1024$ sequences. We pick $S = 10, T = 20, d = 30$. The first layer attention pattern on a single sequence with embedding $h^{(0)}$ is given by
\begin{align*}
    \S(\mathrm{MASK}(h^{(0)}A^{(1)}{h^{(0)}}^\top)) \in \mathbb{R}^{T \times T}.
\end{align*}
We observe that this average attention pattern is also approximately equal to the adjacency matrix of the graph $\mathcal{G}$.

\begin{figure}
    \centering
    \subfigure[Markov Chain.]{\includegraphics[width=0.6\textwidth]{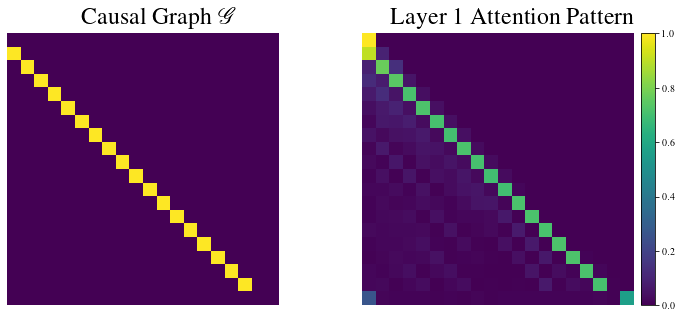}}
    \subfigure[In-Context Learning]{\includegraphics[width=0.6\textwidth]{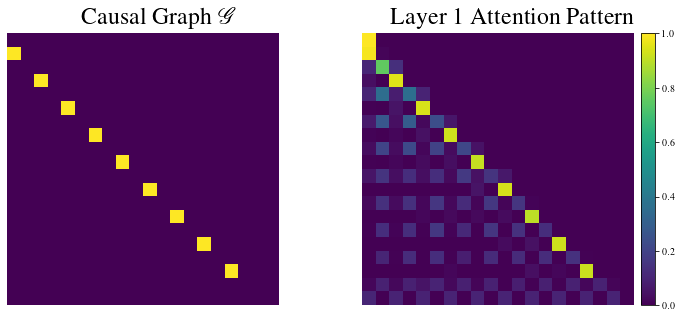}}
    \subfigure[Random Causal Graph]{\includegraphics[width=0.6\textwidth]{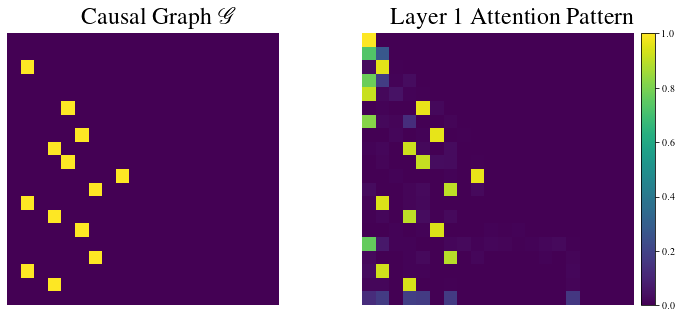}}
    \caption{\textbf{Decoder-Based Transformer with MLPs:} In a two attention-layer, decoder-based transformer with MLPs, we observe that the average attention pattern on a sequence is approximately equal to the adjacency matrix of the causal graph $\mathcal{G}$. We remark that the random causal graph has the peculiar behavior that nodes with no parent seem to attend to the first token in the sequence.}
    \label{fig:TF_with_MLP}
\end{figure}

\paragraph{Quantitative Comparison:} We repeat \Cref{fig:appendix_single_parent} for a set of 20 randomly generated causal graphs on $T = 20$ vertices and vocab size of $S = 3$. In each graph, each node $i$ is a root node with probability $1/2$, and otherwise has its parent $p(i)$ drawn uniformly at random from the set $\{1, \dots, i-1\}$. For each graph, one can compute the average first-layer attention weight from $i$ to its parent $p(i)$, given by
\begin{align*}
    \mathrm{avg\-attn} := \frac{1}{\abs{\overline{\mathcal{R}}}}\sum_{i \in \overline{\mathcal{R}} }\S\qty(\mathrm{MASK}\qty(A^{(1)}))_{i, p(i)}.
\end{align*}
Over all 20 random graphs, $\mathrm{avg\-attn}$ has a mean of \textbf{0.837} and a standard deviation of \textbf{0.054}. In \Cref{fig:many_graphs}, we plot the average value of $\S\qty(\mathrm{MASK}\qty(A^{(1)}))_{i, p(i)}$ for each position $i$ in the sequence. We observe that this attention weight is large across all positions in the sequence.

\begin{figure}[t!]
\centering
    \includegraphics[width=0.5\textwidth]{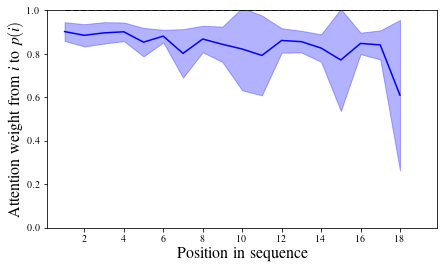}
    \caption{\textbf{Quantitative Comparison:} We plot the mean value of $\S(A^{(1)})_{i, p(i)}$ over all $20$ graphs, as a function of the position $i$ in the sequence. The shaded bars indicate one standard deviation. We observe that this average value is large (close to 1).}
    \label{fig:many_graphs}
\end{figure}

\paragraph{Experimental Details:} Code for all the experiments can be found at \url{https://github.com/eshnich/transformers-learn-causal-structure}. All code was written in JAX~\citep{jax2018github}, and run on a cluster of 10 NVIDIA RTX A6000 GPUs.


\section{Analyzing the Dynamics}
In this section we prove \Cref{thm:main_thm}.
\subsection{Proof of Lemma \ref{lem:rewrite_TF}}
\begin{proof}
    The output of the first attention layer is
    \begin{align*}
        \attn(\tilde{X}; \widetilde{A}^{(1)}) = \S(\mathrm{MASK}(\tilde X \widetilde{A}^{(1)} \tilde X ^\top)\tilde X = \S(\mathrm{MASK}(A^{(1)}))\tilde X.
    \end{align*}
    Next, we have that
    \begin{align*}
        {h_T^{(1)}}^\top \widetilde{A}^{(2)} {h^{(1)}}^\top &= \tilde x_T^\top\begin{bmatrix} A^{(2)} & 0_{S \times T} \\ 0_{T \times S} & 0_{T \times T}\end{bmatrix}\attn(\tilde X; \widetilde{A}^{(1)})^\top\\
        &= \tilde x_T^\top\begin{bmatrix} A^{(2)} & 0_{S \times T} \\ 0_{T \times S} & 0_{T \times T}\end{bmatrix}\tilde X^\top \S(\mathrm{MASK}(A^{(1)}))^\top\\
        &= \overline{x}_T^\top A^{(2)}\overline{X}^\top \S(\mathrm{MASK}(A^{(1)}))^\top.
    \end{align*}
    Thus the output of the second attention layer is
    \begin{align*}
        \attn(h^{(1)}; \widetilde{A}^{(2)})_T = {h^{(1)}}^\top\S\qty({h^{(1)}} \qty(\widetilde{A}^{(2)})^\top h^{(T)})\\
        = {h^{(1)}}^\top\S\qty(\S(\mathrm{MASK}(A^{(1)}))\overline{X}{A^{(2)}}^\top\overline{x}_T)
    \end{align*}
    Finally, the output is
    \begin{align*}
        \widetilde{\mathrm{TF}}_{\widetilde{\theta}}(s_{1:T}) &= \widetilde{W}_O^\top h_T^{(2)}\\
        &= \begin{bmatrix} I_S & 0_{S \times T} \mid 0_{S \times d}\end{bmatrix}\attn(h^{(1)}; \widetilde{A}^{(2)})_T\\
        &= \begin{bmatrix} I_S & 0_{S \times T} \mid 0_{S \times d}\end{bmatrix}{h^{(1)}}^\top\S\qty(\S(\mathrm{MASK}(A^{(1)}))\overline{X}{A^{(2)}}^\top\overline{x}_T)\\
        &= \overline{X}^\top\S\qty(\S(\mathrm{MASK}(A^{(1)}))\overline{X}{A^{(2)}}^\top\overline{x}_T),
    \end{align*}
    as desired.
\end{proof}

\subsection{Notation}
We briefly introduce notation which will be used throughout the rest of the appendix. We let $X \in \R^{T \times S}$ be the token embedding of the sequence $s_{1:T}$. Additionally, for a lower triangular matrix $A \in \R^{T \times T}$, let $A_i \in \R^i$ denote the first $i$ coordinates of the $i$th row of $A$. We overload notation so that $\S(A) \in \R^{T \times T}$ is the lower triangular matrix satisfying $\S(A)_i = \S(A_i)$; i.e, the softmax operation is applied row-wise to the first $i$ coordinates of row $i$. Finally, we reparameterize ${A^{(2)}}^\top$ with $A^{(2)}$.

We can thus rewrite the reduced model $f_\theta$ as
\begin{align*}
    f_\theta(s_{1:T}) = X^\top\S\qty(\S(A^{(1)})XA^{(2)}x_T).
\end{align*}

We let $f_\theta(X; s)$ denote prediction of a transformer with embedding $X \in \R^{T \times S}$ conditioned on $s_T = s$, i.e
\begin{align*}
    f_\theta(X; s) := X^\top\S\qty(\S(A^{(1)})XA^{(2)}e_s).
\end{align*}
It is easy to see that the perturbed loss \eqref{eq:CE_loss} can be written as
\begin{align*}
    L(\theta) = -\frac{1}{S}\E_{\pi, X}\qty[\sum_{s, s' \in [S]}\pi(s' \mid s)\log\qty(f_\theta(X; s)_{s'} + \epsilon)],
\end{align*}
where we use $\E_X$ and $\E_{s_{1:T}}$ interchangeably to represent expectation over the sequence $s_{1:T}$. We set the perturbation as $\epsilon = \TT^{-1/2}$.

For notational convenience, we define $$v_\theta(X; s) := \S(\S(A^{(1)})XA^{(2)}e_s),$$ so that $f_\theta(X; s) = X^\top v_\theta(X; s)$.

Let $\delta_{s}(X) \in \R^T $ be the vector where $\delta_{s}(X)_i = x_{i, s}$, and let $\hat \mu_X(s) := \frac{1}{T}\sum_{i=1}^T  x_{i, s}$ be the empirical estimate of the frequency of $s$ over the sequence $X$. We let $X_{\le i} \in \R^{i \times S}$ be the embedding of the first $i$ tokens in the sequence, and let $\delta_s(X_{\le i}) \in \R^i$ be the indicator of $s$ on these first $i$ tokens.

Given a vector $v \in \R^k$, the operator $J_k: \R^k \rightarrow \R^{k \times k}$ is given by $J_k(v) = diag(v) - vv^\top $. $J_k$ is the Jacobian of $\S$: $\nabla_u \S(u) = J_k(\S(u))$. We drop the subscript $k$ when it is clear from context.

\subsection{Heuristic Derivation of Lemma \ref{lem:gradient_informal}}\label{sec:lemma1_sketch}
During stage 1, the model can be rewritten as
\begin{align*}
    f_\theta(X; s)_{s'} = e_{s'}^\top X^\top \S\qty(\beta_0 \S(A^{(1)}) Xe_s) = \delta_{s'}(X)^\top \S\qty(\beta_0 \S(A^{(1)}) \delta_s(X)).
\end{align*}
When $\beta_0 \approx 0$, we can linearize the outer softmax as
\begin{align*}
    \S(\beta_0z) \approx \frac{1}{T}1_T + \beta_0 \cdot \qty(\frac{1}{T}I_T - \frac{1}{T^2}1_T1_T^\top )z,
\end{align*}
and get that
\begin{align}
    f_\theta(X; s)_{s'} &\approx \frac{1}{T} \delta_{s'}(X)^\top 1_T + \frac{\beta_0}{T}\delta_{s'}(X)^\top \S(A^{(1)})\delta_{s}(X) - \frac{\beta_0}{T^2}\delta_{s'}(X)^\top 1_T \cdot 1_T^\top \S(A^{(1)}) \delta_s(X)\nonumber\\
    &= \hat \mu_X(s') + \frac{\beta_0}{T}\qty(\delta_{s'}(X)^\top \S(A^{(1)})\delta_{s}(X) - \hat \mu_X(s') \cdot 1_T^\top \S(A^{(1)}) \delta_s(X)).\label{eq:approx_f_init}
\end{align}
First, observe that since $\beta_0 \approx 0$,
\begin{align*}
    f_\theta(X; s)_{s'} &\approx \hat \mu_X(s')
\end{align*}

Next, taking the gradient of the approximation \eqref{eq:approx_f_init} with respect to $A^{(1)}_i$ yields
\begin{align*}
    \nabla_{A^{(1)}_i}f_\theta(X; s)_{s'} \approx \frac{\beta_0}{T} J\qty(\S(A^{(1)}_i)) \delta_s(X_{\le i}) \cdot (x_{i, s'} - \hat \mu_X(s')).
\end{align*}
Therefore by the chain rule, the population gradient is given by
\begin{align*}
    \nabla_{A^{(1)}_i} L(\theta) &\approx - \frac{1}{S}\E_{\pi, X}\qty[\sum_{s, s'} \frac{\pi(s' \mid s)}{f_\theta(X; s)_{s'} + \epsilon}\nabla_{A^{(1)}_i}f_\theta(X; s)_{s'}].\\
    &\approx - \frac{\beta_0}{ST}J\qty(\S(A^{(1)}_i))\cdot \mathbb{E}_{\pi, X}\qty[\sum_{s, s'}\frac{\pi(s' \mid s)}{\hat \mu_X(s')}(x_{i, s'} - \hat \mu_X(s'))\delta_s(X_{\le i})].
\end{align*}
Letting $\hat g_i$ denote the term after the preconditioner, i.e $\hat g_i := \mathbb{E}_{\pi, X}\qty[\sum_{s, s'}\frac{\pi(s' \mid s)}{\hat \mu_X(s')}(x_{i, s'} - \hat \mu_X(s'))\delta_s(X_{\le i})]$, we get that the $j$th entry of $\hat g_i$, $\hat g_{i,j}$, is
\begin{align*}
    \hat g_{i, j} &= \mathbb{E}_{\pi, X}\qty[\sum_{s, s'}\frac{\pi(s' \mid s)}{\hat \mu_X(s')}(x_{i, s'} - \hat \mu_X(s'))x_{j, s}]\\
    &= \mathbb{E}_{\pi, X}\qty[\sum_{s, s'}\frac{\pi(s' \mid s)}{\hat \mu_X(s')}x_{i, s'}x_{j, s} - \sum_{s, s'}\pi(s' \mid s)x_{j, s}]\\
    &= \mathbb{E}_{\pi, X}\qty[\sum_{s, s'}\frac{\pi(s' \mid s)}{\hat \mu_X(s')}x_{i, s'}x_{j, s}] - 1.
\end{align*}
Conditioned on $\pi$, as the effective length of the sequence $X$ grows large, due to our assumptions on $P_\pi$ the sequence $x_1, \dots, x_T$ mixes, and thus $\hat \mu_X(s') \rightarrow \mu_\pi(s)$. As such,
\begin{align*}
    \hat g_{i, j} &\approx \mathbb{E}_\pi\qty[\sum_{s, s'}\frac{\pi(s' \mid s)}{\mu_\pi(s')}\mathbb{E}_X[x_{i, s'}x_{j, s}]] - 1\\
    &= \mathbb{E}_\pi\qty[\sum_{s, s'}\frac{\pi(s' \mid s)}{\mu_\pi(s')}\mathbb{P}_X[s_j = s, s_i = s']] - 1\\
    & =g_{i,j}.
\end{align*}

\subsection{Gradient Computations}

Recall that $A_i^{(1)} \in \R^i$ is the $i$th row of $A^{(1)}$. Define the population gradients as
\begin{align*}
    G^{(1)}(A^{(1)}, A^{(2)})_i & := \nabla_{A^{(1)}_i} L(\theta) \big|_{\theta = (A^{(1)}, A^{(2)})} \\
    G^{(2)}(A^{(1)}, A^{(2)})   & := \nabla_{A^{(2)}} L(\theta) \big|_{\theta = (A^{(1)}, A^{(2)})}.
\end{align*}

The following lemma computes the population gradients:
\begin{lemma}[Population gradients]\label{lem:pop_grads}
    \begin{align*}
        G^{(1)}(A^{(1)}, A^{(2)})_i & = -\frac{1}{S}J(\S(A_i^{(1)}))\sum_{s, s'}\E_{\pi, X}\qty[\frac{\pi(s' \mid s)}{f_\theta(X; s)_{s'} + \epsilon}\delta_{s'}(X)^\top J(v_\theta(X; s))e_i \cdot X_{\le i}A^{(2)}e_s] \\
        G^{(2)}(A^{(1)}, A^{(2)})   & = -\frac{1}{S}\sum_{s, s'}\mathbb{E}\qty[\frac{\pi(s' \mid s)}{f_\theta(X; s)_{s'} + \epsilon}\cdot X^\top \S(A^{(1)})^\top J(v_\theta(X; s))\delta_{s'}(X)e_s^\top ]
    \end{align*}
\end{lemma}
\begin{proof}
    The model gradient with respect to $A_i^{(1)}$ is
    \begin{align*}
        \nabla_{A^{(1)}_i} f_\theta(X; s) & = X^\top J(v_\theta(X; s))e_i \otimes J(\S(A^{(1)}_i))X_{\le i}A^{(2)}e_s
    \end{align*}
    Therefore the loss gradient is given by
    \begin{align*}
        G^{(1)}(A^{(1)}, A^{(2)})_i & = -\frac{1}{S}\sum_{s, s'}\E\qty[\frac{\pi(s' \mid s)}{f_\theta(X; s)_{s'} + \epsilon}\nabla f_\theta(X; s)_{s'}]                                                     \\
                                    & = -\frac{1}{S}J(\S(A_i^{(1)}))\sum_{s, s'}\E\qty[\frac{\pi(s' \mid s)}{f_\theta(X; s)_{s'} + \epsilon}\delta_{s'}(X)^\top J(v_\theta(X; s))e_i \cdot X_{\le i}A^{(2)}e_s].
    \end{align*}
    Next, the model gradient of $A^{(2)}$ is
    \begin{align*}
        \nabla_{A^{(2)}}f_\theta(X; s)_{s'} = X^\top \S(A^{(1)})^\top J(v_\theta(X; s))\delta_{s'}(X) e_s^\top.
    \end{align*}
    Thus
    \begin{align*}
        G^{(2)}(A^{(1)}, A^{(2)}) = -\frac{1}{S}\sum_{s, s'}\mathbb{E}\qty[\frac{\pi(s' \mid s)}{f_\theta(X; s)_{s'} + \epsilon}\cdot X^\top \S(A^{(1)})^\top J(v_\theta(X; s))\delta_{s'}(X)e_s^\top ]
    \end{align*}
\end{proof}
\subsection{Gradient of \texorpdfstring{$A^{(1)}$}{A1} (Stage 1)}

We show that during the first stage of training, $A^{(1)}$ converges to the adjacency matrix of the graph $\mathcal{G}$.

The first step is to show that a quantity called the ``idealized gradient" approximately aligns with the adjacency matrix of $\mathcal{G}$. For a transition matrix $\pi$, define
\begin{align*}
    g_{i,j}(\pi) := \sum_{s, s'} \frac{\pi(s' \mid s)}{\mu_\pi(s')}\cdot \mathbb{P}_X[s_i = s', s_j = s] - 1,
\end{align*}
and let $g_{i,j} := \E_\pi[g_{i,j}(\pi)]$.

The following lemma shows that this idealized gradient is maximized at $j = p(i)$. The proof relies on the data processing inequality argument, and is deferred to \Cref{sec:proofs}.

\begin{lemma}[Idealized gradient is aligned with $\mathcal{G}$]\label{thm:DPI}  If $p(i) \neq \emptyset$, then
    \begin{align*}
        g_{i,p(i)} \ge g_{i,j} + \frac{\gamma^3}{2S}
    \end{align*}
    for all $j \in [i] \setminus p(i)$. Otherwise $g_{i,j} = 0$.
\end{lemma}

Next, we show that the true gradient with respect to $A^{(1)}$ can indeed be approximated by this idealized gradient, and hence the adjacency matrix of $\mathcal{G}$.

\begin{lemma}[True gradient of $A^{(1)}$ is aligned with $\mathcal{G}$ (Stage 1)]\label{thm:stage1}
    Let $A^{(2)} = \beta_0 I$. There exist constants $c_{\gamma, S}, C_{\gamma, S}$ such that, if $\beta_0 \le  c_{\gamma, S}\TT^{-3/2}$,
    \begin{itemize}
        \item If $p(i) = \emptyset$,
              \begin{align*}
                  G^{(1)}(A^{(1)}, A^{(2)})_{i} = J(\S(A_i^{(1)}))v
              \end{align*}
              for $v$ with $\|v\|_\infty \le C_{\gamma,S}\frac{\beta_0}{T\sqrt{\TT}}$.
        \item If $p(i) \neq \emptyset$, then for any $j \neq p(i)$,
              \begin{align*}
                  G^{(1)}(A^{(1)}, A^{(2)})_{i, p(i)} \le G^{(1)}(A^{(1)}, A^{(2)})_{i, j} - \S(A_i^{(1)})_{p(i)}\qty(1 - \S(A_i^{(1)})_{p(i)})\cdot \frac{C_{\gamma, S}\beta_0}{T}.
              \end{align*}
    \end{itemize}
\end{lemma}
\begin{proof}
    First, see that
    \begin{align*}
        X_{\le i}A^{(2)}e_s = \beta_0 X_{\le i}e_s = \beta_0 \delta_s(X_{\le i}).
    \end{align*}
    Thus
    \begin{align*}
        G^{(1)}(A^{(1)}, A^{(2)})_i = -\beta_0 J(\S(A_i^{(1)}))\cdot \frac{1}{S}\sum_{s, s'}\E_{\pi, X}\qty[\frac{\pi(s' \mid s)}{f_\theta(X; s)_{s'} + \epsilon}\delta_{s'}(X)^\top J(v_\theta(X; s))e_i \cdot  \delta_s(X_{\le i})].
    \end{align*}
    Let $\hat \theta := (A^{(1)}, 0)$, and define the quantities $g_i^*, \hat g_i$ by
    \begin{align*}
        g^*_i := T\sum_{s, s'}\E_{\pi, X}\qty[\frac{\pi(s' \mid s)}{f_\theta(X; s)_{s'} + \epsilon}\delta_{s'}(X)^\top J(v_\theta(X; s))e_i \cdot  \delta_s(X_{\le i})], \\
        \hat g_i := T\sum_{s, s'}\E_{\pi, X}\qty[\frac{\pi(s' \mid s)}{f_{\htheta}(X; s)_{s'} + \epsilon}\delta_{s'}(X)^\top J(v_{\htheta}(X; s))e_i \cdot  \delta_s(X_{\le i})].
    \end{align*}
    We remark that 
    \begin{align}\label{eq:key_equation_stage1}
    G^{(1)}(A^{(1)}, A^{(2)})_i = -\frac{\beta_0}{ST} J(\S(A_i^{(1)}))g^*_i.
    \end{align}
    Since $\beta_0 \le c_{\gamma, S}\frac{1}{\TT^{3/2}} \le 1$, by \Cref{lem:nonzero_beta} we have
    \begin{align*}
        \norm{\hat g_i - g^*_i}_\infty \le 6S^2\epsilon^{-2}\beta_0 \le C_{\gamma, S}\frac{1}{\sqrt{\TT}}.
    \end{align*}

    It thus suffices to analyze $\hat g_i$. Note that $v_{\htheta}(X; s) = \frac{1}{T}1_T$. Therefore
    \begin{align*}
        f_{\htheta}(X; s)_{s'} = \frac{1}{T}1_T^\top \delta_{s'}(X) = \hat\mu_X(s').
    \end{align*}
    and
    \begin{align*}
        \delta_{s'}(X)^\top J(v_{\htheta}(X; s))e_i = \delta_{s'}(X)^\top \qty(\frac{1}{T}I_T - \frac{1}{T^2}1_T1_T^\top )e_i = \frac{1}{T}\qty(x_{i, s'} - \hat\mu_X(s')).
    \end{align*}

    The $j$th entry of $\hat g_i$, $\hat g_{i, j}$, is thus equal to
    \begin{align*}
        \hat g_{i,j} = \sum_{s, s'}\E_{\pi, X}\qty[\frac{\pi(s' \mid s)}{\hat \mu_X(s') + \epsilon}(x_{i, s'} - \hat\mu_X(s'))x_{j, s}].
    \end{align*}

    By \Cref{lem:concentration_A1}, this is approximately equal to the idealized gradient $g_{i,j}$:
    \begin{align*}
        \abs{\hat g_{i, j} - g_{i, j}} \le C_{\gamma, S}\frac{1}{\sqrt{\TT}}.
    \end{align*}

    We are now ready to prove the theorem. First, consider the case where $p(i) = \emptyset$. By \Cref{thm:DPI}, $g_{i,j} = 0$, and thus
    \begin{align*}
        \abs{g^*_{i,j}} \lesssim \frac{1}{\sqrt{\TT}}
    \end{align*}
    Since $G^{(1)}(A^{(1)}, A^{(2)})_i = - \frac{\beta_0}{ST} J(\S(A_i^{(1)})) g^*_i$, the claim follows.

    Otherwise if $p(i) \neq \emptyset$, \Cref{thm:DPI} tells us that, for all $j \neq p(i)$,
    \begin{align*}
        g^*_{i,j} - g^*_{i, p(i)}\le g_{i,j} - g_{i, p(i)} + \abs{g_{i,j} - g^*_{i,j}} + \abs{g_{i, p(i)} - g^*_{i, p(i)}} \le -\frac{\gamma^3}{2S} + C_{\gamma, S}\frac{1}{\sqrt{\TT}} \le -\frac{\gamma^3}{4S}.
    \end{align*}

    Next, see that
    \begin{align*}
        G^{(1)}(A^{(1)}, A^{(2)})_{i, j} = -\frac{\beta_0}{ST}\qty(\S(A_i^{(1)})_jg^*_{i,j} - \S(A_i^{(1)})^\top g^*_{i}\S(A_i^{(1)})_j).
    \end{align*}
    Therefore for any $j \neq p(i)$, we can bound
    \begin{align*}
         & G^{(1)}(A^{(1)}, A^{(2)})_{i, j} - G^{(1)}(A^{(1)}, A^{(2)})_{i, p(i)}                                                                                                 \\
         & =  \frac{\beta_0}{ST}\qty[\qty(\S(A_i^{(1)})_{p(i)} - \S(A_i^{(1)})_j)\qty(g^*_{i,p(i)} - \S(A_i^{(1)})^\top g^*_{i}) + \S(A_i^{(1)})_j(g^*_{i,p(i)} - g^*_{i,j})]   \\
         & \ge \frac{\beta_0}{ST}\qty[\qty(\S(A_i^{(1)})_{p(i)} - \S(A_i^{(1)})_j)\qty(1 - \S(A_i^{(1)})_{p(i)})\frac{\gamma^3}{4S} + \S(A_i^{(1)})_j\frac{\gamma^3}{4S}] \\
         & \ge \S(A_i^{(1)})_{p(i)}\qty(1 - \S(A_i^{(1)})_{p(i)})\cdot\frac{\beta_0\gamma^3}{4S^2T},
    \end{align*}
    as desired.
\end{proof}

We can now analyze the gradient descent dynamics over multiple timesteps. First, we show that for most root nodes $i \in \calR$, $A_i^{(1)}$ moves very little.

\begin{lemma}\label{lem:no_parents_dont_move}
    Let $i \in \calR$. Then
    \begin{align*}
        \abs{\S(A_i^{(1)}(\tau))_j - \frac{1}{i}} \lesssim \frac{\tau\eta_1\beta_0}{T\sqrt{\TT} \cdot i^2}
    \end{align*}
    for all $j \le i$.
\end{lemma}
\begin{proof}
    Let $r(A_i^{(1)}) = \max_jA_{i, j}^{(1)} - \min_jA_{i, j}^{(1)}$. We have that (where $v$ is the vector from \Cref{thm:stage1}),
    \begin{align*}
        \norm{G^{(1)}(A^{(1)}, A^{(2)})_i}_\infty \le \max_j \S(A_i^{(1)})_j \cdot \|v\|_\infty,
    \end{align*}
    and thus
    \begin{align*}
        r(A_i^{(1)}(t+1)) \le r(A_i^{(1)}(t)) + 2\eta_1 \max_j \S(A_i^{(1)}(t))_j \cdot \|v\|_\infty.
    \end{align*}
    Fix $\omega \le 1$. Assume there exists some $t \le \tau$ such that $r(A_i^{(1)}(t)) > \log (1 + \omega)$, and let $t^*$ be the first such time $t$. We can always bound
    \begin{align*}
        \max_j \S(A_i^{(1)}(t))_j \le \frac{\exp(r(A_i^{(1)}(t)))}{(i-1) + \exp(r(A_i^{(1)}(t)))},
    \end{align*}
    and thus for $t < t^*$, $\max_j \S(A_i^{(1)}(t))_j \le \frac{1 + \omega}{i+\omega} \le \frac{1 + \omega}{i}$. Therefore
    \begin{align*}
        \log (1 + \omega) < r(A_i^{(1)}(t^*)) \le 2\tau\eta_1 \norm{v}_\infty i^{-1} \cdot (1 + \omega),
    \end{align*}
    Bounding $\log(1 + \omega) \ge \omega/2$ and $1 + \omega \le 2$, we get that
    \begin{align*}
        \omega \le 8\tau\eta_1 \norm{v}_{\infty}i^{-1} \lesssim \frac{\tau\eta_1\beta_0}{T\sqrt{\TT} \cdot i}.
    \end{align*}
    Additionally, when  $r(A_i^{(1)}(t)) \le \log (1 + \omega)$, we have the bound
    \begin{align*}
        \frac{1}{i}(1 - \omega) \le \frac{1}{1 + (1 + \omega)(i-1)} \le \S(A_i^{(1)}(t))_j \le \frac{1}{i}(1 + \omega).
    \end{align*}
    Therefore
    \begin{align*}
        \abs{\S(A_i^{(1)}(\tau))_j - \frac{1}{i}} \le \frac{\omega}{i} \lesssim \frac{\tau\eta_1\beta_0}{T\sqrt{\TT} \cdot i^2},
    \end{align*}
    as desired.
\end{proof}

Next, we bound the time it takes until $\S\qty(A^{(1)}(t))_{i, p(i)} \approx 1$.

\begin{lemma}\label{lem:A1_dynamics}
    Let $A^{(2)}(0) = \beta_0I_S$, where $\beta_0 \le  c_{\gamma, S}\frac{1}{\TT^{3/2}}.$ There exists $\tau_1 \lesssim \eta_1^{-1}\beta_0^{-1}(T^2 + T\alpha^{-1})\log(T/\alpha)$ such that, for any $t \ge \tau_1,$
    \begin{align*}
        \S\qty(A^{(1)}(t))_{i, p(i)} \ge 1 - \alpha.
    \end{align*}
    for all $i$ with $p(i) \neq \emptyset$.
\end{lemma}
\begin{proof}
    By induction, one has that $A^{(1)}(t)_{i,p(i)} \ge A^{(1)}(t)_{i, j}$ throughout training. Thus $\S\qty(A^{(1)}(t))_{i, p(i)} \ge \frac{1}{T}$. Additionally, by \Cref{thm:stage1}, one has that $\S\qty(A^{(1)}(t))_{i, p(i)}$ is increasing in $t$.

    Fix $i$. Define $\Delta(t) = A^{(1)}(t)_{i, p(i)} - \max_{j \neq p(i)}A^{(1)}(t)_{i, j}$. One sees that
    \begin{align*}
        \S\qty(A^{(1)}(t))_{i, p(i)} \ge \frac{\exp(\Delta(t))}{T + \exp(\Delta(t))}.
    \end{align*}

    Let $\tau^+(1/2)$ be the first time $t$ at which $\S\qty(A^{(1)}(t))_{i, p(i)} > \frac12$. For $t < \tau^+(1/2)$ we have $1 - \S\qty(A^{(1)}(t))_{i, p(i)} \ge \frac12$, and thus by \Cref{thm:stage1},
    \begin{align*}
        \Delta(t+1) & \ge \Delta(t) + \frac{C_{\gamma, S}\beta_0}{T^2} \eta_1.
    \end{align*}
    Therefore $\Delta(\tau^+(1/2)) \gtrsim \frac{\beta_0\eta_1}{T^2}\tau^+(1/2)$. Assume that $\Delta(\tau^+(1/2)) \ge \log(2T)$. Then
    \begin{align*}
        \S\qty(A^{(1)}(\tau^+(1/2)))_{i, p(i)} \ge \frac{\exp(\log (2T))}{T + \exp(\log (2T))} = \frac23,
    \end{align*}
    a contradiction. Thus $\Delta(\tau^+(1/2)) \le \log(2T)$, so $\tau^+(1/2) \lesssim T^2\eta_1^{-1}\beta_0^{-1}\log(2T)$.

    Let $\tau^+(\alpha)$ be the first time at which $\S(A^{(1)}(\tau^+(\alpha))_{i, p(i)} < 1 - \alpha$. For $\tau^+(1/2) \le t < \tau^+(\alpha)$, we then have
    \begin{align*}
        \Delta(t+1) & \ge \Delta(t) + \frac{C_{\gamma, S}\beta_0\alpha}{T} \eta_1,
    \end{align*}
    and thus if $\tau^+(\alpha) - \tau^+(1/2) \gtrsim T\alpha^{-1}\beta_0^{-1}\log(T/\alpha)$,
    \begin{align*}
        \Delta(\tau^+(\alpha)) \ge \frac{C_{\gamma, S}\beta_0\alpha}{T} \eta_2(\tau^+(\alpha) - \tau^+(1/2)) \ge \log(\frac{T}{\alpha})
    \end{align*}
    Then
    \begin{align*}
        \S\qty(A^{(1)}(\tau^+(\alpha))_{i, p(i)}) \ge \frac{\exp(\log (T/\alpha))}{T + \exp(\log (T/\alpha))} = \frac{\frac{1}{\alpha}}{1 + \frac{1}{\alpha}} \ge 1 - \alpha,
    \end{align*}
    a contradiction. Thus $\tau^+(\alpha) - \tau^+(1/2) \lesssim T\alpha^{-1}\beta_0^{-1}\log(T/\alpha)$, and so $\tau^+(\alpha) \lesssim T^2\eta_1^{-1}\beta_0^{-1}\log(2T) + T\alpha^{-1}\beta_0^{-1}\log(T/\alpha) \lesssim \eta_1^{-1}\beta_0^{-1}(T^2 + T\alpha^{-1})\log(T/\alpha)$, as desired.
\end{proof}

Combining the previous two lemmas, the following corollary tells us the value of $A^{(1)}$ after stage 1 of the algorithm.
\begin{corollary}[Ouptut of stage 1]\label{cor:end_of_stage_1}
    Let $\beta_0 \le  c_{\gamma, S}\frac{1}{\TT^{3/2}}$, and set $\tau_1 = C_{\gamma, S}\eta_1^{-1}\beta_0^{-1}T^2\log(T)$ for appropriately chosen constants $c_{\gamma, S}, C_{\gamma, S}$. Then: 
\begin{itemize}
    \item If $i \in \overline{\calR}$,
    \begin{align*}
        1 - \S\qty(A^{(1)}(\tau_1))_{i, p(i)} \lesssim T^{-1},
    \end{align*}
    \item If $i \in \calR$,
    \begin{align*}
        \sup_{j \in [i]}\abs{\S(A_i^{(1)}(\tau_1))_j - \frac{1}{i}} \lesssim \min\qty(1, \frac{T \log T}{\sqrt{\TT} \cdot i^2}).
    \end{align*}
\end{itemize}
\end{corollary}
\begin{proof}
    This follows directly from plugging in $\tau = \tau_1$ into \Cref{lem:no_parents_dont_move} and selecting $\alpha = \Theta(T^{-1})$ in \Cref{lem:A1_dynamics}.
\end{proof}

\subsection{Gradient of \texorpdfstring{$A^{(2)}$}{A2} (Stage 2)}

First, we observe that the population dynamics of $A^{(2)}$ possess a certain symmetry:
\begin{lemma}\label{lem:A2_grad_sym}
    For all time, $A^{(2)} = \beta_0 I_S + \beta (I_S - \frac{1}{S}1_S1_S^\top )$ for some scalar $\beta$.
\end{lemma}
\begin{proof}
    If $A^{(2)} = \beta_1 I_S + \beta_2 1_S1_S^\top $ (all diagonals are equal and all off-diagonals are equal), then by symmetry the gradient is also of this form. Additionally, see that
    \begin{align*}
        1_S^\top G^{(2)}(A^{(1)}, A^{(2)}) & = -\frac{1}{S}\sum_{s, s'}\mathbb{E}_{\pi, X}\qty[\frac{\pi(s' \mid s)}{f_\theta(X; s)_{s'} + \epsilon}\cdot 1_S^\top X^\top \S(A^{(1)})^\top J(v_\theta(X; s))\delta_{s'}(X)e_s^\top ] \\
                                       & = -\frac{1}{S}\sum_{s, s'}\mathbb{E}_{\pi, X}\qty[\frac{\pi(s' \mid s)}{f_\theta(X; s)_{s'} + \epsilon}\cdot 1_T^\top \S(A^{(1)})^\top J(v_\theta(X; s))\delta_{s'}(X)e_s^\top ]    \\
                                       & = -\frac{1}{S}\sum_{s, s'}\mathbb{E}_{\pi, X}\qty[\frac{\pi(s' \mid s)}{f_\theta(X; s)_{s'} + \epsilon}\cdot 1_T^\top J(v_\theta(X; s))\delta_{s'}(X)e_s^\top ]                \\
                                       & = 0,
    \end{align*}
    since $J(v_\theta(X; s))1_T = 0$. Therefore $G^{(2)}(A^{(1)}, A^{(2)}) = \beta \cdot \qty(I_S - \frac{1}{S}1_S1_S^\top )$ for some scalar $\beta$.
    Since we initialize $A^{(2)} = \beta_0 I$, throughout training $A^{(2)}$ is of the form $A^{(2)} = \beta_0 I_S + \beta (I_S - \frac{1}{S}1_S1_S^\top )$.
\end{proof}
Throughout the rest of the proof, we let $\beta(t)$ be the scalar such that $$A^{(2)}(t) = \beta(t) I_S - (\beta(t) - \beta_0)\frac{1}{S}1_S1_S^\top .$$

The goal of this section is to show that when $A^{(1)}$ approximates the adjacency matrix of $\mathcal{G}$, $\beta(t)$ will grow large. Since the gradient descent update for $A^{(2)}$ is
\begin{align*}
    A^{(2)}(t+1) = A^{(2)}(t) - \eta_2G^{(2)}(A^{(1)}(t), A^{(2)}(t)),
\end{align*}
the update for $\beta(t)$ is
\begin{align*}
    \beta(t+1) = \beta(t) - \eta_2 \cdot \frac{1}{S-1}\Tr(G^{(2)}(A^{(1)}(t), A^{(2)}(t))).
\end{align*}
As such, we define the quantity $\Delta_\beta(\theta)$ by
\begin{align*}
    \Delta_\beta(\theta) := \frac{1}{S-1}\Tr(G^{(2)}(A^{(1)}, A^{(2)})).
\end{align*}

Finally, for notational convenience, let $A_*^{(1)}$ be the $T \times T$ matrix such that
\begin{align*}
    \S\qty(A^{(1)}_*)_{ij} = \begin{cases}
                                \mathbf{1}(j = p(i)) & \text{if}~i \in \overline{\calR} \\
                                \S(A^{(1)}(\tau_1))_{i,j}        & \text{if}~i \in \calR
                            \end{cases}.
\end{align*}
$A_*^{(1)}$ encodes the adjacency matrix of $\mathcal{G}$ on nodes $i$ where $p(i) \neq \emptyset$.

\begin{lemma}[Stage 2]\label{thm:stage2}
    Let $\theta = (A^{(1)}, A^{(2)})$, where $A^{(1)} = A^{(1)}(\tau_1)$ is the output of stage 1, and $A^{(2)} = \beta I_S - (\beta - \beta_0)\frac{1}{S}1_S1_S^\top $ for $\beta \ge 0$. If $\beta$ satisfies $$\exp(\beta) \le \exp(\beta^*) := C_{\gamma, S} \TT^{1/12}\log^{-1/6}T,$$ then
    \begin{align*}
        1 \ge -\Delta_\beta(\theta) \ge  \frac14\gamma^8S^{-6}e^{-2\beta} > 0.
    \end{align*}
\end{lemma}

\begin{proof}
    Note that $XA^{(2)}e_s =  \beta Xe_s - (\beta - \beta_0)\frac{1}{S}1_T$. Since the row sums of $\S(A^{(1)})$ are 1,
    \begin{align*}
        \S(A^{(1)})XA^{(2)}e_s = \beta \S(A^{(1)})Xe_s - \frac{\beta - \beta_0}{S}1_T,
    \end{align*}
    and thus
    \begin{align*}
        v_\theta(X; s) = \S(\beta \S(A^{(1)})Xe_s).
    \end{align*}
    Define $z_\theta(X; s) = \S(A^{(1)})Xe_s$. We have that
    \begin{align*}
        -\Delta_\beta(\theta) &= - \frac{1}{S - 1}\Tr\qty[G^{(2)}(A^{(1)}, A^{(2)})]\\
        &= \frac{1}{S(S-1)}\sum_{s, s'}\E_{\pi, X}\qty[\frac{\pi(s' \mid s)}{f_\theta(X; s)_{s'} + \epsilon} \delta_{s'}(X)^\top J(v_\theta(X; s))\S(A^{(1)})Xe_s]\\
        & = \frac{1}{S(S-1)}\sum_{s, s'}\mathbb{E}_{\pi, X}\qty[\frac{\pi(s' \mid s)}{f_\theta(X; s)_{s'} + \epsilon}\delta_{s'}(X)^\top J(\S(\beta z_\theta(X; s)))z_\theta(X; s)]                     \\
                                & = \frac{1}{S(S-1)}\sum_{s, s'}\mathbb{E}_{\pi, X}\qty[\frac{\pi(s' \mid s)}{\delta_{s'}(X)^\top \S(\beta z_\theta(X; s)) + \epsilon}\delta_{s'}(X)^\top J(\S(\beta z_\theta(X; s)))z_\theta(X; s)]
    \end{align*}

    We first show the upper bound. We can write
    \begin{align*}
        \delta_{s'}(X)^\top J(\S(\beta z_\theta(X; s)))z_\theta(X; s) &\le \sum_i\delta_{s'}(X)_i \S(\beta z_\theta(X; s))_i z_\theta(X; s)_i\\
        &\le \sum_i\delta_{s'}(X)_i \S(\beta z_\theta(X; s))_i\\
        &= \delta_{s'}(X)^\top \S(\beta z_\theta(X; s)),
    \end{align*}
    since $0 \le z_\theta(X; s)_i \le 1$. Therefore
    \begin{align*}
        -\Delta_\beta(\theta) \le \frac{1}{S(S-1)}\sum_{s, s'}\E_{\pi, X}\qty[\pi(s' \mid s)] = \frac{1}{S-1} \le 1.
    \end{align*}

    We next move to the lower bound. Define $\tilde z(X; s) := \S(A_*^{(1)})Xe_s$. We have that
    \begin{align}\label{eq:tilde_z_formula}
        \tilde z(X; s)_i = \begin{cases}
                               x_{p(i), s}      & \text{if}~i \not\in \calR \\
                               z_\theta(X; s)_i & \text{if}~i \in \calR
                           \end{cases}.
    \end{align}

    First, we will aim to replace $z_\theta(X; s)$ with $\tilde z(X; s)$. Indeed, when $p(i) \neq \emptyset$,
    \begin{align*}
        \abs{\tilde z(X; s)_i - z_\theta(X; s)_i} = \abs{\qty(\S(A^{(1)})_i - \S(A_*^{(1)})_i)^\top \delta_s(X)} \le \norm{\S(A^{(1)})_i - \S(A_*^{(1)})_i}_1 \lesssim T^{-1},
    \end{align*}
    
    by \Cref{cor:end_of_stage_1}. Thus $\norm{\tilde z(X; s) - z_\theta(X; s)}_\infty \lesssim T^{-1}$.
    
    Define
    \begin{align*}
        q_{s'}(z) := \frac{\delta_{s'}(X)^\top J(\S(\beta z))z}{\delta_{s'}(X)^\top \S(\beta z) + \epsilon},
    \end{align*}
    so that
    \begin{align*}
        -\Delta_\beta(\theta) = \frac{1}{S(S-1)}\sum_{s, s'}\mathbb{E}_{\pi, X}\qty[\pi(s' \mid s)q_{s'}\qty(z_\theta(X; s))].
    \end{align*}

    By \Cref{lem:compare_to_shift}, we have that
    \begin{align*}
        \abs{q_{s'}(z_\theta(X; s)) - q_{s'}(\tilde z(X; s))} \lesssim (1 + \beta) \norm{z_\theta(X; s) - \tilde z(X; s)}_\infty \lesssim (1 + \beta)T^{-1},
    \end{align*}
    and thus $$\abs{-\Delta_\beta(\theta) - \frac{1}{S(S-1)}\sum_{s, s'}\mathbb{E}_{\pi, X}\qty[\pi(s' \mid s)q_{s'}\qty(\tilde z(X; s))]} \lesssim (1 + \beta)T^{-1}.$$

    Next, plugging in the definition of $q_{s'}$, we get
    \begin{align}\label{eq:term_to_concentrate}
         & \frac{1}{S(S-1)}\sum_{s, s'}\mathbb{E}_{\pi, X}\qty[\pi(s' \mid s)q_{s'}\qty(\tilde z(X; s))]\nonumber\\
         & = \frac{1}{S(S-1)}\sum_{s, s'}\mathbb{E}_{\pi, X}\qty[\pi(s' \mid s)\frac{\delta_{s'}(X)^\top \qty(\text{diag}(\S(\beta \tilde z(X; s))) - \S(\beta \tilde z(X; s))\S(\beta \tilde z(X; s))^\top )\tilde z(X; s)}{\delta_{s'}(X)^\top \S(\beta\tilde z(X; s))  + \epsilon}] \nonumber    \\
         & \ge \frac{1}{S(S-1)}\sum_{s, s'}\mathbb{E}_{\pi, X}\qty[\pi(s' \mid s)\cdot \qty(\frac{\sum_{i}x_{i, s'}\S(\beta \tilde z(X; s))_i\tilde z(X; s)_i}{\epsilon + \sum_i x_{i, s'}\S(\beta \tilde z(X; s))_i} -  \sum_{i}\S(\beta \tilde z(X; s))_i\tilde z(X; s)_i)].
    \end{align}

    Our next goal is to replace the term in the parentheses in \eqref{eq:term_to_concentrate} with something independent of $X$, where the concentration holds as $\TT$ grows large. Indeed, define the quantities $E^{(1)}_{s, s'}(X), E^{(2)}_{s, s'}(X), E^{(3)}_{s}(X)$ by
    \begin{align}\label{eq:E1_formula}
        E^{(1)}_{s, s'}(X) & := \sum_{i}x_{i, s'}\S(\beta \tilde z(X; s))_i\tilde z(X; s)_i \\
        E^{(2)}_{s, s'}(X) & := \sum_i x_{i, s'}\S(\beta \tilde z(X; s))_i                  \\
        E^{(3)}_{s}(X) & := \sum_i \S(\beta \tilde z(X; s))_i\tilde z(X; s)_i,
    \end{align}
    so that
    \begin{align*}
        \frac{1}{S(S-1)}\sum_{s, s'}\mathbb{E}_{\pi, X}\qty[\pi(s' \mid s)q_{s'}\qty(\tilde z(X; s))] \ge \frac{1}{S(S-1)}\sum_{s, s'}\E_{\pi, X}\qty[\pi(s' \mid s)\cdot \qty(\frac{E^{(1)}_{s, s'}(X)}{\epsilon + E^{(2)}_{s, s'}(X)} - E^{(3)}_{s}(X))].
    \end{align*}
    Let $r = \frac{\abs{\mathcal{R}}}{T}$. One can make the approximation
    \begin{align}
        \frac{E^{(1)}_{s, s'}(X)}{\epsilon + E^{(2)}_{s, s'}(X)} & \approx \frac{(1 - r)e^\beta \mu_\pi(s)\pi(s' \mid s) + re^{\beta \mu_\pi(s)}\mu_\pi(s)\mu_\pi(s')}{(1-r)(e^\beta - 1)\mu_\pi(s)\pi(s' \mid s) + (1-r)\mu_\pi(s') + re^{\beta \mu_\pi(s)}\mu_\pi(s')} \label{eq:E1_approx}\\
        E^{(3)}_{s}(X) & \approx \frac{(1-r)e^\beta \mu_\pi(s) + r e^{\beta \mu_\pi(s)}\mu_\pi(s)}{(1-r)(e^\beta - 1)\mu_\pi(s) + (1-r) + re^{\beta \mu_\pi(s)}}.\label{eq:E2_approx}
    \end{align}
    This motivates defining the following idealized gradient:
    \begin{align*}
        \hat g(\beta) & := \frac{1}{S(S-1)}\sum_s\mathbb{E}_\pi\Big[\mu_\pi(s)\cdot \Big(\sum_{s'}\frac{(1 - r)e^\beta\pi(s' \mid s)^2 + re^{\beta \mu_\pi(s)}\mu_\pi(s')\pi(s' \mid s)}{(1-r)(e^\beta - 1)\mu_\pi(s)\pi(s' \mid s) + (1-r)\mu_\pi(s') + re^{\beta \mu_\pi(s)}\mu_\pi(s')} \\
                      & \qquad - \frac{(1-r)e^\beta  + r e^{\beta \mu_\pi(s)}}{(1-r)(e^\beta - 1)\mu_\pi(s) + (1 - r) + re^{\beta \mu_\pi(s)}}\Big)\Big]
    \end{align*}
    Indeed, the approximations in \eqref{eq:E1_approx} and \eqref{eq:E2_approx} can be made rigorous: by \Cref{lem:concentrate_E3} and \Cref{lem:concentrate_E1_E2}, we have that
    \begin{align*}
        \abs{\frac{1}{S(S-1)}\sum_{s, s'}\E_{\pi, X}\qty[\pi(s' \mid s)\cdot \qty(\frac{E^{(1)}_{s, s'}(X)}{\epsilon + E^{(2)}_{s, s'}(X)} - E^{(3)}_{s}(X))] -\hat g(\beta)} &\lesssim  (1 + \beta) \cdot \frac{\log^{1/2} T}{\TT^{1/4}}\\
        &\lesssim e^\beta \frac{\log^{1/2} T}{\TT^{1/4}}.
    \end{align*}

    Finally, it suffices to show that $\hat g(\beta) \ge 0$. Define the function $h_s : \R \rightarrow \R$ by $$h_s(z) = \frac{(1 - r)e^\beta z^2 + re^{\beta \mu_\pi(s)}z}{(1 - r)(e^\beta - 1)\mu_\pi(s)z + (1 - r) + re^{\beta \mu_\pi(s)}} - \frac{(1-r)e^\beta  + r e^{\beta \mu_\pi(s)}}{(1-r)(e^\beta - 1)\mu_\pi(s) + (1 - r) + re^{\beta \mu_\pi(s)}}.$$
    Simplifying the formula for $\hat g(\beta)$, we see that it can be written in terms of this $h_s$:
    \begin{align*}
        \hat g(\beta) & = \frac{1}{S(S-1)}\sum_{s}\mathbb{E}_\pi\qty[\mu_\pi(s)\cdot \qty(\sum_{s'}\mu_\pi(s')h_s\qty(\frac{\pi(s' \mid s)}{\mu_\pi(s')}))].
    \end{align*}
    Furthermore, $h_s$ is convex, and so $\hat g(\beta)$ is actually a linear combination of $h_s$-divergences and is hence nonnegative. The following lemma relates the $h_s$-divergence to the $\chi^2$-divergence in order to get a quantitative lower bound on $\hat g(\beta)$ away from 0. The proof is deferred to \Cref{sec:proofs}.
    
    \begin{lemma}\label{lem:bound_g_beta}
        $ \hat g(\beta) \ge \frac12\gamma^8S^{-6}e^{-2\beta} > 0$.
    \end{lemma}

    To conclude, when $\beta \le \beta^*$,
    \begin{align*}
        \abs{-\Delta_\beta(\theta) - \frac{1}{S(S-1)}\sum_{s, s'}\mathbb{E}\qty[\pi(s' \mid s)q\qty(\tilde z(X; s))]} &\lesssim e^\beta T^{-1} \le \frac18\gamma^8S^{-6}e^{-2\beta}\\
        \abs{\frac{1}{S(S-1)}\sum_{s, s'}\E_{\pi, X}\qty[\pi(s' \mid s)\cdot \qty(\frac{E^{(1)}_{s, s'}(X)}{\epsilon + E^{(2)}_{s, s'}(X)} - E^{(3)}_{s}(X))] -\hat g(\beta)} &\lesssim e^\beta \frac{\log^{1/2} T}{\TT^{1/4}} \le \frac18\gamma^8S^{-6}e^{-2\beta},
    \end{align*}
    and thus
    \begin{align*}
        -\Delta_\beta(\theta) \ge \frac14\gamma^8S^{-6}e^{-2\beta},
    \end{align*}
    as desired.
\end{proof}

\begin{lemma}[Dynamics of $A^{(2)}$]\label{lem:A2_dynamics}
    Let $A^{(1)}(\tau_1)$ be the output of stage 1 of \Cref{alg:training_alg}, and let $\eta_2 \le 1$. There exists $\tau_2 \lesssim_{\gamma, S} e^{2\beta^*}\beta^*\eta_2^{-1}$ such that $$1 + \beta^* \ge \beta(\tau_1 + \tau_2) \ge \beta^*.$$
\end{lemma}

\begin{proof}
    If $\beta(t) \le \beta^*$, then by \Cref{thm:stage2}
    \begin{align*}
        \beta(t+1) \ge \beta(t) + \eta_2 \cdot \frac14\gamma^8S^{-6}e^{-2\beta(t)} \ge \beta(t) + \eta_2 \cdot \frac14\gamma^8S^{-6}e^{-2\beta^*}.
    \end{align*}
    Assume that $\beta(\tau_1 + t) < \beta^*$ for all $t \le \mathcal{T} := 4S^6\gamma^{-8}e^{2\beta^*}\beta^*\eta_2^{-1}$. Then
    \begin{align*}
        \beta(\tau_1 + \mathcal{T}) \ge  \frac14\gamma^8S^{-6}e^{-2\beta^*}\mathcal{T}\eta_2 = \beta^* ,
    \end{align*}
    a contradiction. Therefore $\beta(\tau_1 + \tau_2) \ge \beta^*$ for some $\tau_2 \le \mathcal{T} \lesssim e^{2\beta^*}\beta^*\eta_2^{-1}$. Finally, by \Cref{thm:stage2}, $\beta(t+1) \le \beta(t) + 1$, and thus letting $\tau_2$ be the smallest such time we have $1 + \beta^* \ge \beta(\tau_1 + \tau_2) \ge \beta^*$.
\end{proof}

\subsection{Proof of Theorem \ref{thm:main_thm}}

\begin{proof}[Proof of \Cref{thm:main_thm}]

    Pick $\beta_0 \le  c_{\gamma, S}\frac{1}{\TT^{3/2}}$, and set $\tau_1 = C_{\gamma, S}\eta_1^{-1}\beta_0^{-1}T^2\log(T)$ for constants $c_{\gamma, S}, C_{\gamma, S}$ chosen appropriately. By \Cref{cor:end_of_stage_1}, the output of stage 1 satisfies
    \begin{align*}
        1 - \S\qty(A^{(1)}(\tau_1))_{i, p(i)} \lesssim T^{-1}.
    \end{align*}
    for $i \in \overline{\calR}$.
    
    Next, by \Cref{lem:A2_dynamics} there exists $\tau_2 = \Tilde O\qty( \TT^{1/6}\eta_2^{-1})$ such that $\beta(\tau_1 + \tau_2) \ge \beta^*.$

    It now suffices to bound the loss of the predictor $\hat \theta$. We have
    \begin{align*}
        \abs{L(\htheta) - L^*} & \le \mathbb{E}_{\pi, X}\qty[\frac{1}{S}\sum_{s, s'}\pi(s' \mid s)\cdot \abs{\log\qty(f_{\htheta}(X; s)_{s'} + \epsilon) - \log\pi(s' \mid s)}]              \\
                               & = \mathbb{E}_{\pi}\qty[\frac{1}{S}\sum_{s, s'}\pi(s' \mid s)\mathbb{E}_X\qty[\abs{\log\qty(f_{\htheta}(X; s)_{s'} + \epsilon) - \log\pi(s' \mid s)}]] \\
    \end{align*}
    For $A, B > 0$, one has the bound
    \begin{align*}
        \abs{\log A - \log B} \le \frac{\abs{A - B}}{\min(A, B)}.
    \end{align*}
    Therefore
    \begin{align*}
        &\abs{\log\qty(f_{\htheta}(X; s)_{s'} + \epsilon) - \log\pi(s' \mid s)}\\
        & \le \qty(\abs{f_{\htheta}(X; s)_{s'} - \pi(s' \mid s)} + \epsilon)\cdot \frac{1}{\min(f_{\htheta}(X; s)_{s'} + \epsilon, \pi(s' \mid s))}                                                                         \\
                                                                               & \lesssim \qty(\abs{f_{\htheta}(X; s)_{s'} - \pi(s' \mid s)} + \epsilon)\qty(\mathbf{1}_{f_{\htheta}(X; s)_{s'} \ge \frac{\gamma^3}{4S}} + \epsilon^{-1}\mathbf{1}_{f_{\htheta}(X; s)_{s'} < \frac{\gamma^3}{4S}}),
    \end{align*}
    and thus by Cauchy
    \begin{align*}
         & \E_X\abs{\log\qty(f_{\htheta}(X; s)_{s'} + \epsilon) - \log\pi(s' \mid s)}                                                                                                      \\
         & \lesssim \qty(\qty(\E_X\abs{f_{\htheta}(X; s)_{s'} - \pi(s' \mid s)}^2)^{1/2} + \epsilon)\qty(1 + \epsilon^{-1}\mathbb{P}_X\qty(f_{\htheta}(X; s)_{s'} < \frac{\gamma^3}{4S})) \\
         & \lesssim \qty(\qty(\E_X\abs{f_{\htheta}(X; s)_{s'} - \pi(s' \mid s)}^2)^{1/2} + \epsilon)\qty(1 + \epsilon^{-1}\TT^{-1})                                                   \\
         & \lesssim \qty(\E_X\abs{f_{\htheta}(X; s)_{s'} - \pi(s' \mid s)}^2)^{1/2} + \epsilon,
    \end{align*}
    where the bound $\mathbb{P}_X\qty(f_{\htheta}(X; s)_{s'} < \frac{\gamma^3}{4S}) \lesssim \TT^{-1}$ follows from \Cref{lem:f_to_pi}. Altogether, applying \Cref{lem:f_to_pi} again, we get
    \begin{align*}
        \abs{L(\htheta) - L^*} & \lesssim \qty(\E_X\abs{f_{\htheta}(X; s)_{s'} - \pi(s' \mid s)}^2)^{1/2} + \epsilon \\
                               & \le \frac{e^{\beta^*}\log^{1/2} T}{\TT^{1/4}} + e^{-{\beta^*}\gamma/2} + \epsilon    \\
                               &\lesssim \frac{\log^{1/3}T}{\TT^{1/6}} + \qty(\frac{\log^{1/6}T}{\TT^{1/12}})^{\gamma/2}\\
                               &\lesssim \qty(\frac{\log^2 T}{\TT})^{\gamma/24}.
    \end{align*}
\end{proof}

\section{Markov Chain Preliminaries}

Given a Markov chain $\pi$ with stationary measure $\mu_\pi$, we define the normalized and centered transition matrix $B_\pi \in \R^{S \times S}$ by:
\begin{align*}
    (B_{\pi})_{s,s'} := \sqrt{\frac{\mu_\pi(s)}{\mu_\pi(s')}} [\pi(s'|s) - \mu_\pi(s')].
\end{align*}
An immediate consequence is that
\begin{align*}
    (B_{\pi}^k)_{s,s'} := \sqrt{\frac{\mu_\pi(s)}{\mu_\pi(s')}} [\pi^k(s'|s) - \mu_\pi(s')]
\end{align*}
which allows for the decomposition
\begin{align*}
    \pi^k(s'|s) = \mu_\pi(s') + (B_{\pi}^k)_{s,s'} \sqrt{\frac{\mu_\pi(s')}{\mu_\pi(s)}}.
\end{align*}

We also observe that
\begin{align}\label{eq:B_pi_norm}
    \norm{B_\pi}_F^2 = \sum_{s, s'} \qty(\frac{\mu_\pi(s)\pi(s' \mid s)^2}{\mu_\pi(s')} - \mu_\pi(s')\mu_\pi(s)) = \sum_{s, s'}\frac{\mu_\pi(s)\pi(s' \mid s)^2}{\mu_\pi(s')} - 1.
\end{align}

\begin{definition}[Spectral Gap]\label{def:spectral_gap}
    We say that a Markov chain $\pi$ with stationary measure $\mu_\pi$ has a spectral gap of $1-\lambda(\pi)$ where $\lambda(\pi) := \|B_\pi\|_2$.
\end{definition}

\begin{lemma}\label{lem:spectral_gap_bound}
    Let $\min_{s, s'} \pi(s \mid s') \ge \gamma/S$. Then $\lambda(\pi) \le 1 - \gamma/S$.
\end{lemma}
\begin{proof}
    By \Cref{lem:mu_lowerbound}, we can write
    \begin{align*}
        \pi = \frac{\gamma}{S} 1\mu_\pi^\top  + (1 - \gamma)Q
    \end{align*}
    for another stochastic matrix $Q$. One then sees that $\pi^\top Q = \pi$. Therefore
    \begin{align*}
        \pi - 1\mu_\pi^\top  = (1 - \gamma/S)(\pi - 1\mu_\pi^\top ),
    \end{align*}
    so
    \begin{align*}
        \|\pi - 1\mu_\pi^\top \|_{\mu_\pi} = (1 - \gamma/S)\|Q - 1\mu_\pi^\top \|_{\mu_\pi} \le 1 - \gamma/S.
    \end{align*}
    Therefore $\lambda(\pi) \le 1 - \gamma/S$.
\end{proof}

\begin{lemma}\label{lem:mu_lowerbound}
    Let $\min_{s, s'} \pi(s \mid s') \ge \gamma/S$. Then $\min_s \mu_\pi(s) \ge \gamma/S$.
\end{lemma}
\begin{proof}
    Since $\mu_\pi(s)$ is stationary,
    \begin{align*}
        \mu_\pi(s') & = \sum_s \pi(s' \mid s) \mu_\pi(s) \\
                    & \ge \sum_s \gamma/S \cdot \mu_\pi(s) \\
                    & = \gamma/S,
    \end{align*}
    as desired.
\end{proof}

\begin{lemma}\label{lem:TV_bound}
    Let $\min_{s, s'} \pi(s \mid s') \ge \gamma/S$. Then
    \begin{align*}
        \min_{j \neq k} \mathrm{TV}\qty(\pi(\cdot \mid j), \pi(\cdot \mid k)) \le 1 - \gamma.
    \end{align*}
\end{lemma}
\begin{proof}
    Write
    \begin{align*}
        \mathrm{TV}\qty(\pi(\cdot \mid j), \pi(\cdot \mid k)) & = \frac12 \sum_s \abs{\pi(s \mid j) - \pi(s \mid k)}                                       \\
                                                              & =\frac12\sum_s \qty(\pi(s \mid j) + \pi(s \mid k) - 2\min\{\pi(s \mid j), \pi(s \mid k)\}) \\
                                                              & \le 1 - \gamma.
    \end{align*}
\end{proof}



\begin{lemma}\label{lem:bound_g1}
$\|B_\pi\|_F^2 \ge  \gamma^2/S$
\end{lemma}
\begin{proof}
    By definition
    \begin{align*}
        \|B_\pi\|_F^2 = \sum_{s, s'} \frac{\pi(s' \mid s)^2\mu_\pi(s)}{\mu_\pi(s')} - 1,
    \end{align*}
    and thus
    \begin{align*}
        \|B_\pi\|_F^2 &= \sum_{s'}\frac{1}{\mu_\pi(s')}\qty(\sum_s \pi(s' \mid s)^2\mu_\pi(s) - \mu_\pi(s')^2)\\
        &\ge \sum_{s, s'} \mu_\pi(s) \qty(\pi(s' \mid s) - \mu_\pi(s'))^2\\
        &\ge \frac{\gamma^2}{S} \sum_s\mu_\pi(s)\\
        &= \frac{\gamma^2}{S}.
    \end{align*}
\end{proof}

\begin{lemma}[\cite{cohen1993}, Theorem 3.1]\label{lem:contraction}
    Let $\pi$ be a stochastic matrix such that $\max_s \pi(s' \mid s) > 0$ for all $s'$. Then, for any $f$-divergence $D_f$ and probability vectors $x, y$,
    \begin{align*}
        D_f(\pi \circ x || \pi \circ y) \le \alpha(\pi) D_f(x || y),
    \end{align*}
    where the contraction coefficient $\alpha(\pi)$ is defined as
    \begin{align*}
        \alpha(\pi) := \max_{j \neq k} \mathrm{TV}(\pi(\cdot \mid j), \pi(\cdot \mid k)) = \frac12 \max_{j \neq k} \|\pi(\cdot \mid j) - \pi(\cdot \mid k)\|_1.
    \end{align*}

\end{lemma}

\section{Concentration}

\begin{definition}[Graph Distance]
    Let $\mathcal{G}$ be the directed acyclic graph in \Cref{sec:dag_single_parent}. Let $\overline{G}$ denote the undirected version of $\mathcal{G}$. Then we define $d(i,j)$ to be length of the shortest path between $i,j$ in $\mathcal{G}$. If $i,j$ are not connected in $G$ then $d(i,j) := \infty$.
\end{definition}

\begin{definition}[Effective Sequence Length]\label{def:T_eff_lambda}
    For $\lambda \in (0,1)$, we define the effective sequence length $\TT(\lambda)$ by:
    \begin{align*}
        \TT(\lambda) := \frac{T^2}{\sum_{i,j = 1}^T  \lambda^{d(i,j)}}.
    \end{align*}
\end{definition}

This formula for $\TT(\lambda)$ is closely related to the definition of $\TT$ (\Cref{def:T_eff}):

\begin{lemma}\label{lem:bound_T_eff_leaves}
    Decompose $\mathcal{G} = \bigcup_{i=1}^k \mathcal{T}_i$ where $\mathcal{T}_i$ are disjoint trees. Let $L_i$ denote the number of leaves of tree $\mathcal{T}_i$ for $i=1,\ldots,k$. Then,
    \begin{align*}
        \TT(\lambda) \ge \frac{T(1-\lambda)}{\max_{i=1}^k L_i} =: (1 - \lambda) \TT
    \end{align*}
\end{lemma}
\begin{proof}
    Note that $\TT(\lambda)^{-1}$ naturally decomposes to a sum within each tree as $d(i,j) := \infty$ when $i$ and $j$ are not connected:
    \begin{align*}
        \frac{1}{\TT(\lambda)}
         & = \frac{1}{T^2} \sum_{l=1}^k \sum_{i,j \in \mathcal{T}_l} \lambda^{d(i,j)}                                                 \\
         & = \frac{1}{T^2} \sum_{l=1}^k \sum_{i,j \in \mathcal{T}_l} \lambda^{d(i,j)}                                                 \\
         & = \frac{1}{T^2} \sum_{l=1}^k \sum_{i \in \mathcal{T}_l} \sum_{k \ge 0} \#\{j \in \mathcal{T}_l ~:~ d(i,j)=k\} \lambda^{k}.
    \end{align*}
    Now note that for a fixed node $i$, each path from $i$ to $j$ with $d(i,j) = k$ can be lengthened to a path that reaches a leaf. Furthermore, for each leaf there can be only one such $j$. Therefore, $\#\{j \in \mathcal{T}_l ~:~ d(i,j)=k\} \le L_l.$ Plugging this in gives:
    \begin{align*}
        \frac{1}{\TT(\lambda)}
         & \le \frac{1}{T^2} \sum_{l=1}^k |\mathcal{T}_l| L_l \sum_{k \ge 0} \lambda^{k} \\
         & = \frac{\sum_{l=1}^k |\mathcal{T}_l| L_l}{T^2 (1-\lambda)}                    \\
         & \le \frac{\max_l T_l}{T (1-\lambda)}
    \end{align*}
    which completes the proof.
\end{proof}

Throughout the remainder of this section, the only assumption we place on $\pi$ is that $\min\pi(s' \mid s) \ge \gamma/S$. Defining $\lambda := 1 - \gamma/S$, we have that $\lambda \ge \lambda(\pi)$ by \Cref{lem:spectral_gap_bound}, and thus $\TT(\lambda)^{-1} \lesssim \TT^{-1}$.

\begin{lemma}
    For any $\pi$ and any $i,j < T$,
    \begin{align*}
        \abs{\P_X[x_j=s,x_i=s'] - \mu_\pi(s)\mu_\pi(s')} \le \sqrt{\mu_\pi(s) \mu_\pi(s')} \lambda(\pi)^{d(i,j)}.
    \end{align*}
\end{lemma}
\begin{proof}
    Let $k$ be the closest common parent of $i,j$ so that $d(k,i) + d(k,j) = d(i,j)$ and there exist directed paths from $k$ to $i$ and $k$ to $j$ in $\mathcal{G}$. Then,
    \begin{align*}
         & \P[s_j = s, s_i = s'] - \mu_\pi(s)\mu_\pi(s') \\
         &= \mathrm{Cov}[x_{j,s} x_{i,s'}] \\
         &= \E[(x_{j,s} - \mu_\pi(s))(x_{i,s'} - \mu_\pi(s'))] \\
         & = \sum_{s_k \in [S]} \mu_\pi(s_k) (\pi^{d(k,j)}(s'|s_k) - \mu_\pi(s')) (\pi^{d(k,i)}(s|s_k) - \mu_\pi(s)) \\
         &= \sum_{s_k \in [S]} \mu_\pi(s_k) \qty((B_\pi^{d(k,j)})_{s_k,s} \sqrt{\frac{\mu_\pi(s)}{\mu_\pi(s_k)}}) \qty((B_\pi^{d(k,i)})_{s_k,s'} \sqrt{\frac{\mu_\pi(s')}{\mu_\pi(s_k)}}) \\
         &= \sqrt{\mu_\pi(s)\mu_\pi(s')} \sum_{s_k \in [S]} (B_\pi^{d(k,j)})_{s_k,s}(B_\pi^{d(k,i)})_{s_k,s'} \\
         &= \sqrt{\mu_\pi(s)\mu_\pi(s')} [(B_\pi^{d(k,j)})^\top  (B_\pi^{d(k,i)})]_{s,s'}.
    \end{align*}
    Therefore taking absolute values gives:
    \begin{align*}
        \abs{\P[s_j = s, s_i = s'] - \mu_\pi(s)\mu_\pi(s')}
        &\le \sqrt{\mu_\pi(s)\mu_\pi(s')} \|B_\pi\|^{d(k,j)+d(k,i)} \\
        &\le \sqrt{\mu_\pi(s)\mu_\pi(s')} \lambda(\pi)^{d(i,j)}.
    \end{align*}
\end{proof}

\begin{lemma}\label{lem:muhat_variance}
    For any subset $I \subset [T-1]$, define
    \begin{align*}
        \hat \mu_{X_{I}}(s) := \frac{1}{|I|} \sum_{i \in I} x_{i,s}.
    \end{align*}
    Then,
    \begin{align*}
        \E_X[\hat \mu_{X_I}(s)] = \mu_\pi(s) \qq{and} \E_X[(\hat \mu_{X_{I}}(s) - \mu_\pi(s))^2] \le \frac{\mu_\pi(s) T^2}{\TT(\lambda) |I|^2}.
    \end{align*}
\end{lemma}
Note that \Cref{lem:muhat_variance} is excluding the token $x_T$ as it is resampled from $\unif([S])$.
\begin{proof}
    The first claim follows from the fact that $\E[x_{i,s}] = \mu_\pi(s)$ as the sequence $X$ is initialized from $\mu_\pi$. Then,
    \begin{align*}
        \E_X[(\hat \mu_{X_{I}}(s) - \mu_\pi(s))^2]
         & = \frac{1}{|I|^2} \sum_{i,j \in I} \E_X[x_{i,s} x_{j,s} - \mu_\pi(s)^2] \\
         & \le \frac{\mu_\pi(s)}{|I|^2} \sum_{i,j \in I} \lambda^{d(i,j)}          \\
         & \le \frac{\mu_\pi(s)}{|I|^2} \sum_{i,j=1}^{T-1} \lambda^{d(i,j)}          \\
         & = \frac{\mu_\pi(s) (T-1)^2}{\TT(\lambda) |I|^2}
    \end{align*}
    which completes the proof.
\end{proof}

\begin{corollary}\label{cor:muhat_X}
\begin{align*}
    \E_X[(\hat \mu_{X}(s) - \mu_\pi(s))^2] \lesssim \frac{1}{\TT(\lambda)}.
\end{align*}
\end{corollary}
\begin{proof}
    One can write
    \begin{align*}
        \hat \mu_{X}(s) = \frac{T-1}{T}\hat \mu_{X_{[T-1]}}(s) + \frac{1}{T}x_{T, s}.
    \end{align*}
    Thus
    \begin{align*}
        \E_X[(\hat \mu_{X}(s) - \mu_\pi(s))^2] \le \qty(\frac{T-1}{T})^2\E_X\qty[\qty(\hat \mu_{X_{[T-1]}}(s) - \mu_\pi(s))^2] + \frac{1}{T^2} \lesssim \frac{1}{\TT(\lambda)}.
    \end{align*}
\end{proof}

\begin{lemma}\label{lem:chat_variance}
    For any subset $I \subset [T-1]$ such that $p(i) \ne \emptyset$ for all $i \in I$, define
    \begin{align*}
        \hat c_{X_{I}}(s,s') := \frac{1}{|I|} \sum_{i \in I} x_{p(i),s} x_{i,s}.
    \end{align*}
    Then,
    \begin{align*}
        \E_X[\hat c_{X_I}(s,s')] = \mu_\pi(s)\pi(s'|s) \qq{and} \E_X[(\hat \mu_{X_{I}}(s) - \mu_\pi(s)\pi(s'|s))^2] \lesssim \frac{T^2}{\TT(\lambda) |I|^2}.
    \end{align*}
\end{lemma}
\begin{proof}
    The first result follows from linearity of expectation and the fact that the Markov process is stationary. Then,
    \begin{align*}
        &\E_X[(\hat \mu_{X_{I}}(s) - \mu_\pi(s)\pi(s'|s))^2] \\
        &= \frac{1}{|I|^2}\sum_{i,j \in I} \E[x_{p(i),s}x_{i,s'} x_{p(j),s} x_{j,s'}] - \mu_\pi(s)^2 \pi(s'|s)^2.
    \end{align*}
    There are three possibilities for the dependency graph of $i,j$. First, if $i=j$ the expression in the sum is equal to $\mu_\pi(s) \pi(s'|s)(1-\mu_\pi(s) \pi(s'|s))$. Next, if $i,j$ are independent conditioned on $p(i),p(j)$, we get
    \begin{align*}
        &\E[x_{p(i),s}x_{i,s'} x_{p(j),s} x_{j,s'}] - \mu_\pi(s)^2 \pi(s'|s)^2 \\
        &= \pi(s'|s)^2 (\E[x_{p(i),s}x_{p(j),s}] - \mu_\pi(s)^2) \\
        &\le \mu_\pi(s) \pi(s'|s)^2 \lambda^{d(p(i),p(j))}.
     \end{align*}
     Finally, if $i,j$ are dependent conditioned on $p(i),p(j)$ it means that either there is a directed path from $i$ to $p(j)$ or a direct path from $j$ to $p(i)$ in the directed graph $\mathcal{G}$. Without loss of generality, we can assume that there is a directed path from $j$ to $p(i)$. Then we have:
     \begin{align*}
         &\mu_\pi(s)\pi(s'|s)\pi^{d(j,p(i))}(s|s')\pi(s'|s) - \mu_\pi(s)^2 \pi(s'|s)^2 \\
         &= \mu_\pi(s) \pi(s'|s)^2 \qty[\pi^{d(j,p(i))}(s|s') - \mu_\pi(s)] \\
         &\le \sqrt{\mu_\pi(s)\mu_\pi(s')} \pi(s'|s)^2 \lambda^{d(j,p(i))}.
     \end{align*}
     Therefore,
     \begin{align*}
        &\E_X[(\hat \mu_{X_{I}}(s) - \mu_\pi(s)\pi(s'|s))^2] \\
        &\lesssim \frac{1}{|I|^2}\sum_{i,j \in I} \lambda^{d(i,j)} \\
        &\lesssim \frac{T^2}{\TT(\lambda) |I|^2}.
    \end{align*}
\end{proof}

\begin{lemma}\label{lem:muhat_parent_variance}
    For any subset $I \subset [T-1]$ such that $p(i) \ne \emptyset$ for all $i \in I$,
    \begin{align*}
        \E_X\qty[\qty(\frac{1}{|I|} \sum_{i \in I} x_{p(i),s} - \mu_\pi(s))^2] \lesssim \frac{T^2}{\TT(\lambda) |I|^2}.
    \end{align*}
\end{lemma}
\begin{proof}
    As above, we will directly compute the second moment:
    \begin{align*}
        \frac{1}{|I|^2} \sum_{i,j \in I} x_{p(i),s} x_{p(j),s} - \mu_\pi(s)^2
        &\le \frac{\mu_\pi(s)}{|I|^2} \sum_{i,j \in I} \lambda^{d(p(i),p(j))} \\
        &\le \frac{\mu_\pi(s)}{|I|^2} \sum_{i,j \in I} \lambda^{d(i,j)-2} \\
        &\le \frac{\mu_\pi(s)}{\lambda^2 |I|^2} \sum_{i,j \in T} \lambda^{d(i,j)} \\
        &\le \frac{T^2 \mu_\pi(s)}{\TT(\lambda) \lambda^2 |I|^2}.
    \end{align*}
\end{proof}

\section{Lemmas for Stage 1}

\subsection{Strong Data Processing Inequality}\label{sec:proofs}

We briefly recall the definition of the $\chi^2$ divergence between two probability distributions on state space $\mathcal{X}$:
\begin{align*}
    \chi^2(P || Q) := \sum_{x \in \mathcal{X}}\frac{P(x)^2}{Q(x)} - 1,
\end{align*}
along with the $\chi^2$ mutual information between two random variables $Y, Z$
\begin{align*}
    I_{\chi^2}(Y; Z) = \sum_{y, z \in \mathcal{X}}\frac{P(Y=y, Z=z)^2}{P(Y=y)P(Z=z)} - 1
\end{align*}

\begin{proof}[Proof of \Cref{thm:DPI}]

    First we consider the case where $i$ and $j$ are in separate trees. If $i \neq T$, then $\mathbb{P}_X[s_i = s', s_j = s] = \mu_\pi(s)\mu_\pi(s')$, and thus
    \begin{align*}
        g_{i,j}(\pi) = \sum_{s, s'}\pi(s' \mid s)\mu_\pi(s) - 1 = 0.
    \end{align*}
    We note that this subsumes the case where $i$ is a root note, since that necessarily implies that $j$ is in a different tree. Otherwise when $i = T$,
    \begin{align*}
        g_{i,j}(\pi) = \frac{1}{S}\sum_{s, s'}\frac{\pi(s' \mid s)\mu_\pi(s)}{\mu_\pi(s')} - 1 = \frac{1}{S}\sum_{s'}\frac{\mu_\pi(s')}{\mu_\pi(s')} - 1 =  0.
    \end{align*}

    Next, assume that $i$ and $j$ are in the same tree. When $j = p(i)$, we have
    \begin{align*}
        g_{i, p(i)}(\pi) & = \sum_{s, s'} \frac{\pi(s' \mid s)}{\mu_\pi(s')}\cdot \mathbb{P}_X[s_i = s', s_j = s] - 1 \\
                         & = \sum_{s, s'} \frac{\pi(s' \mid s)^2\mu_\pi(s)}{\mu_\pi(s')} - 1                          \\
                         & = \|B_\pi\|_F^2,
    \end{align*}
    where the last equality is \eqref{eq:B_pi_norm}.
    
    If $j \neq p(i)$ and $j \neq i$, then by AM-GM:
    \begin{align*}
        g_{i,j}(\pi) & =  \sum_{s, s'} \frac{\pi(s' \mid s)}{\mu_\pi(s')}\cdot \mathbb{P}_X[s_i = s', s_j = s] - 1                                                      \\
                     & \le \frac12\sum_{s,s'} \frac{\mu_\pi(s)\pi(s' \mid s)^2}{\mu_\pi(s')} + \frac12\sum_{s,s'} \frac{\mathbb{P}_X[s_i = s', s_j = s]^2}{\mu_\pi(s)\mu_\pi(s')} - 1 \\
                     & = \frac12\|B_\pi\|_F^2 + \frac12 I_{\chi^2}(s_i; s_j \mid \pi).
    \end{align*}
    We see that the $\chi^2$-mutual information can be rewritten as
    \begin{align*}
        I_{\chi^2}(s_i; s_j \mid \pi) = \sum_{s'}\mu_\pi(s') \cdot \chi^2\qty(\mathbb{P}_X\qty[s_j = \cdot \mid s_i = s'] ||~\mu_\pi).
    \end{align*}
    Let $p(i,j)$ be the least common ancestor of $i$ and $j$. Let $x$ be the probability distribution defined by $x = \mathbb{P}_X\qty[s_{p(i,j)} = \cdot \mid s_i = s'] $. The distribution $\pi^{d(j, p(i,j))} \circ x$ is
    \begin{align*}
        (\pi^{d(j, p(i,j))} \circ x)(s) & = \sum_{s^*}\pi^{d(j, p(i,j))}(s \mid s^*)\cdot x(s^*)                                                          \\
                                        & = \sum_{s^*}\mathbb{P}_X[s_{j} = s \mid s_{p(i,j)} = s^*]\cdot \mathbb{P}_X\qty[s_{p(i,j)} = s^* \mid s_i = s'] \\
                                        & = \mathbb{P}_X\qty[s_j = s \mid s_i = s'],
    \end{align*}
    where the last line uses the fact that $s_i$ and $s_j$ are conditionally independent given $p(i,j)$.

    Applying \Cref{lem:contraction}, we thus have
    \begin{align*}
        \chi^2\qty(\mathbb{P}_X\qty[s_j = \cdot \mid s_i = s'] ||~\mu_\pi) \le \alpha(\pi)^{d(j, p(i,j))} \cdot \chi^2\qty(\mathbb{P}_X\qty[s_{p(i,j)} = \cdot \mid s_i = s'] ||~\mu_\pi).
    \end{align*}
    Therefore
    \begin{align*}
        I_{\chi^2}(s_i; s_j \mid \pi) & \le \alpha(\pi)^{d(j, p(i,j))}\sum_{s'}\mu_\pi(s') \cdot \chi^2\qty(\mathbb{P}_X\qty[s_{p(i,j)} = \cdot \mid s_i = s'] ||~\mu) \\
                             & = \alpha(\pi)^{d(j, p(i,j))} \cdot I_{\chi^2}(s_{p(i, j)}; s_i \mid \pi)\\
                             & = \alpha(\pi)^{d(j, p(i,j))} \sum_s \mu_\pi(s) \cdot \chi^2\qty(\mathbb{P}_X\qty[s_i = \cdot \mid s_{p(i,j)} = s] ||~\mu )     \\
                             & = \alpha(\pi)^{d(j, p(i,j))} \sum_s \mu_\pi(s) \cdot \chi^2\qty(\pi^{d(i, p(i,j))}(\cdot \mid s) ||~\mu ).
    \end{align*}
    Since $i > j$, $d(i, p(i,j)) \ge 1$, and thus we can apply \Cref{lem:contraction} to get
    \begin{align*}
        \chi^2\qty(\pi^{d(i, p(i,j))}(\cdot \mid s) ||~\mu ) \le \alpha(\pi)^{d(i,p(i,j))-1} \cdot \chi^2\qty(\pi(\cdot \mid s) ||~\mu ).
    \end{align*}
    Altogether,
    \begin{align*}
        I_{\chi^2}\qty(s_i; s_j \mid \pi) & \le \alpha(\pi)^{d(j, p(i,j)) + d(i, p(i,j)) - 1} \sum_s \mu_\pi(s) \cdot \chi^2\qty(\pi(\cdot \mid s) ||~\mu ) \\
                                 & = \alpha(\pi)^{d(i,j) - 1} \cdot \qty(\sum_{s, s'} \frac{\pi(s' \mid s)^2\mu_\pi(s)}{\mu_\pi(s')} - 1   )             \\
                                 & = \alpha(\pi)^{d(i,j) - 1} \|B_\pi\|_F^2.
    \end{align*}
    For $j \neq p(i), d(i,j) \ge 2$, so
    \begin{align*}
        g_{i,j}(\pi) & \le \frac12\qty(\alpha(\pi)^{d(i,j) - 1} + 1)\|B_\pi\|_F^2 \\
                     & \le \frac12\qty(\alpha(\pi) + 1)\|B_\pi\|_F^2.
    \end{align*}
    and thus
    \begin{align*}
        g_{i, p(i)}(\pi) - g_{i,j}(\pi) \ge \frac{1 - \alpha(\pi)}{2}\cdot \|B_\pi\|_F^2.
    \end{align*}
    By \Cref{assume:pi_prior} and \Cref{lem:TV_bound}, we have $1 - \alpha(\pi) \ge \gamma$ and $\|B_\pi\|_F^2 \ge \gamma^2/S$. Therefore
    \begin{align*}
        g_{i, p(i)}(\pi) - g_{i,j}(\pi) \ge \frac{\gamma^3}{2S}.
    \end{align*}
    Finally, when $j = i$, we have
    \begin{align*}
        g_{i,i}(\pi) = \sum_s \pi(s \mid s) - 1.
    \end{align*}
    Therefore
    \begin{align*}
        g_{i,i} = \sum_s\mathbb{E}[\pi(s \mid s)] - 1 = 0.
    \end{align*}
    Therefore $g_{i, p(i)} - g_{i, i} \ge \gamma^2/S \ge \frac{\gamma^3}{2S}$.
\end{proof}

\subsection{Auxiliary Dynamics Lemmas}

\begin{lemma}\label{lem:nonzero_beta}
    Let $\theta = (A^{(1)}, \beta_0I_S), \hat \theta = (A^{(1)}, 0),$ for $\beta_0 \le 1$. Define $g^*_i, \hat g_i \in \R^i$ by
    \begin{align*}
        g^*_i := T\sum_{s, s'}\E\qty[\frac{\pi(s' \mid s)}{f_\theta(X; s)_{s'} + \epsilon}\delta_{s'}(X)^\top J(v_\theta(X; s))e_i \cdot  \delta_s(X_{\le i})], \\
        \hat g_i := T\sum_{s, s'}\E\qty[\frac{\pi(s' \mid s)}{f_{\htheta}(X; s)_{s'} + \epsilon}\delta_{s'}(X)^\top J(v_{\htheta}(X; s))e_i \cdot  \delta_s(X_{\le i})].
    \end{align*}
    Then $\norm{g^*_{i} - \hat g_{i}}_\infty \le 3S^2\epsilon^{-2}(e^{\beta_0} - 1)$
\end{lemma}

\begin{proof}
    We can bound
    \begin{align*}
         & \abs{\frac{1}{f_\theta(X; s)_{s'} + \epsilon}\delta_{s'}(X)^\top J(v_\theta(X; s))e_i - \frac{1}{f_{\htheta}(X; s)_{s'} + \epsilon}\delta_{s'}(X)^\top J(v_{\htheta}(X; s))e_i} \\
         & \quad\le \abs{\frac{1}{f_\theta(X; s)_{s'} + \epsilon} - \frac{1}{f_{\htheta}(X; s)_{s'} + \epsilon}}\abs{\delta_{s'}(X)^\top J(v_{\htheta}(X; s))e_i}                      \\
         & \quad+ \frac{1}{f_{\htheta}(X; s)_{s'} + \epsilon}\abs{\delta_{s'}(X)^\top \qty(J(v_\theta(X; s)) - J(v_{\htheta}(X; s)))e_i}.
    \end{align*}
    First, see that
    \begin{align*}
        \abs{f_\theta(X; s)_{s'} - f_{\htheta}(X; s)_{s'}} & = \abs{\delta_{s'}(X)^\top \qty(v_\theta(X; s) - v_{\htheta}(X; s))} \\
                                                              & \le \norm{v_\theta(X; s) - v_{\htheta}(X; s)}_1,
    \end{align*}
    since $\norm{\delta_{s'}(X)}_\infty \le 1$. Next, we have
    \begin{align*}
        v_\theta(X; s) & = \S\qty(\beta_0 \cdot \S(A^{(1)})X^\top I_S e_s)) = \S\qty(\beta_0 \cdot \S(A^{(1)})\delta_s(X))).
    \end{align*}
    Since $\S(A^{(1)})\delta_s(X)$ has entries in $[0, 1]$, we can bound each entry of $v_\theta(X; s)$ as
    \begin{align*}
        \frac{1}{(T-1)e^{\beta_0} + 1} \le v_\theta(X; s)_i & \le \frac{e^{\beta_0}}{e^{\beta_0} + (T-1)},
    \end{align*}
    and thus
    \begin{align*}
        \abs{v_\theta(X; s)_i - v_{\htheta}(X; s)_i} = \abs{v_\theta(X; s)_i - \frac{1}{T}} \le \frac{e^{\beta_0}}{e^{\beta_0} + (T-1)} - \frac{1}{T} \le \frac{e^{\beta_0} - 1}{T}.
    \end{align*}
    Thus
    \begin{align}\label{eq:bound_f}
        \abs{f_\theta(X; s)_{s'} - f_{\htheta}(X; s)_{s'}} \le e^{\beta_0} - 1.
    \end{align}
    Next, see that
    \begin{align*}
        \delta_{s'}(X)^\top J(v_\theta(X; s))e_i = v_\theta(X; s)_i\qty[x_{i, s'} - f_\theta(X; s)_{s'}],
    \end{align*}
    and thus
    \begin{align*}
         & \abs{\delta_{s'}(X)^\top \qty(J(v_\theta(X; s)) - J(v_{\htheta}(X; s)))e_i}                                                                                             \\
         &\quad\le \abs{v_\theta(X; s)_i - v_{\htheta}(X; s)_i}\abs{x_{i, s'} - f_\theta(X; s)_{s'}} + v_{\htheta}(X; s)_i\abs{f_\theta(X; s)_{s'} - f_{\htheta}(X; s)_{s'}} \\
         &\quad \le \frac{2(e^{\beta_0} - 1)}{T}.
    \end{align*}
    Altogether, we have the bound
    \begin{align*}
        \abs{\frac{1}{f_\theta(X; s)_{s'} + \epsilon}\delta_{s'}(X)^\top J(v_\theta(X; s))e_i - \frac{1}{f_{\htheta}(X; s)_{s'} + \epsilon}\delta_{s'}(X)^\top J(v_{\htheta}(X; s))e_i} \le \frac{3(e^{\beta_0} - 1)}{\epsilon^2T}.
    \end{align*}
    Therefore
    \begin{align*}
        &\norm{g^*_i - \hat g_i}\\
        &\le T\sum_{s, s'}\E_{\pi, X}\qty[\pi(s' \mid s)\abs{\frac{1}{f_\theta(X; s)_{s'} + \epsilon}\delta_{s'}(X)^\top J(v_\theta(X; s))e_i - \frac{1}{f_{\htheta}(X; s)_{s'} + \epsilon}\delta_{s'}(X)^\top J(v_{\htheta}(X; s))e_i}]\\
        &\le T\sum_{s, s'} E_{\pi, X}\qty[\pi(s' \mid s)\cdot \frac{3(e^{\beta_0} - 1)}{\epsilon^2T}]\\
        &\le 3S\epsilon^{-2}(e^{\beta_0} - 1)\\
        &\le 6S\epsilon^{-2}\beta_0,
    \end{align*}
    since $e^z - 1 \le 2z$ for $z \in [0, 1]$.
\end{proof}

\subsection{Concentration}

\begin{lemma}\label{lem:concentration_A1}
    For any $s,s' \in \mathcal{S}$ and any $\pi$ with spectral gap $1-\lambda(\pi) \ge 1-\lambda$ (see \Cref{def:spectral_gap}) and $\mu_\pi(s') \ge \gamma/S$, there exists a sufficiently large constant $C_{\gamma,S}$ such that if $\epsilon \ge C_{\gamma,S}\TT^{-1/2}$ and $i \ge j$,
    \begin{align*}
        \abs{\mathbb{E}_X\qty[\frac{(x_{i,s'} - \hat \mu_X(s'))x_{j,s}}{\hat \mu_X(s') + \epsilon}] - \qty(\frac{\mathbb{P}_X[s_i = s', s_j = s]}{\mu_\pi(s')} -\mathbb{P}_X[s_j = s])} \lesssim \frac{1}{\sqrt{\TT}}.
    \end{align*}
\end{lemma}
\begin{proof}
    \begin{align*}
        E_\pi(s,s')
         & := \mathbb{E}_X\qty[\frac{(x_{i,s'} - \hat \mu_X(s'))x_{j,s}}{\hat \mu_X(s') + \epsilon}] - \frac{\mathbb{P}_X[s_i = s', s_j = s]}{\mu_\pi(s')} + \mathbb{P}_X[s_j = s]                                                    \\
         & = \mathbb{E}_X\qty[\frac{x_{i,s'} x_{j,s}}{\hat \mu_X(s') + \epsilon}] - \frac{\mathbb{P}_X[s_i = s', s_j = s]}{\mu_\pi(s')} - \mathbb{E}_X\qty[\frac{\hat \mu_X(s')}{\hat \mu_X(s') + \epsilon}  x_{j,s}] + \mathbb{P}_X[s_j = s].
    \end{align*}
    $E_\pi(s,s')$ can be rewritten as:
    \begin{align*}
        E_\pi(s,s')
         & = \mathbb{E}_X\qty[\frac{x_{i,s'} x_{j,s}}{\hat \mu_X(s') + \epsilon} - \frac{x_{i,s'} x_{j,s}}{\mu_\pi(s')} - \frac{\hat \mu_X(s')}{\hat \mu_X(s') + \epsilon}  x_{j,s} + x_{j,s}] \\
         & = \mathbb{E}_X\qty[\frac{x_{i,s'} x_{j,s}}{\hat \mu_X(s') + \epsilon} - \frac{x_{i,s'} x_{j,s}}{\mu_\pi(s')} + \frac{\epsilon x_{j,s}}{\hat \mu_X(s') + \epsilon}]              \\
         & = \mathbb{E}_X\qty[\frac{x_{i,s'} x_{j,s} [\mu_\pi(s') - \hat \mu_X(s') - \epsilon] + \epsilon x_{j,s} \mu_\pi(s')}{(\hat \mu_X(s') + \epsilon)\mu_\pi(s')}]
    \end{align*}
    Note that the inside of the expectation is upper bounded by $O(\epsilon^{-1})$. Therefore by the triangle inequality we have
    \begin{align*}
        \abs{E_\pi(s,s')} & \le \mathbb{E}_X\qty[\frac{x_{i,s'}x_{j,s}\abs{\hat \mu_X(s') - \mu_\pi(s')} + \epsilon [x_{i,s'}x_{j,s} + \mu_\pi(s') x_{j,s}]}{(\hat \mu_X(s') + \epsilon) \mu_\pi(s')}]                                                    \\
                          & = \mathbb{E}_X\qty[\frac{x_{i,s'}x_{j,s}\abs{\hat \mu_X(s') - \mu_\pi(s')} + \epsilon [x_{i,s'}x_{j,s} + \mu_\pi(s') x_{j,s}]}{(\hat \mu_X(s') + \epsilon) \mu_\pi(s')} \1_{\hat \mu_X(s') > \frac{\mu_\pi(s')}{2}}]          \\
                          & \qquad + \mathbb{E}_X\qty[\frac{x_{i,s'}x_{j,s}\abs{\hat \mu_X(s') - \mu_\pi(s')} + \epsilon [x_{i,s'}x_{j,s} + \mu_\pi(s') x_{j,s}]}{(\hat \mu_X(s') + \epsilon) \mu_\pi(s')} \1_{\hat \mu_X(s') \le \frac{\mu_\pi(s')}{2}}] \\
                          & \lesssim \mathbb{E}_X\qty[\frac{x_{i,s'}x_{j,s}\abs{\hat \mu_X(s') - \mu_\pi(s')} + \epsilon [x_{i,s'}x_{j,s} + \mu_\pi(s') x_{j,s}]}{\mu_\pi(s')^2}]                                                                         \\
                          & \qquad + \epsilon^{-1} \P_X\qty[\hat \mu_X(s') \le \frac{\mu_\pi(s')}{2}]                                                                                                                                                     \\
                          & \lesssim \sqrt{\mathbb{E}[(\hat \mu_X(s') - \mu_\pi(s'))^2]} + \epsilon + \frac{1}{\epsilon \TT}                                                                                                                              \\
                          & \lesssim \frac{1}{\sqrt{\TT}} + \epsilon,
    \end{align*}
    where the last inequality follows from \Cref{cor:muhat_X}.
\end{proof}

\section{Lemmas for Stage 2}

\subsection{Idealized Gradient}

\begin{proof}[Proof of \Cref{lem:bound_g_beta}]

    Recall $$h_s(z) = \frac{(1 - r)e^\beta z^2 + re^{\beta \mu_\pi(s)}z}{(1 - r)(e^\beta - 1)\mu_\pi(s)z + (1 - r) + re^{\beta \mu_\pi(s)}} - \frac{(1-r)e^\beta  + r e^{\beta \mu_\pi(s)}}{(1-r)(e^\beta - 1)\mu_\pi(s) + (1 - r) + re^{\beta \mu_\pi(s)}}.$$
    and
    \begin{align*}
        \hat g(\beta) & = \frac{1}{S(S-1)}\sum_{s}\mathbb{E}_\pi\qty[\mu_\pi(s)\cdot \qty(\sum_{s'}\mu_\pi(s')h_s\qty(\frac{\pi(s' \mid s)}{\mu_\pi(s')}))].
    \end{align*}

    For a function $h(z) = \frac{Az^2 + Bz}{Cz + D}$, one has
    \begin{align*}
        h''(z) = \frac{2D(AD - BC)}{(Cz + D)^3}.
    \end{align*}
    Thus for $z \in [0, S\gamma^{-1}]$,
    \begin{align*}
        h_s''(z) &= \frac{2\qty(1 - r + re^{\beta \mu_\pi(s)})\cdot \qty(1 - r) \cdot\qty(e^\beta(1 - r) + re^{\beta \mu_\pi(s) + \beta} - r(e^\beta - 1)\mu_\pi(s)e^{\beta\mu_\pi(s)})}{\qty((1 - r)(e^\beta - 1)\mu_\pi(s)z + (1 - r) + re^{\beta \mu_\pi(s)})^3}\\
        &\ge \frac{2(1 - r)^2 e^\beta}{\qty((1 - r)(e^\beta - 1)\mu_\pi(s)S\gamma^{-1} + (1 - r) + re^{\beta \mu_\pi(s)})^3}\\
        &\ge \frac{2(1 - r)^2 e^\beta}{\qty(S\gamma^{-1}e^\beta)^3}\\
        &\ge 2\gamma^5S^{-3}e^{-2\beta}.
    \end{align*}

    Therefore for $z \in [0, S\gamma^{-1}]$,
    \begin{align*}
        h_s(z) \ge h_s'(1)(z - 1) + \gamma^5S^{-3}e^{-2\beta} \cdot (z - 1)^2.
    \end{align*}
    Note that $\frac{\pi(s' \mid s)}{\mu_\pi(s')} \le \frac{S}{\gamma}$. Therefore
    \begin{align*}
         h_s\qty(\frac{\pi(s' \mid s)}{\mu_\pi(s')}) \ge h_s'(1)\qty(\frac{\pi(s' \mid s)}{\mu_\pi(s')} - 1) + \gamma^5S^{-3}e^{-2\beta} \cdot\qty(\frac{\pi(s' \mid s)}{\mu_\pi(s')} - 1)^2
    \end{align*}
    and thus
    \begin{align*}
        \qty(\sum_{s'}\mu_\pi(s')h_s\qty(\frac{\pi(s' \mid s)}{\mu_\pi(s')})) &\ge \gamma^5S^{-3}e^{-2\beta}\sum_{s'}\mu_\pi(s')\qty(\frac{\pi(s' \mid s)}{\mu_\pi(s')} - 1)^2\\
        &= \gamma^5S^{-3}e^{-2\beta}\chi^2\qty(\pi(\cdot \mid s) || \mu_\pi).
    \end{align*}
        Altogether,
    \begin{align*}
        \hat g(\beta) & = \frac{1}{S(S-1)}\sum_{s}\mathbb{E}_\pi\qty[\mu_\pi(s)\cdot \qty(\sum_{s'}\mu_\pi(s')h_s\qty(\frac{\pi(s' \mid s)}{\mu_\pi(s')}))] \\
                      & \ge \gamma^5S^{-5}e^{-2\beta}\mathbb{E}_\pi\qty[\sum_s \mu_\pi(s)\chi^2((\pi(\cdot \mid s) || \mu)] \\
                      & = \gamma^5S^{-5}e^{-2\beta}\mathbb{E}_\pi\qty[\|B_\pi\|_F^2]\\
                      & = \frac12\gamma^8S^{-6}e^{-2\beta}.
    \end{align*}
\end{proof}

\subsection{Auxiliary Dynamics Lemmas}
\begin{lemma}\label{lem:compare_to_shift}
    Define
    \begin{align*}
        q_{s'}(z) = \frac{\delta_{s'}(X)^\top J(\S(\beta z))z}{\delta_{s'}(X)^\top \S(\beta z) + \epsilon},
    \end{align*}
    Then $\sup_{z \in [0, 1]^T }\norm{\nabla q_{s'}(z)}_1 \le 10(1 + \beta)$.
\end{lemma}
\begin{proof}
    We have that
    \begin{align*}
        \nabla_z q_{s'}(z) = \frac{J(\beta z)\delta_{s'}(X) + \beta\nabla J(\S(\beta z))(\delta_{s'}(X), z)}{\delta_{s'}(X)^\top \S(\beta z) + \epsilon} - \frac{\delta_{s'}(X)^\top J(\beta z)z\cdot \beta J(\beta z)\delta_{s'}(X)}{(\delta_{s'}(X)^\top \S(\beta z) + \epsilon)^2}.
    \end{align*}
    First, by \Cref{lem:J_l1},
    \begin{align*}
        \|J(\beta z)\delta_{s'}(X)\|_1 \le 2 \delta_{s'}(X)^\top \S(\beta z).
    \end{align*}
    Next, by \Cref{lem:grad_J},
    \begin{align*}
        \norm{\nabla J(\S(\beta z))(\delta_{s'}(X), z)}_1 \le 2 \S(\beta z)^\top (\delta_{s'}(X) \odot z) + 4\S(\beta z)^\top \delta_{s'}(X) \S(\beta z)^\top z \le 6\S(\beta z)^\top \delta_{s'}(X),
    \end{align*}
    where the last inequality uses the fact that $z$ has entries in $[0, 1]$. Finally,
    \begin{align*}
        \abs{\delta_{s'}(X)^\top J(\beta z)z\cdot J(\beta z)\delta_{s'}(X)} \le  \|J(\beta z)\delta_{s'}(X)\|_1 \cdot \|J(\beta z)\delta_{s'}(X)\|_1 \cdot \|z\|_\infty \le 4(\delta_{s'}(X)^\top \S(\beta z))^2.
    \end{align*}

    Altogether,
    \begin{align*}
        \|\nabla_zq_{s'}(z)\|_1 \le \frac{(2 + 6\beta)\delta_{s'}(X)^\top \S(\beta z)}{\delta_{s'}(X)^\top \S(\beta z) + \epsilon} + \frac{4\beta(\delta_{s'}(X)^\top \S(\beta z))^2}{(\delta_{s'}(X)^\top \S(\beta z) + \epsilon)^2} \le 2 + 10\beta.
    \end{align*}
\end{proof}

\begin{lemma}\label{lem:J_l1}
    Let $u$ be a vector with nonnegative entries. Then $\|J(\S(v))u\|_1 \le 2\S(v)^\top u$
\end{lemma}
\begin{proof}
    \begin{align*}
        \|J(\S(v))u\|_1 = \sum_i \abs{\S(v)_i(u_i - \S(v)^\top u)} \le \sum_i \S(v)_i u_i + \S(v)^\top u\cdot \sum_i \S(v)_i = 2\S(v)^\top u.
    \end{align*}
\end{proof}

\begin{lemma}\label{lem:grad_J}
    Recall that $J(s) = \text{diag}(s) - ss^\top $. Then $\nabla_v J(\S(v)) \in \mathbb{R}^{d \times d \times d}$ satisfies
    \begin{align*}
        \|\nabla J(\S(v))(u, w)\|_1 \le 2 \S(v)^\top (u \odot w) + 4 \S(v)^\top u \S(v)^\top w.
    \end{align*}
    for nonnegative vectors $u, w$.
\end{lemma}
\begin{proof}
    See that
    \begin{align*}
        J(\S(v))(u, w) = u^\top \text{diag}(\S(v))w - \S(v)^\top u\S(v)^\top w = \S(v)^\top (u \odot w) - \S(v)^\top u\S(v)^\top w.
    \end{align*}
    Taking the gradient, and noting that $\nabla_v\S(v) = J(v)$, we get
    \begin{align*}
        \nabla J(\S(v))(u, w) = J(\S(v))(u \odot w)- \S(v)^\top w \cdot J(\S(v))u - \S(v)^\top u \cdot J(\S(v))w.
    \end{align*}
    Since $u \odot w$ is also a nonnegative vector, we get that
    \begin{align*}
        \|\nabla J(\S(v))(u, w)\|_1 \le 2 \S(v)^\top (u \odot w) + 4 \S(v)^\top u \S(v)^\top w.
    \end{align*}
\end{proof}

\subsection{Concentration}

\begin{lemma}\label{lem:compute_diffs}
For any nonzero scalars $A_1, A_2, B_1, B_2$, 
\begin{align*}
    \abs{\frac{A_1}{B_1} - \frac{A_2}{B_2}} \le \frac{1}{\abs{B_2}}\qty(\abs{\frac{A_1}{B_1}}\cdot \abs{B_1 - B_2} + \abs{A_1 - A_2}).
\end{align*}
\end{lemma}
\begin{proof}
    \begin{align*}
        \abs{\frac{A_1}{B_1} - \frac{A_2}{B_2}} &\le \abs{\frac{A_1}{B_1} - \frac{A_1}{B_2}} + \abs{\frac{A_1}{B_2} - \frac{A_2}{B_2}}\\
        &= \abs{A_1}\abs{\frac{1}{B_1} - \frac{1}{B_2}} + \frac{1}{\abs{B_2}}\abs{A_1 - A_2}\\
        &= \frac{\abs{A_1}\abs{B_1 - B_2}}{\abs{B_1B_2}} + \frac{\abs{A_1 - A_2}}{\abs{B_2}}\\
        &= \frac{1}{\abs{B_2}}\qty(\abs{\frac{A_1}{B_1}}\cdot \abs{B_1 - B_2} + \abs{A_1 - A_2}).
    \end{align*}
\end{proof}

For the following lemmas, let $\hat \theta = (\hat A^{(1)}, \hat A^{(2)})$ be the output of \Cref{alg:training_alg}. Define \begin{align*}
        \S\qty(A^{(1)}_*)_{ij} = \begin{cases}
                                    \mathbf{1}(j = p(i)) & p(i) \neq \emptyset \\
                                    \hat A^{(1)}_{i,j}        & p(i) = \emptyset
                                \end{cases}.
    \end{align*}
    and let $\tilde z(X; s) := \S(A_*^{(1)})Xe_s$.

\begin{lemma}\label{lem:compare_z_tilde}
    For $i \in \calR$,
    \begin{align*}
        \E_X\qty[\abs{\tilde z(X; s)_i - \hat \mu_{X_{\le i}}(s)}^2] \lesssim \min\qty(1, \frac{T^2\log^2T}{\TT \cdot i^2}).
    \end{align*}
\end{lemma}
\begin{proof}
    By \Cref{cor:end_of_stage_1}
    \begin{align*}
        \abs{\S(A_*^{(1)})_{i,j} - \frac{1}{i}} \lesssim \frac{T\log T}{\TT^{1/2} i^2}.
    \end{align*}
    Therefore
    \begin{align*}
        \abs{\tilde z(X; s)_i - \hat \mu_{X_{\le i}}(s)} & = \abs{\qty(\S(A_*^{(1)})_i - \frac{1}{i}\mathbf{1}_i)\cdot \delta_s(X_{\le i})} \\
                                                         & \le \norm{\S(A_*^{(1)})_i - \frac{1}{i}\mathbf{1}_i}_1                           \\
                                                         & \lesssim \frac{T \log T}{\TT^{1/2} i}.
    \end{align*}
    Finally,
    \begin{align*}
        \E_X\qty[\abs{\hat \mu_{X_{\le i}}(s) - \mu_\pi(s)}^2] \lesssim \frac{\mu_\pi(s) T^2}{\TT(\lambda) i^2}.
    \end{align*}
    Altogether,
    \begin{align*}
        \E_X\qty[\abs{\tilde z(X; s)_i - \hat \mu_{X_{\le i}}(s)}^2] \lesssim \frac{T^2\log^2T}{\TT \cdot i^2},
    \end{align*}
    and the conclusion follows as $\tilde z(X; s)_i, \hat \mu_{X_{\le i}}(s) \in [0, 1]$.
\end{proof}

\begin{lemma}\label{lem:concentrate_E3}
    Define \begin{align*}
        E^{(3)}_s(X) & := \sum_i \S(\beta \tilde z(X; s))_i\tilde z(X; s)_i
    \end{align*}
    Then
    \begin{align*}
         \mathbb{E}_X\qty[\abs{E^{(3)}_s(X) - \frac{(1-r)e^\beta\mu_\pi(s)  + r e^{\beta \mu_\pi(s)}\mu_\pi(s)}{(1-r)(e^\beta - 1)\mu_\pi(s) + (1 - r) + re^{\beta \mu_\pi(s)}}}^2] &\lesssim (1 + \beta^2) \cdot \frac{\log T}{\TT^{1/2}}.
    \end{align*}
\end{lemma}
\begin{proof}
    Plugging in the formula for $\tilde z(X; s)$ \eqref{eq:tilde_z_formula}, we get that
    \begin{align*}
        E^{(3)}_s(X) &= \frac{\sum_i \exp\qty(\beta \tilde z(X; s)_i)\tilde z(X; s)_i}{\sum_i \exp\qty(\beta \tilde z(X; s)_i)}\\
        &= \frac{e^\beta \sum_{i \in \overline{\mathcal{R}}} x_{p(i), s} + \sum_{i \in \calR}e^{\beta \tilde z(X; s)_i}\tilde z(X; s)_i}{(e^\beta - 1)\sum_{i \in \overline{\mathcal{R}}} x_{p(i), s} + \abs{ \overline{\mathcal{R}}} + \sum_{i \in \mathcal{R}}e^{\beta  \tilde z(X; s)_i}}
    \end{align*}
    We define the error terms
    \begin{align*}
        \mathcal{E}_1(X) & := \frac{1}{T}\sum_{i \in \overline{\mathcal{R}}} x_{p(i), s} - (1 - r)\mu_\pi(s)                               \\
        \mathcal{E}_2(X) & := \frac{1}{T}\sum_{i \in \mathcal{R}}e^{\beta \tilde z(X; s)_i}\tilde z(X; s)_i - r e^{\beta \mu_\pi(s)}\mu_\pi(s) \\
        \mathcal{E}_3(X) & := \frac{1}{T}\sum_{i \in \mathcal{R}}e^{\beta \tilde z(X; s)_i} - r e^{\beta \mu_\pi(s)}.
    \end{align*}
    Then
    \begin{align*}
        E^{(3)}_s(X) = \frac{(1 - r)e^\beta\mu_\pi(s) + r e^{\beta\mu_\pi(s)}\mu_\pi(s) + e^\beta\calE_1(X) + \calE_2(X)}{(1 - r)(e^\beta - 1)\mu_\pi(s) + (1 - r) + r e^{\beta \mu_\pi(s)} + (e^\beta - 1)\calE_1(X) + \calE_3(X)}
    \end{align*}
    Thus applying \Cref{lem:compute_diffs}, we get that
    \begin{align*}
         & \abs{E^{(3)}_s(X) - \frac{(1-r)e^\beta\mu_\pi(s)  + r e^{\beta \mu_\pi(s)}\mu_\pi(s)}{(1-r)(e^\beta - 1)\mu_\pi(s) + (1 - r) + re^{\beta \mu_\pi(s)}}} \\
         &\quad \le \abs{E^{(3)}_s(X)}\cdot\frac{\qty(\abs{e^\beta \calE_1(X)} + \calE_3(X)) + \qty(\abs{e^\beta \calE_1(X)} + \calE_2(X))}{(1-r)(e^\beta - 1)\mu_\pi(s) + (1 - r) + re^{\beta \mu_\pi(s)}}\\
         &\quad \lesssim (1 - r)^{-1}\cdot \qty(\abs{\mathcal{E}_1(X)} + e^{-\beta}\abs{\mathcal{E}_2(X)} + e^{-\beta}\abs{\mathcal{E}_3(X)}),
    \end{align*}
    since $\abs{E^{(3)}_s(X)} \le 1$ and $(1-r)(e^\beta - 1)\mu_\pi(s) + (1 - r) + re^{\beta \mu_\pi(s)} \ge (1 - r)e^\beta \gamma$.

    First, by \Cref{lem:muhat_parent_variance}, we have
    \begin{align*}
        \mathbb{E}\qty[\calE_1(X)^2] \lesssim \frac{1}{\TT}.
    \end{align*}

    Next, we bound $\calE_2$:
    \begin{align*}
        \abs{\calE_2(X)} \le \frac{1}{T}\sum_{i \in \calR} \abs{e^{\beta \tilde z(X; s)}\tilde z(X; s) - e^{\beta \mu_\pi(s)}\mu_\pi(s)} \le \frac{1}{T}\sum_{i \in \calR}(1 + \beta) e^{\beta}\abs{\tilde z(X; s)_i - \mu_\pi(s)},
    \end{align*}
    and thus by \Cref{lem:compare_z_tilde}
    \begin{align*}
        \mathbb{E}_X\qty[\calE_2(X)^2] & \le \frac{(1 + \beta)^2e^{2\beta}}{T}\sum_{i \in \calR}\mathbb{E}\abs{\tilde z(X; s)_i - \mu_\pi(s)}^2                                         \\
                                       & \lesssim \frac{(1 + \beta)^2e^{2\beta}}{T}\sum_i\min\qty(1, \frac{T^2\log^2T}{\TT \cdot i^2})                                              \\
                                       & =  \frac{(1 + \beta)^2e^{2\beta}}{T}\qty(\frac{T\log T}{\TT^{1/2}} + \sum_{i > \frac{T\log T}{\TT^{1/2}}}\frac{T^2\log^2T}{\TT \cdot i^2}) \\
                                       & \lesssim \frac{(1 + \beta)^2e^{2\beta} \log T}{\TT^{1/2}}.
    \end{align*}

    Next, we bound $\calE_3$.
    \begin{align*}
        \abs{\calE_3(X)} \le \frac{1}{T}\sum_{i \in \calR}\abs{e^{\beta\tilde z(X; s)_i} - e^{\beta \mu_\pi(s)}} \le \frac{1}{T}\sum_{i \in \calR}\beta e^{\beta}\abs{\tilde z(X; s)_i - \mu_\pi(s)},
    \end{align*}
    so by an identical calculation to as for $\calE_2$,
    \begin{align*}
        \mathbb{E}_X\qty[\calE_3(X)^2] & \lesssim \frac{\beta^2 e^{2\beta} \log T}{\TT^{1/2}}.
    \end{align*}
    Altogether,
    \begin{align*}
         & \mathbb{E}\qty[\abs{E^{(3)}_s(X) - \frac{(1-r)e^\beta\mu_\pi(s)  + r e^{\beta \mu_\pi(s)}\mu_\pi(s)}{(1-r)(e^\beta - 1)\mu_\pi(s) + (1 - r) + re^{\beta \mu_\pi(s)}}}^2] \\
         &\lesssim (1 - r)^{-2}\qty(\mathbb{E}\qty[\calE_1(X)^2] + e^{-2\beta}\mathbb{E}\qty[\calE_2(X)^2] + e^{-2\beta}\mathbb{E}\qty[\calE_3(X)^2])\\
         &\lesssim (1 + \beta^2) \cdot \frac{\log T}{\TT^{1/2}},
    \end{align*}
    where the last inequality also relies on \Cref{assume:r}.
\end{proof}

\begin{lemma}\label{lem:concentrate_E1_E2}
Define
    \begin{align*}
        E^{(1)}_{s, s'}(X) & := \sum_{i}x_{i, s'}\S(\beta \tilde z(X; s))_i\tilde z(X; s)_i \\
        E^{(2)}_{s, s'}(X) & := \sum_i x_{i, s'}\S(\beta \tilde z(X; s))_i,
    \end{align*}
Then
\begin{align*}
&\E\qty[\abs{\frac{E^{(1)}_{s, s'}(X)}{E^{(2)}_{s, s'}(X) + \epsilon} - \frac{(1 - r)e^\beta\mu_\pi(s)\pi(s' \mid s) + re^{\beta \mu_\pi(s)}\mu_\pi(s')\mu_\pi(s)}{(1-r)(e^\beta - 1)\mu_\pi(s)\pi(s' \mid s) + (1-r)\mu_\pi(s') + re^{\beta \mu_\pi(s)}\mu_\pi(s')}}^2]\\
  &\quad \lesssim  (1 + \beta^2) \cdot \frac{\log T}{\TT^{1/2}}.
\end{align*}
\end{lemma}
\begin{proof}
    Plugging in the formula for $\tilde z(X; s)$ \eqref{eq:tilde_z_formula}, we have that
    \begin{align*}
        \frac{E^{(1)}_{s, s'}(X)}{E^{(2)}_{s, s'}(X) + \epsilon} = \frac{e^\beta \sum_{i \in \overline{\mathcal{R}}} x_{i, s'}x_{p(i), s} + \sum_{i \in \calR}e^{\beta \tilde z(X; s)_i}\tilde z(X; s)_ix_{i, s'}}{(e^\beta - 1)\sum_{i \in \overline{\mathcal{R}}} x_{i, s'}x_{p(i), s} + \sum_{i \in \overline{\mathcal{R}}} x_{i, s'} + \sum_{i \in \mathcal{R}}e^{\beta  \tilde z(X; s)_i}x_{i, s'} + \epsilon\sum_{i}e^{\beta \tilde z(X; s)_i}}
    \end{align*}    
    We define the error terms
    \begin{align*}
        \mathcal{E}_4(X) & := \frac{1}{T}\sum_{i \in \overline{\mathcal{R}}} x_{p(i), s}x_{i, s'} - (1 - r)\mu_\pi(s)\pi(s' \mid s)                            \\
        \mathcal{E}_5(X) & := \frac{1}{T}\sum_{i \in \overline{\mathcal{R}}} x_{i, s'} - (1 - r)\mu_\pi(s')                                                    \\
        \mathcal{E}_6(X) & := \frac{1}{T}\sum_{i \in \mathcal{R}}\qty(e^{\beta \tilde z(X; s)_i}\tilde z(X; s)_ix_{i, s'} - e^{\beta \mu_\pi(s)}\mu_\pi(s)\mu_\pi(s')) \\
        \mathcal{E}_7(X) & = \frac{1}{T}\sum_{i \in \mathcal{R}}\qty(e^{\beta \tilde z(X; s)_i}x_{i, s'} - e^{\beta \mu_\pi(s)}\mu_\pi(s')).                        \\
    \end{align*}
    Then
    \begin{align*}
        &\frac{E^{(1)}_{s, s'}(X)}{E^{(2)}_{s, s'}(X) + \epsilon} = \\
        &\frac{(1 - r)e^\beta \mu_\pi(s)\pi(s' \mid s) + r e^{\beta \mu_\pi(s)}\mu_\pi(s)\mu_\pi(s') + e^\beta \calE_4(X) + \calE_6(X)}{(1 - r)\qty[(e^\beta - 1)\mu_\pi(s)\pi(s' \mid s) +\mu_\pi(s')] + re^{\beta \mu_\pi(s)}\mu_\pi(s') + (e^\beta - 1)\calE_4(X) + \calE_5(X) + \calE_7(X) + \frac{\epsilon}{T}\sum_{i}e^{\beta \tilde z(X; s)_i}}.
    \end{align*}
    Therefore by \Cref{lem:compute_diffs},
    \begin{align*}
        &\abs{\frac{E^{(1)}_{s, s'}(X)}{E^{(2)}_{s, s'}(X) + \epsilon} - \frac{(1 - r)e^\beta\mu_\pi(s)\pi(s' \mid s) + re^{\beta \mu_\pi(s)}\mu_\pi(s')\mu_\pi(s)}{(1-r)(e^\beta - 1)\mu_\pi(s)\pi(s' \mid s) + (1-r)\mu_\pi(s') + re^{\beta \mu_\pi(s)}\mu_\pi(s')}}\\
        &\le \abs{\frac{E^{(1)}_{s, s'}(X)}{E^{(2)}_{s, s'}(X) + \epsilon}}\cdot\frac{ \qty(e^\beta\abs{\calE_4(X)} + \abs{\calE_5(X)} + \abs{\calE_7(X)} + e^\beta\epsilon) + e^\beta\abs{\calE_4(x)} + \abs{\calE_6(X)}}{(1-r)(e^\beta - 1)\mu_\pi(s)\pi(s' \mid s) + (1-r)\mu_\pi(s') + re^{\beta \mu_\pi(s)}\mu_\pi(s')}\\
        &\lesssim \abs{\calE_4(X)} + e^{-\beta}\abs{\calE_5(X)} + e^{-\beta}\abs{\calE_6(X)} + e^{-\beta}\abs{\calE_7(X)} + \epsilon,
    \end{align*}
    where the last step uses $(1-r)(e^\beta - 1)\mu_\pi(s)\pi(s' \mid s) + (1-r)\mu_\pi(s') + re^{\beta \mu_\pi(s)}\mu_\pi(s') \ge (1 - r)e^\beta \gamma^2$ along with \Cref{assume:r} and the convention that $\lesssim$ subsumes $\gamma^{-1}$ terms, and also that $\abs{\frac{E^{(1)}_{s, s'}(X)}{E^{(2)}_{s, s'}(X) + \epsilon}} \le 1$.

We can use \Cref{lem:chat_variance} to bound $\calE_4$:
\begin{align*}
    \E[\calE_4(X)^2] &= \frac{\abs{\overline \calR}^2}{T^2}\E[(\hat c_{X_{\overline \calR}}(s, s') - \mu_\pi(s)\pi(s' \mid s)^2]\\
    & \lesssim \frac{\abs{\overline \calR}^2}{T^2} \cdot \frac{\mu_\pi(s')T^2}{\TT(\lambda) \abs{\overline \calR}^2}\\
    &\lesssim 
    \frac{1}{\TT(\lambda)}. 
\end{align*}
Next, we use \Cref{lem:muhat_variance} to bound $\calE_5$:
\begin{align*}
    \E[\calE_5(X)^2] &= \frac{\abs{\overline \calR}^2}{T^2}\E[(\hat \mu_{X_{\overline \calR}}(s') - \mu_\pi(s'))^2]\\
    &\lesssim \frac{\abs{\overline \calR}^2}{T^2} \cdot \frac{\mu_\pi(s')T^2}{\TT(\lambda) \abs{\overline \calR}^2}\\
    &\lesssim 
    \frac{1}{\TT(\lambda)}.
\end{align*}

Next, we bound $\calE_6$:
\begin{align*}
    \abs{\calE_6(X)} &\le \frac{1}{T}\sum_{i \in \calR}\abs{x_{i, s'}\qty(e^{\beta \tilde z(X; s)}\tilde z(X; s) - e^{\beta \mu_\pi(s)}\mu_\pi(s))} + \frac{1}{T}\abs{e^{\beta\mu_\pi(s)}\mu_\pi(s)\sum_{i \in \calR}\qty(x_{i, s'} - \mu_\pi(s'))}\\
    &\le \frac{1}{T}\sum_{i \in \calR}(1 + \beta)e^\beta \abs{\tilde z(X; s) - \mu_\pi(s)} + \frac{1}{T}e^\beta \abs{\sum_{i \in \calR}\qty(x_{i, s'} - \mu_\pi(s'))}.
\end{align*}

The first term can be bounded equivalently as to was done for $\calE_2$, and thus
\begin{align*}
    \E\qty[\qty(\frac{1}{T}\sum_{i \in \calR}(1 + \beta)e^\beta \abs{\tilde z(X; s) - \mu_\pi(s)})^2] \lesssim \frac{(1 + \beta)^2e^{2\beta} \log T}{\TT^{1/2}}.
\end{align*}
In the second term, since $x_{i,s'} - \mu_\pi(s')$ are independent and mean 0 for all $i \neq T$,
\begin{align*}
    \E\qty[\qty(\frac{1}{T}e^\beta \abs{\sum_{i \in \calR}\qty(x_{i, s'} - \mu_\pi(s'))})^2] &= \frac{e^{2\beta}}{T^2} \sum_{i \in \calR}\E\qty[\qty(x_{i, s'} - \mu_\pi(s'))^2]\\
    &\lesssim \frac{e^{2\beta}}{T}.
\end{align*}
Altogether
\begin{align*}
    \E[\abs{\calE_6(X)}^2] \lesssim \frac{(1 + \beta)^2e^{2\beta} \log T}{\TT^{1/2}}.
\end{align*}

Finally, we bound $\calE_7$:
\begin{align*}
    \abs{\calE_7(X)} &\le \frac{1}{T}\sum_{i \in \calR}\abs{x_{i, s'}\qty(e^{\beta \tilde z(X; s)} - e^{\beta \mu_\pi(s)})} + \frac{1}{T}\abs{e^{\beta\mu_\pi(s)}\sum_{i \in \calR}\qty(x_{i, s'} - \mu_\pi(s'))}\\
    &\le \frac{1}{T}\sum_{i \in \calR}\beta e^\beta \abs{\tilde z(X; s) - \mu_\pi(s)} + \frac{1}{T}e^\beta \abs{\sum_{i \in \calR}\qty(x_{i, s'} - \mu_\pi(s'))}.
\end{align*}
Thus via an identical calculation as $\calE_6$,
\begin{align*}
    \E[\abs{\calE_7(X)}^2] \lesssim \frac{(1 + \beta)^2e^{2\beta} \log T}{\TT^{1/2}}.
\end{align*}
Altogether,
\begin{align*}
&\E\qty[\abs{\frac{E^{(1)}_{s, s'}(X)}{E^{(2)}_{s, s'}(X) + \epsilon} - \frac{(1 - r)e^\beta\mu_\pi(s)\pi(s' \mid s) + re^{\beta \mu_\pi(s)}\mu_\pi(s')\mu_\pi(s)}{(1-r)(e^\beta - 1)\mu_\pi(s)\pi(s' \mid s) + (1-r)\mu_\pi(s') + re^{\beta \mu_\pi(s)}\mu_\pi(s')}}^2]\\
  &\lesssim  (1 + \beta^2) \cdot \frac{\log T}{\TT^{1/2}}.
\end{align*}
\end{proof}

\begin{lemma}\label{lem:f_to_pi}
    Let $\hat \theta = \qty(\hat A^{(1)}, \hat A^{(2)})$ be the output of \Cref{alg:training_alg}, where $\hat A^{(2)} = (\beta_0 + \beta(\tau_1 + \tau_2))I_S - \frac{\beta(\tau_1 + \tau_2)}{S}1_S1_S^\top $. Then
    \begin{align*}
        \E_X\qty[\abs{f_{\htheta}(X; s)_{s'} - \pi(s' \mid s)}^2] &\lesssim (1 + {\beta^*}^2)\cdot \frac{\log T}{\TT^{1/2}} + e^{-\beta^*\gamma}.
    \end{align*}
    and
    \begin{align*}
        \mathbb{P}\qty[f_{\htheta}(X; s)_{s'} \le \frac{\gamma^3}{4S^2}] \lesssim \frac{1}{\TT}.
    \end{align*}
\end{lemma}
\begin{proof}
    First, by \Cref{lem:A2_dynamics}, $1 + \beta^* \ge \beta(\tau_1 + \tau_2) \ge \beta^*$. For notational convenience, let $\beta = \beta(\tau_1 + \tau_2)$ Recall the definitions     
    \begin{align*}
        \S\qty(A^{(1)}_*)_{ij} = \begin{cases}
                                    \mathbf{1}(j = p(i)) & p(i) \neq \emptyset \\
                                    \hat A^{(1)}_{i,j}        & p(i) = \emptyset
                                \end{cases}.
    \end{align*}
    and
    \begin{align*}
        \tilde z(X; s) = \S(A_*^{(1)})\delta_s(X) = \begin{cases}
                               x_{p(i), s}      & \text{if}~i \not\in \calR \\
                               z_{\htheta}(X; s)_i & \text{if}~i \in \calR
                           \end{cases}.
    \end{align*}
    By \Cref{cor:end_of_stage_1} $\|z_{\htheta}(X; s) - \tilde z(X; s)\|_\infty \lesssim T^{-1}$. Letting $f(z) = \delta^\top \S(\beta z)$, we see that $\|\nabla_zf(z)\|_1 = \beta\|J(\S(\beta z))\delta\|_1 \le 2\beta$, and thus $$\abs{f_{\htheta}(X; s)_{s'} - \delta_{s'}(X)^\top \S(\beta \tilde z(X; s))} \lesssim \beta T^{-1}.$$
    Next, we have that
    \begin{align*}
        \delta_{s'}(X)^\top \S(\beta \tilde z(X; s)) = \frac{\sum_i x_{i, s'}\exp(\beta \tilde z(X; s)_i)}{\sum_i \exp(\beta \tilde z(X; s)_i)},
    \end{align*}
    and thus
    \begin{align*}
        &\delta_{s'}(X)^\top \S(\beta \tilde z(X; s))\\
        &= \frac{(e^\beta - 1)\sum_{i \in \overline{\calR}}x_{p(i), s}x_{i, s'} + \sum_{i \in \overline{\calR}}x_{i, s'} + \sum_{i \in \calR} x_{i,s'}e^{\beta \tilde z(X; s)_i}}{(e^\beta - 1)\sum_{i \in \overline{\calR}}x_{p(i), s} + \abs{\overline{\cal R}} + \sum_{i \in \calR} e^{\beta \tilde z(X; s)_i}}\\
        &= \frac{(1 - r)(e^\beta - 1)\mu_\pi(s)\pi(s' \mid s) + (1 - r)\mu_\pi(s') + r\mu_\pi(s')e^{\beta \mu_\pi(s)} + (e^\beta - 1)\calE_4(X) + \calE_5(X) + \calE_7(X)}{(1 - r)(e^{\beta} - 1)\mu_\pi(s) + (1 - r) + re^{\beta \mu_\pi(s)} + (e^\beta - 1)\calE_1(X) + \calE_3(X)}
    \end{align*}

    Therefore by \Cref{lem:compute_diffs},
    \begin{align*}
        &\abs{\delta_{s'}(X)^\top \S(\beta \tilde z(X; s)) - \frac{(1 - r)(e^\beta - 1)\mu_\pi(s)\pi(s' \mid s) + (1 - r)\mu_\pi(s') + r\mu_\pi(s')e^{\beta \mu_\pi(s)}}{(1 - r)(e^{\beta} - 1)\mu_\pi(s) + (1 - r) + re^{\beta \mu_\pi(s)}}}\\
        &\quad\le \frac{\abs{\delta_{s'}(X)^\top \S(\beta \tilde z(X; s))}\cdot \qty(e^\beta\abs{\calE_4(X)} + \abs{\calE_5(X)} + \abs{\calE_7(X)}) + e^\beta\abs{\calE_1(x)} + \abs{\calE_3(X)}}{(1 - r)(e^{\beta} - 1)\mu_\pi(s) + (1 - r) + re^{\beta \mu_\pi(s)}}\\
        &\quad\lesssim  \abs{\calE_1(x)} + e^{-\beta}\abs{\calE_3(X)} + \abs{\calE_4(X)} + e^{-\beta}\abs{\calE_5(X)} + e^{-\beta}\abs{\calE_7(X)},
    \end{align*}
    where the last inequality uses \Cref{assume:r}. Next, see that
    \begin{align*}
        &\abs{\frac{(1 - r)(e^\beta - 1)\mu_\pi(s)\pi(s' \mid s) + (1 - r)\mu_\pi(s') + r\mu_\pi(s')e^{\beta \mu_\pi(s)}}{(1 - r)(e^{\beta} - 1)\mu_\pi(s) + (1 - r) + re^{\beta \mu_\pi(s)}} - \pi(s' \mid s)}\\
        &\quad\le \frac{\abs{(1 - r)\mu_\pi(s') + r\mu_\pi(s')e^{\beta \mu_\pi(s)} + (1 - r)\pi(s' \mid s) + r e^{\beta \mu_\pi(s)}\pi(s' \mid s)}}{(1 - r)(e^{\beta} - 1)\mu_\pi(s) + (1 - r) + re^{\beta \mu_\pi(s)}}\\
        &\quad\lesssim \frac{e^{\beta \mu_\pi(s)}}{(1 - r)e^\beta \gamma}\\
        &\quad\lesssim e^{\beta(\mu_\pi(s) - 1)}\\
        &\quad\lesssim e^{-\beta \frac{\gamma(S - 1)}{S}}\\
        &\quad\lesssim e^{-\beta \gamma/2}
    \end{align*}

    Altogether, we get
    \begin{align*}
        \abs{f_{\htheta}(X; s)_{s'} - \pi(s' \mid s)} \lesssim \frac{\beta}{T} + \abs{\calE_4(X)} + \abs{\calE_1(x)} + e^{-\beta}\abs{\calE_3(X)} + e^{-\beta}\abs{\calE_5(X)} + e^{-\beta}\abs{\calE_7(X)} + e^{-\beta \gamma /2},
    \end{align*}
    and thus
    \begin{align*}
        \E_X\qty[\abs{f_{\htheta}(X; s)_{s'} - \pi(s' \mid s)}^2] &\lesssim (1 + \beta^2)\cdot \frac{\log T}{\TT^{1/2}} + e^{-\beta\gamma} \lesssim (1 + {\beta^*}^2)\cdot \frac{\log T}{\TT^{1/2}} + e^{-\beta^*\gamma}
    \end{align*}

    Next, we need to bound $\mathbb{P}_X\qty[f_{\htheta}(X; s)_{s'} \le \frac{\gamma^3}{4S^2}]$. We start by bounding the probability $\delta_{s'}(X)^\top \S(\beta \tilde z(X; s))$ is small. We have the naive bound
    \begin{align*}
        \delta_{s'}(X)^\top \S(\beta \tilde z(X; s)) &= \frac{\sum_i x_{i, s'}\exp(\beta \tilde z(X; s)_i)}{\sum_i \exp(\beta \tilde z(X; s)_i)}\\
        &\ge \frac{e^\beta \sum_{i \in \overline{\cal R}}x_{p(i), s}x_{i, s'}}{e^\beta \cdot T}\\
        &= \frac{1}{T}\sum_{i \in \overline{\cal R}}x_{p(i), s}x_{i, s'}\\
        & = (1 - r) \hat c_{X_{\overline{\cal R}}}(s, s').
    \end{align*}
    By Markov's inequality and \Cref{lem:chat_variance},
    \begin{align*}
        \mathbb{P}_X\qty[\hat c_{X_{\overline{\cal R}}}(s, s') \le \frac{\gamma^2}{2S^2}] &\le \mathbb{P}_X\qty[\abs{\hat c_{X_{\overline{\cal R}}}(s, s') - \mu_\pi(s)\pi(s' \mid s)} \ge  \frac{\gamma^2}{2S^2}]\\
        &\le \frac{2S^2}{\gamma^2} \E_X\qty[\abs{c_{X_{\overline{\cal R}}}(s, s') - \mu_\pi(s)\pi(s' \mid s)}^2]\\
        &\lesssim \frac{T^2}{\abs{\overline{\cal R}}^2 \TT}\\
        &\lesssim \frac{1}{\TT}.
    \end{align*}
    Therefore
    \begin{align*}
        \mathbb{P}_X\qty[\delta_{s'}(X)^\top \S(\beta \tilde z(X; s)) \le  \frac{\gamma^3}{2S^2}] \le \mathbb{P}_X\qty[\delta_{s'}(X)^\top \S(\beta \tilde z(X; s)) \le  (1 - r)\frac{\gamma^2}{2S^2}] \lesssim \frac{1}{\TT}.
    \end{align*}
    To conclude, on the event that $\delta_{s'}(X)^\top \S(\beta \tilde z(X; s)) >  \frac{\gamma^3}{2S^2}$, we have
    \begin{align*}
        f_{\htheta}(X; s)_{s'} &>  \frac{\gamma^3}{2S^2} - \abs{f_{\htheta}(X; s)_{s'} - \delta_{s'}(X)^\top \S(\beta \tilde z(X; s))}\\
        &\ge \frac{\gamma^3}{2S^2} - O\qty(\beta T^{-1})\\
        &\ge \frac{\gamma^3}{4S^2},
    \end{align*}
    since $\beta \le 1 + \beta^* \lesssim T$. Altogether,
    \begin{align*}
        \mathbb{P}_X\qty[f_{\htheta}(X; s)_{s'} \le \frac{\gamma^3}{4S^2}]  \lesssim \frac{1}{\TT}.
    \end{align*}
\end{proof}

\subsection{Proof of Theorem \ref{thm:OOD}}

\begin{proof}
    By \Cref{lem:f_to_pi}, we get that
    \begin{align*}
        \mathbb{E}_X\qty[\qty(f(X; s)_{s'} - \pi(s' \mid s))^2] \lesssim_{\gamma, S} \frac{\log T}{\TT^{\Theta_\gamma(1)}}
\end{align*}
    Therefore by Markov's inequality,
    \begin{align*}
        \mathbb{P}_X\qty[\qty(f(X; s)_{s'} - \pi(s' \mid s))^2 \ge 100S^2\cdot \mathbb{E}_X\qty[\qty(f(X; s)_{s'} - \pi(s' \mid s))^2]] \le \frac{1}{100S^2}.
    \end{align*}
    Union bounding, with probability $0.99$ we have 
    \begin{align*}
        \sup_{s, s'}\abs{f(X; s)_{s'} - \pi(s' \mid s)} \le 100S^2\cdot \mathbb{E}_X\qty[\qty(f(X; s)_{s'} - \pi(s' \mid s))^2] \lesssim_{\gamma, S} \frac{\log T}{\TT^{\Theta_\gamma(1)}},
    \end{align*}
    as desired.
\end{proof}

\section{Finite Sample Analysis}
Our theory focuses on the case of gradient descent on the population loss \eqref{eq:CE_loss}. It is relatively straightforward to extend our analysis to the finite sample setting. In this case, we are given a dataset of $N$ prompts of length $T$:
\begin{align*}
    \mathcal{D} = \{s^{(n)}_{1:T}\}_{n \in [N]}.
\end{align*}
Each sequence $s^{(n)}_{1:T}$ is generated via the procedure in \Cref{task:single_parent}, with transition matrix $\pi^{(i)} \sim P_\pi$. Let $X^{(n)} \in \mathbb{R}^{T \times S}$ be the embedding of $s^{(n)}_{1:T}$. 

We now consider running gradient descent on the finite sample loss $\hat L$:
\begin{align}
    \hat L(\theta) = -\frac{1}{N}\sum_{n=1}^N \sum_{s' \in [S]} \pi^{(n)}(s' \mid s_T^{(n)})\log\qty(f_\theta(s_{1:T}^{(n)}) + \epsilon).
\end{align}
Below, we present a sketch of the extension of the analysis of our main theorem to this finite sample setting.

\subsection{Stage 1}
The crux of Stage 1 is \Cref{thm:stage1}, where in \eqref{eq:key_equation_stage1} we show that 
\begin{align*}
    \nabla_{A^{(1)}_i}L(\theta) = - \frac{\beta_0}{ST}J\qty(\S\qty(A_i^{(1)}))g_i^*,
\end{align*}
where the vector $g_i^* \in \mathbb{R}$ is defined by
\begin{align*}
    g^*_i := T\sum_{s, s'}\E_{\pi, X}\qty[\frac{\pi(s' \mid s)}{f_\theta(X; s)_{s'} + \epsilon}\delta_{s'}(X)^\top J(v_\theta(X; s))e_i \cdot  \delta_s(X_{\le i})].
\end{align*}
In particular, we show that $g^*_i$ satisfies the property that
\begin{align*}
    g^*_{i,j} - g^*_{i, p(i)} \le -\frac{\gamma^3}{4S}
\end{align*}
for $i \in \overline{\mathcal{R}}$ and $\abs{g^*_{i,j}} \lesssim \TT^{-1/2}$ for $i \in \mathcal{R}$. As a step towards proving this, we let $\hat\theta = (A^{(1)}, 0)$, define the quantity $\hat g_i$ by
\begin{align*}
    \hat g_i := T\sum_{s, s'}\E_{\pi, X}\qty[\frac{\pi(s' \mid s)}{f_\theta(X; s)_{s'} + \epsilon}\delta_{s'}(X)^\top J(v_\theta(X; s))e_i \cdot  \delta_s(X_{\le i})],
\end{align*}
and show that
\begin{align*}
    \hat g_{i,j} - \hat g_{i, p(i)} &\le -\frac{\gamma^3}{4S}~\text{for}~i \in \overline{\mathcal{R}}\\
    \abs{\hat g_{i,j}} &\lesssim \TT^{-1/2}~\text{for}~i \in \mathcal{R}.
\end{align*}

The empirical gradient can be written as
\begin{align*}
    \nabla_{A^{(1)}_i}\hat L(\theta) = - \frac{\beta_0}{ST}J\qty(\S\qty(A_i^{(1)}))g^{\mathrm{emp}}_i,
\end{align*}
where
\begin{align*}
    g^{\mathrm{emp}}_i = \frac{ST}{N}\sum_{n=1}^N\sum_{s'}\frac{\pi^{(n)}(s' \mid s_T^{(n)})}{f_\theta(X^{(n)}; s_T^{(n)})_{s'} + \epsilon}\delta_{s'}(X^{(n)})^\top J(v_\theta(X^{(n)}; s_T^{(n)}))e_i \cdot  \delta_{s_T^{(n)}}(X^{(n)}_{\le i}).
\end{align*}
As in the population setting, we define $\hat g^{\mathrm{emp}}_i$ to be
\begin{align*}
    \hat g^{\mathrm{emp}}_i = \frac{ST}{N}\sum_{n=1}^N\sum_{s'}\frac{\pi^{(n)}(s' \mid s_T^{(n)})}{f_{\htheta}(X^{(n)}; s_T^{(n)})_{s'} + \epsilon}\delta_{s'}(X^{(n)})^\top J(v_{\htheta}(X^{(n)}; s_T^{(n)}))e_i \cdot  \delta_{s_T^{(n)}}(X^{(n)}_{\le i}).
\end{align*}
By \Cref{lem:nonzero_beta}, we get that $\|g^{\mathrm{emp}}_i - \hat g^{\mathrm{emp}}_i\|_\infty \le \frac{C_{\gamma, S}}{\sqrt{\TT}}.$ It thus suffices to show that $\|\hat g_i^{\mathrm{emp}} - \hat g_i\|_{\infty}$ is small, which is given by the following lemma:
\begin{lemma}\label{lem:concentrate_stage1}
    For any $\delta > 0$
    \begin{align*}
        \|\hat g_i^{\mathrm{emp}} - \hat g_i\|_{\infty} \le \frac{C_{\gamma, S}\log\qty(\frac{T}{\delta})}{\sqrt{N}}
    \end{align*}
    with probability $1 - \delta$.
\end{lemma}
\begin{proof}
First, see that $\hat g^{\mathrm{emp}}_{i,j}$ can be written as
\begin{align*}
    \hat g^{\mathrm{emp}}_{i, j} = \frac{1}{N}\sum_{n=1}^N \underbrace{S\sum_{s'}\frac{\pi^{(n)}(s' \mid s_T^{(n)})}{\hat\mu_{X^{(n)}}(s') + \epsilon}(x^{(n)}_{i, s'} - \hat\mu_{X^{(n)}}(s'))x^{(n)}_{j, s_T^{(n)}}}_{=:Z^{(n)}}.
\end{align*}
Define the event $\mathcal{A}^{(n)}$ as 
\begin{align*}
    \mathcal{A}^{(n)} = \bigcup_{s' \in [S]}\{ \hat \mu_{X^{(n)}}(s') \le \frac12\mu_{\pi^{(n)}}(s')\}.
\end{align*}
By a union bound, $\mathbb{P}(\mathcal{A}^{(n)}) \le S/\TT$. We can naively bound $\abs{Z^{(n)}} \le \epsilon^{-1}S$, and on the complement of $\mathcal{A}^{(n)}$ (denoted by $\overline{\mathcal{A}^{(n)}}$) we can bound $\abs{Z^{(n)}} \lesssim S\gamma^{-1}$. Therefore we can concentrate $\frac{1}{N}\sum_n Z^{(n)}$ as:
\begin{align*}
    \abs{\frac{1}{N}\sum_n Z^{(n)} - \mathbb{E}[Z]} &\le \abs{\frac{1}{N}\sum_n Z^{(n)}\mathbf{1}(\overline{\mathcal{A}^{(n)}}) - \mathbb{E}[Z^{(n)}\mathbf{1}(\overline{\mathcal{A}^{(n)}})]} + \abs{\frac{1}{N}\sum_n Z^{(n)}\mathbf{1}(\mathcal{A}^{(n)}) - \mathbb{E}[Z^{(n)}\mathbf{1}(\mathcal{A}^{(n)})]}\\
    &\lesssim \frac{S\gamma^{-1}\log(T/\delta)}{\sqrt{N}} + \epsilon^{-1}\cdot \frac{1}{N}\sum_n\mathbf{1}(\mathcal{A}^{(n)}) + \epsilon^{-1}\mathbb{P}(\mathcal{A}^{(n)}),
\end{align*}
with probability $1 - \frac{\delta}{2T^2}$ by Hoeffding's inequality on the $Z^{(n)}\mathbf{1}(\mathcal{A}^{(n)})$. Next, we see that the $\mathbf{1}(\mathcal{A}^{(n)})$ are Bernoulli random variables with mean $\mathbb{P}(\mathcal{A}^{(n)}) \le S/\TT$ and standard deviation at most $\sqrt{\mathbb{P}(\mathcal{A}^{(n)})/N}$. Therefore with probability $1 - \frac{\delta}{T^2}$, one has
\begin{align*}
    \abs{\frac{1}{N}\sum_n Z^{(n)} - \mathbb{E}[Z]} &\lesssim \frac{S\gamma^{-1} + \epsilon^{-1}\sqrt{S/\TT}}{\sqrt{N}}\log(T/\delta) \\
    &= \frac{C_{\gamma, S}\log(T/\delta)}{\sqrt{N}},
\end{align*}
since $\epsilon = \Theta(\TT^{-1/2})$. Union bounding over $i,j \in [T]$ yields the desired result
\end{proof}
Combining everything together, we get that for $N \gtrsim C_{\gamma, S}\TT \log T$, the quantities $g_i^{\mathrm{emp}}$ satisfy
\begin{align*}
    g_{i,j}^{\mathrm{emp}} - g_{i, p(i)}^{\mathrm{emp}} &\le -\frac{\gamma^3}{4S}~\text{for}~i \in \overline{\mathcal{R}}\\
    \abs{g_{i,j}^{\mathrm{emp}}} &\lesssim \TT^{-1/2}~\text{for}~i \in \mathcal{R}.
\end{align*}
with high probability, and thus Stage 1 succeeds on the empirical loss with high probability.

\subsection{Stage 2}
One challenge in directly using the population analysis for stage 2 is that the finite-sample update no longer preserves symmetry, and hence we do not have that $A^{(2)} = \beta_0I_S + \beta (I_S - \frac{1}{S}1_S1_S^\top)$ throughout the entirety of stage 2. Instead, we will consider taking only a single large gradient step with learning rate $\eta_2$, on an independent dataset of $N$ prompts.

During the first step of stage 2, the population gradient is
\begin{align*}
    \nabla_{A^{(2)}}L(\theta) = -\beta^{\mathrm{pop}} \cdot \qty(I_S - \frac{1}{S}1_S1_S^\top),
\end{align*}
where, by \Cref{thm:stage2},
\begin{align*}
    1 \ge \beta^{\mathrm{pop}} \ge C_{\gamma, S}^{-1} > 0.
\end{align*}
The following lemma concentrates the population gradient to the empirical gradient at $\theta(\tau_1)$:
\begin{lemma}
    \begin{align*}
        \norm{\nabla_{A^{(2)}}\hat L(\theta) - \nabla_{A^{(2)}}L(\theta)}_\infty \lesssim \frac{C_{\gamma, S}}{\sqrt{N}}.
    \end{align*}
\end{lemma}
\begin{proof}
    The finite sample gradient can be written as
    \begin{align*}
        \nabla_{A^{(2)}}\hat L(\theta) = -\frac{1}{N}\sum_{n=1}^N\underbrace{\sum_{s'}\frac{\pi^{(n)}(s' \mid s_T^{(n)})}{f_\theta(X^{(n)}; s_T^{(n)})_{s'} + \epsilon}\cdot {X^{(n)}}^\top \S(A^{(1)})^\top J(v_\theta(X^{(n)}; s_T^{(n)}))\delta_{s'}(X^{(n)})e_{s_T^{(n)}}^\top}_{M^{(n)}}.
    \end{align*}
    Let $\hat \theta = (A^{(1)}(\tau_1), 0)$. By \eqref{eq:bound_f}, we have
    \begin{align*}
        \abs{f_\theta(X^{(n)}; s_T^{(n)})_{s'} - f_{\htheta}(X^{(n)}; s_T^{(n)})_{s'}} \lesssim e^{\beta_0 - 1} \lesssim \TT^{-1/2}.
    \end{align*}
    Additionally, $f_{\htheta}(X^{(n)}; s_T^{(n)})_{s'} = \hat\mu_{X^{(n)}}(s')$. We next see that we can bound
    \begin{align*}
        &\abs{e_s^\top{X^{(n)}}^\top \S(A^{(1)})^\top J(v_\theta(X^{(n)}; s_T^{(n)}))\delta_{s'}(X^{(n)})}\\
        &\quad \le 2 \norm{\S(A^{(1)}){X^{(n)}}e_s}_\infty \norm{\delta_{s'}(X^{(n)})}_\infty\\
        &\quad \le 2.
    \end{align*}
    Therefore on $\mathcal{A}^{(n)}$, each entry of $M^{(n)}$ can be bounded in absolute value by $2\epsilon^{-1}$, and on $\overline{\mathcal{A}^{(n)}}$ each entry can be bounded by some $C_{\gamma, S}$.
    Therefore by an identical concentration argument as in \Cref{lem:concentrate_stage1}, with high probability we get that
    \begin{align*}
        \norm{\nabla_{A^{(2)}}\hat L(\theta) - \nabla_{A^{(2)}}L(\theta)}_\infty \lesssim \frac{C_{\gamma, S}}{\sqrt{N}}.
    \end{align*}
\end{proof}
After one gradient step with learning rate $\eta_2$, $A^{(2)}(\tau_1 + 1)$ is equal to 
\begin{align*}
    A^{(2)}(\tau_1 + 1) = \beta_0 I_S + \eta_2\beta^{\mathrm{pop}}\cdot\qty(I_S - \frac{1}{S}1_S1_S^\top) + R,
\end{align*}
where the error matrix $R$ satisfies $\|R\|_\infty \lesssim \frac{\eta_2}{\sqrt{N}}$. 

Next, define the parameter vector $\theta^{\mathrm{pop}}$ as $\theta^{\mathrm{pop}} = (A^{(1)}(\tau_1), A^{(2)}_{\mathrm{pop}})$, where $A^{(2)}_{\mathrm{pop}} = \beta_0 I_S + \eta_2\beta^{\mathrm{pop}}\cdot\qty(I_S - \frac{1}{S}1_S1_S^\top)$ is the result of the population update. We can bound the error between the finite-sample predictor and population predictor as
\begin{align*}
    \abs{f_{\theta(\tau_1 + 1)}(X; s)_{s'} - f_{\theta^{\mathrm{pop}}}(X; s)_{s'}} &= \abs{\delta_{s'}^\top\qty(\S(\S(A^{(1)})XA^{(2)}(\tau_1 + 1)e_s) - \S(\S(A^{(1)})XA^{(2)}_{\mathrm{pop}}e_s))}\\
    &\le \|\S(\S(A^{(1)})XA^{(2)}(\tau_1 + 1)e_s) - \S(\S(A^{(1)})XA^{(2)}_{\mathrm{pop}}e_s)\|_\infty\\
    &\le \|\S(A^{(1)})XA^{(2)}(\tau_1 + 1)e_s - \S(A^{(1)})XA^{(2)}_{\mathrm{pop}}e_s\|_\infty \\
    &\le \|A^{(2)}(\tau_1 + 1) - A^{(2)}_{\mathrm{pop}}\|_\infty \\
    &\lesssim \frac{\eta_2}{\sqrt{N}}.
\end{align*}
Now if we choose $\eta_2$ so that $\eta_2\beta^{\mathrm{pop}} = \beta^* = \Theta(\log \TT)$, then by \Cref{lem:f_to_pi}, we get that $\theta^{\mathrm{pop}}$ satisfies
\begin{align*}
    \mathbb{E}_X\qty[\abs{f_{\theta^{\mathrm{pop}}}(X; s)_{s'} - \pi(s' \mid s)}^2] \lesssim \TT^{-c\gamma}.
\end{align*}
Therefore 
\begin{align*}
    \mathbb{E}_X\qty[\abs{f_{\theta(\tau_1 + 1)}(X; s)_{s'} - \pi(s' \mid s)}^2] &\lesssim \TT^{c\gamma} + \frac{\eta_2^2}{N}\\
    &\lesssim \TT^{-c\gamma} + \frac{\log^2 \TT}{N}\\
    &\lesssim \TT^{-c\gamma},
\end{align*}
as long as $N \gtrsim \TT^{c\gamma}$. Therefore the output of running gradient descent on the finite sample loss, $f_{\theta(\tau_1 + 1)}$, achieves small population loss.

Altogether, both Stage 1 and Stage 2 succeed on the finite-sample loss as long as $N \gtrsim \TT \log T$.
\end{document}